%% file: main.tex
\documentclass{article} 
\usepackage{iclr2027_conference,times}

\input{math_commands.tex}

\usepackage{hyperref}
\usepackage{url}
\usepackage{capt-of}
\usepackage{graphicx} 
\usepackage{algpseudocode}
\usepackage{algorithm}
\usepackage{amsmath,amssymb,amsthm,mathtools}
\usepackage{tikz}
\usepackage{xspace}
\usetikzlibrary{arrows.meta,positioning,calc,decorations.pathreplacing}
\usepackage{enumitem}
\usepackage{hyperref}
\usepackage{array}
\usepackage{booktabs}
\usepackage{tabularx}
\usepackage{float}
\usepackage{multirow}
\usepackage{subcaption} 
\usepackage[section]{placeins}
\usepackage[most]{tcolorbox}
\newtcolorbox{theorembox}{
  enhanced,
  colback=gray!3,
  colframe=black,
  boxrule=0pt,
  leftrule=2pt,
  arc=0pt,
  left=7pt,
  right=5pt,
  top=5pt,
  bottom=5pt,
  before skip=10pt,
  after skip=10pt
}

\usepackage{titletoc}
\contentsmargin{0pt}
\titlecontents{section}
  [2.0em]                                   
  {\addvspace{0.9em}\bfseries}            
  {\contentslabel{1.5em}}                 
  {}                                      
  {\titlerule*[0.7pc]{.}\contentspage}    
\titlecontents{subsection}
  [5.5em]                                 
  {}                                      
  {\contentslabel{3.5em}}                 
  {}                                      
  {\titlerule*[0.7pc]{.}\contentspage}

\title{Structure-Adaptive Tree Field Integrators}

\author{Millend Roy, Soham Samal, Ivan Zelich \& Krzysztof Marcin Choromanski 
\\
Columbia University\\
New York, NY 10027, USA \\
}

\newcommand{\dtree}{d_T}

\newcommand{\STADTFI}[1][]{\textsc{StAd-TFI}#1\xspace}
\newcommand{\TFI}{\textsc{TFI}\xspace}
\newcommand{\dist}{dist}

\newtheorem{definition}{Definition}
\newtheorem{lemma}{Lemma}
\newtheorem{theorem}{Theorem}
\newtheorem{corollary}{Corollary}
\newtheorem{proposition}{Proposition}

\iclrfinalcopy

\begin{document}

\maketitle

\lhead{Preprint}
\chead{}
\rhead{}
\vspace{-15pt}
\begin{abstract}
\input{sections/abstract}
\end{abstract}

\input{sections/introduction}

\vspace{-1pt}
\section{\STADTFI Algorithm}
\label{sec:algo}
\input{sections/algo}
\vspace{-5pt}
\section{Experiments}
\label{sec:experiments}
\vspace{-5pt}
\subsection{Runtime Comparisons across synthetically generated trees}
\label{sec:synthetic-benchmarks}
\input{sections/synthetic_experiments}

\subsection{Application 1 : Geodesic Optimal Transport via Tree metric approximation}
\label{sec:sinkhornexp}
\input{sections/sinkhornexp}
\subsection{Application 2 : Vision Transformers}
\label{sec:vit}
\input{sections/vit}

\section{Conclusion}
\label{sec:conclusion}
\input{sections/conclusion}

\subsection*{AI use statement}

We use generative AI tools to assist with language polishing in parts of the experimental section. Claude Code was additionally used to assist with the visual styling and formatting of figures. These tools were not used to generate experimental results, or perform the underlying analysis, or determine the scientific conclusions; all results are produced by the authors.

\subsection*{Ethics statement}

We do believe that this paper does not raise any questions regarding the Code of Ethics. It is mainly a theoretical contribution to the field. We do not use here any studies involving human subjects, do not include any harmful insights and are not aware of any potential conflicts of interest.

\subsection*{Reproducibility statement}
All presented experimental results can be accurately reproduced. We provide all the details regarding empirical evaluation in Sec. \ref{sec:experiments}. In addition to that, very detailed description of the \STADTFI algorithm is given in Sec. \ref{sec:algo}. We will provide code upon the acceptance of the paper.

\bibliography{iclr2027_conference}
\bibliographystyle{iclr2027_conference}

\appendix
\input{sections/appendix}




\end{document}

%% file: math_commands.tex
\usepackage{amsmath,amsfonts,bm}

\def\eqref#1{equation~\ref{#1}}

\def\1{\bm{1}}

\DeclareMathAlphabet{\mathsfit}{\encodingdefault}{\sfdefault}{m}{sl}
\SetMathAlphabet{\mathsfit}{bold}{\encodingdefault}{\sfdefault}{bx}{n}



%% file: sections/abstract.tex
We present a new class of near-linear algorithms for efficiently integrating general tensor fields defined on trees with distance dependent kernels,  the \textit{\textbf{St}ructure-\textbf{Ad}aptive \textbf{T}ree \textbf{F}ield \textbf{I}ntegrators} (\STADTFI[s]). \STADTFI[s] exploit the tree's underlying structure through decompositions built around path backbones and single vertex separators, and use two-dimensional fast Fourier transforms to compute interactions jointly. By exploiting this structural information, \STADTFI[s] achieve more computationally efficient integration than their regular efficient \textit{tree field integrators} (\TFI)  counterparts. We provide a detailed theoretical analysis of our proposed approach and complement it with an exhaustive empirical evaluation, ranging from speed tests on synthetic trees, through accelerated Sinkhorn-based relaxations of the Optimal Transport algorithms on real meshes, to Topological Attention Transformers for vision tasks. To the best of our knowledge, we provide some of the first results showing that efficient to compute and accurate relaxations of the geodesic Sinkhorn-based solutions of the Optimal Transport problem can be derived by applying fast TFI methods.

%% file: sections/introduction.tex
\section{Introduction and Related Work}
\label{sec:introduction}
\vspace{-5pt}
Consider the following problem. Let $G=(V,E,W)$ be an undirected graph with the set of vertices $V$, set of edges $E$ and set of weights $W$. Assume also that there exists a tensor field $\mathcal{F} :\mathrm{V} \rightarrow \mathcal{T}$ defined on the vertex-set $V$. Fix some $f:\mathbb{R} \rightarrow \mathbb{R}$. The goal is to conduct \textit{graph field integration} (GFI), i.e. for every vertex $v \in V$, calculate the following expression:
\vspace{-2pt}
\begin{equation}
\label{eq:integration}
y_{v} = \sum_{w \in \mathrm{V}} f(d(w,v))\mathcal{F}(w),
\vspace{-2pt}
\end{equation}
where $d(w,v)$ is the shortest-path distance from $w$ to $v$. The brute-force integration requires a pre-computation of the distance-matrix $\mathbf{D}=[d(w,v)]_{w \in \mathrm{V}}^{v \in \mathrm{V}} \in \mathbb{R}^{N \times N}$, where $N$ is the number of the vertices of $G$, followed by the calculations of time complexity $O(N^{2}\eta)$, where $\eta$ is time complexity of calculating the value of $\mathcal{F}$ in a particular vertex (we will assume that it is constant). Thus naive graph integration takes time complexity at least quadratic in $N$.

Surprisingly, this abstractly formulated problem is at the core of several computational backbones of modern machine learning algorithms operating on manifold data, with graphs providing natural discretizations. When $f \circ d : \mathrm{V} \times \mathrm{V} \rightarrow \mathbb{R}$ encodes a \textit{kernel} on the manifold, it becomes a core computational primitive for kernel methods on manifolds \citep{JayasumanaHSLH15, lafferty-1, lafferty-2, BelkinNS06, JayasumanaHSLH13, FeragenLH15, FeragenH16}, including Gaussian processes \citep{RosaBTR23, HolalkereBST25, AzangulovSTB24, AzangulovSTB24a, BorovitskiyATMD21, grfs-gaussian}. It also appears in neural operators for solving partial differential equations on manifolds \citep{ChenLSXHCLH25, JiaoYHL26, minglang-yin}, and is a natural primitive for energy-based and interactive-particle models \citep{joanna-marks, shiqin-tang}, molecular machine learning \citep{barrett2025, dumittan, Chau13} and graph neural networks \citep{ThurlemannR23, fognini}. Further instances include generalized distance-matrix multiplication in Sinkhorn methods for entropic optimal transport on manifolds \citep{choromanski2026near} and computations underlying relative positional encoding in point-cloud and vision Transformers \citep{choromanski2022block, choromanski2024fast}.



\vspace{-5pt}
\hspace{20pt}However, quadratic time complexity in the number of vertices is often prohibitive for massive graphs (e.g. encoding large point clouds). When $d$ in Eq.~\ref{eq:integration} is the standard Euclidean distance, fast multipole methods \citep{Greengard1987AFA, March2015AKF, FongD09, mcallister-fmm} provide efficient approximate computations for a large class of functions $f$. However, these methods do not directly apply to shortest-path distances, the subject of this paper.


Another line of work focuses on replacing general input graphs $G$ with their more structured proxies, more amenable to efficient GFI calculations. Trees are often chosen as particularly good candidates, due to the voluminous research on the so-called \textit{low distorion trees}, providing accurate approximation of the shortest-path distance metrics in the original graphs $G$ \citep{FakcharoenpholRT04, BartalFN22, BartalFN19} (e.g. logarithmic distortion on expectation for certain classes of randomly sampled trees). It was shown in \citep{choromanski2022block} that GFI for the unweighted trees and any function $f$ can be conducted in time $O(N \log^{2} N)$ with the use of the divide-and-conquer methods combined with the FFT-based algorithms for fast multiplications with Hankel matrices. As shown in \citep{choromanski2024fast}, this is the case also for the general weighted trees, if function $f$ is the so-called \textit{cordial function} (the class of cordial functions is large and includes in particular all rational functions, as well as products of exponential and polynomial maps). Similar results have been also obtained for the so-called \textit{bounded connected treewidth graphs} that can be thought of as generalizations of trees \citep{ChoromanskiSLZB23}, by leveraging both: FFT methods and the theory of graph separators. It was shown in \citep{choromanski2024fast} that RPE mechanisms induced by minimum spanning trees of the grid graphs corresponding to input images and leveraging efficient GFI methods for near-linear (in the number of tokens) computations involving low-rank attention ViTs significantly improves their downstream performance in standard classification tasks.

We introduce \textit{\textbf{St}ructure-\textbf{Ad}aptive \textbf{T}ree \textbf{F}ield \textbf{I}ntegrators} (\STADTFI[s]), a class of near-linear algorithms which make tree geometry an explicit part of the integration strategy. Our central observation is that a path backbone organizes cross-component interactions by backbone position and branch depth, enabling joint two-dimensional convolution. An adaptive rule selects between path backbones and single-vertex separators based on the current computational cost and resulting subproblems. We prove that the adaptive variant retains the $O(N\log^2 N)$ worst-case guarantee of existing efficient TFIs on arbitrary unweighted trees. Alongside this, we prove sharper bounds for several tree families through a decomposition-dependent analysis. 
We evaluate the practical advantages of \STADTFI[s] over the efficient TFIs of \citet{choromanski2024fast} through (i) synthetic benchmarks examining the influence of tree structure, (ii) accelerated Sinkhorn computations on tree approximations of real meshes, and (iii) applications to Topological Attention Transformers for vision tasks. To the best of our knowledge, our mesh experiments provide some of the first results demonstrating that fast TFI methods can support computationally efficient and accurate approximations to geodesic Sinkhorn-based optimal transport.

%% file: sections/algo.tex
\vspace{-6pt}
We introduce \emph{\textbf{St}ructure-\textbf{Ad}aptive \textbf{T}ree \textbf{F}ield \textbf{I}ntegrators} (\STADTFI[s]) for the tree-field integration problem studied by \citet[Eq.~(23)]{choromanski2022block}. Their balanced separator construction in Lemma~6.1 combines recursive computation within subtrees with efficient evaluation of cross-subtree interactions. \STADTFI\ retains this recursive principle while adapting the decomposition to the tree's geometric structure, rather than relying solely on balancing subtree-sizes. We specialize Eq.~\ref{eq:integration} to an unweighted tree $T=(V,E)$ with $n=|V|$ vertices. For clarity, we describe the algorithm for a scalar field $x\in\mathbb R^n$; tensor fields are handled coordinatewise. Given a kernel $f:\mathbb Z_{\geq0}\to\mathbb R$, the goal is to compute
\vspace{-11pt}
\begin{equation}
w_u=\sum_{v\in V}f\bigl(\dist_T(u,v)\bigr)x_v,
\qquad u\in V,
\label{eq:stad-target}
\vspace{-1pt}
\end{equation}
where $\dist_T(u,v)$ is the number of edges on the unique path between $u$ and $v$.

The algorithm recursively decomposes connected subtrees of $T$ by removing nonempty backbone paths (Section~\ref{sec:stad-backbone}). Source values are aggregated by (i) attachment position on the backbone and (ii) distance to the backbone into a shared interaction table (Section~\ref{sec:aggmasses}). Section~\ref{sec:stad-complexity} evaluates this table via joint 2D-FFT convolution. Cross-component contributions are combined with recursively computed within-component contributions, as detailed in Algo.~\ref{alg:stad-solve} in Appendix \ref{app:stad-recursion}.


\vspace{-5pt}
\hspace{20pt} Section~\ref{sec:stad-adaptive} presents an adaptive rule that selects between a diameter path and a singleton centroid based on the current cost and remaining subproblems.This proves two of our main results: (i) in the worst case, \STADTFI(adaptive) has an asymptotic runtime guarantee at least as good as \citet{choromanski2022block}'s, namely $O(n\log^2 n)$ (Theorem~\ref{thm:stad-adaptive-complexity}), and (ii) Theorem~\ref{thm:stad-adaptive-geometry} refines this guarantee using tree parameters, giving sharper bounds for structured families, including $O(n\log n)$ for equal-arm spiders (Corollary~\ref{cor:stad-adaptive-equal-spider}). Section~\ref{sec:synthetic-benchmarks} complements this analysis with empirical results. 


\subsection{Backbone Decomposition}
\label{sec:stad-backbone}
Decomposing trees into paths is a classical algorithmic technique, as exemplified by the path decompositions of \citet{sleator1981data}.  At a recursive call, let $S\subseteq T$ be a connected subtree. We choose a nonempty path $P=(p_1,\ldots,p_m)\subseteq S$, called its \emph{backbone}, allowing $m=1$.\footnote{\label{ftnt:adaptivebackbonechoice}$m=1$ gives a single-pivot interface, analogous to the shared pivot used in the proof of \citet[Lemma~6.1]{choromanski2022block}. A singleton backbone can be computationally preferable on trees with dense branching.} Removing the vertices of $P$ and their incident edges leaves connected components $B_1,\ldots,B_k$. Since $S$ is a tree, each component $B_\alpha$ is connected to $P$ by exactly one edge. We denote its attachment vertex by $p_{a(\alpha)}$, where $a(\alpha)\in\{1,\ldots,m\}$ is its \emph{anchor index}. Different components may share the same anchor (Fig. \ref{fig:stad-backbone}). For each $u\in V(S)$, its projection $\pi(u)$ onto the backbone is defined by Eq. \ref{eq:projection}.
\par\smallskip
\noindent
\begin{minipage}{\linewidth}

\begin{minipage}[c]{0.50\linewidth}
\centering
\hspace*{0.2\linewidth}%
\resizebox{0.70\linewidth}{!}{%
\begin{tikzpicture}[scale=0.90,
 vertex/.style={circle,fill=black,inner sep=1.65pt},
 spine/.style={circle,fill=blue!65!black,inner sep=2pt}]

 \node[spine,label=below:$p_1$] (p1) at (0,0) {};
 \node[spine,label=below right:$p_2$] (p2) at (1.5,0) {};
 \node[spine,label=above:$p_3$] (p3) at (3,0) {};
 \node[spine] (p4) at (4.5,0) {};
 \node[spine] (p5) at (6,0) {};
 \node[spine,label=below:$p_m$] (pm) at (7.5,0) {};

 \draw[very thick,blue!65!black]
   (p1)--(p2)--(p3)--(p4)--(p5);
 \draw[very thick,dashed,blue!65!black] (p5)--(pm);
 \node[blue!65!black] at (6.2,0.60) {backbone $P$};

 \node[vertex] (b11) at (0.8,-0.8) {};
 \node[vertex] (b12) at (0.5,-1.6) {};
 \draw (p2)--(b11)--(b12);
 \node at (0.05,-1.1) {$B_1$};

 \node[vertex] (b21) at (1.25,0.8) {};
 \node[vertex] (b22) at (0.7,1.5) {};
 \node[vertex] (b23) at (1.8,1.5) {};
 \draw (p2)--(b21)--(b22);
 \draw (b21)--(b23);
 \node at (0.15,1.2) {$B_2$};

 \node[vertex] (b31) at (3,-0.8) {};
 \node[vertex] (b32) at (2.55,-1.55) {};
 \node[vertex,label=below:$u$] (b33) at (3.45,-1.55) {};
 \draw (p3)--(b31)--(b32);
 \draw[thick] (b31)--(b33);
 \node at (3.9,-0.9) {$B_3$};

 \node[vertex] (b41) at (4.95,0.8) {};
 \node[vertex] (b42) at (5.4,1.5) {};
 \draw (p4)--(b41)--(b42);
 \node at (5.8,1.2) {$B_4$};

 \node[align=left] at (6.05,-1.2)
   {$\pi(u)=p_3$\\$d(u)=2$};

\end{tikzpicture}%
}
\end{minipage}\hfill
\begin{minipage}[c]{0.42\linewidth}
\begin{equation}
\pi(u)=
\begin{cases}
u, & u\in V(P),\\[3pt]
p_{a(\alpha)}, & u\in V(B_\alpha).
\end{cases}
\label{eq:projection}
\end{equation}
\end{minipage}
\vspace{-10pt}

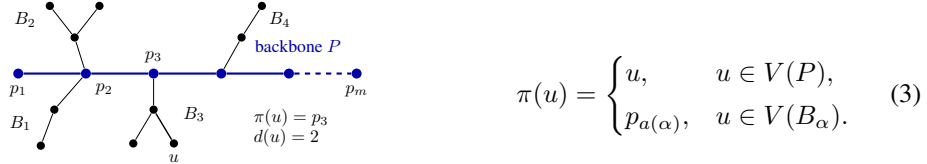
\captionof{figure}{Backbone decomposition of a tree. Removing the $P$ leaves off-backbone components attached at unique anchors. Distinct components may share an anchor, as $B_1$ \& $B_2$ do at $p_2$.}
\label{fig:stad-backbone}

\end{minipage}
\par\smallskip

\vspace{-5pt}
Its depth and the maximum off-backbone depth relative to P\footnote{\label{ftnt:diameterbackbonechoice}A natural geometric choice of backbone is therefore to select a path that minimizes this maximum off-backbone depth i.e. $P^\star \in \arg\min_{\substack{P\subseteq S}} H_P(S)$. For an unweighted tree, every diameter path achieves this minimum, as noted by \citet{dragan2015minimum}, motivating its use as a backbone candidate.} are
\vspace{-2pt}
\begin{equation}
d(u)=\dist_S\bigl(u,\pi(u)\bigr),
\qquad
H=H_P(S):=\max_{u\in V(S)}d(u).
\label{eq:stad-depth}
\vspace{-4pt}
\end{equation}
Thus, $d(u)=0$ precisely on the backbone and $d(u)\geq 1$ elsewhere. Additionally, for each component $B_\alpha$, we define $H_\alpha:=\max_{v\in V(B_\alpha)}d(v)$. In other words, $H=\max\bigl(\{0\}\cup\{H_\alpha:1\leq\alpha\leq k\}\bigr)$. Given $P$, Algo.~\ref{alg:decompose-components}, computes the components, anchors, projections, and depths in $O(|V(S)|+b_S)$ time, where $b_S$ is the number of original-tree edges with exactly one endpoint in $V(S)$.

\begin{lemma}[Distances through the backbone]
\label{lem:stad-distance}
If $u,v$ lie in different off-backbone components, or at least one lies on $P$, then, for $\pi(u)=p_s$ and $\pi(v)=p_t$, we have: $\dist_S(u,v)=d(u)+|s-t|+d(v)$.
\end{lemma}
\vspace{-5pt}
The unique path between such vertices passes through their projections. For vertices in the same component, the path can bypass the anchor; their true interaction is therefore handled recursively.  Hence, for $u\in V(B_\alpha)$, we split the sum into within-component and cross-component contributions:
\vspace{-5pt}
\begin{equation}
w_u^{(S)} = \underbrace{\sum_{v\in V(B_\alpha)} f\!\bigl(\dist_S(u,v)\bigr)x_v}_{w_u^{(B_\alpha)}\text{: computed recursively}} + \underbrace{
\sum_{v\in V(S)\setminus V(B_\alpha)} f\!\bigl(\dist_S(u,v)\bigr)x_v }_{\operatorname{Cross}_P(u)}.
\label{eq:stad-split}
\vspace{-5pt}
\end{equation}
For $u\in V(P)$, all contributions belong to $\operatorname{Cross}_P(u)=w_u^{(S)}$. Thus, we compute the \textit{within-component term} recursively and we next focus on evaluating the \textit{cross-component term}.

\subsection{Aggregating Fields for evaluating Cross Interactions}
\label{sec:aggmasses}
For every source vertex $v$, having anchor $t$ and depth $r$, we define the total mass in the branches at depth $r$ attached to vertex $p_t$ as: $A_r[t] := \sum_{\substack{v\in V(S), \pi(v)=p_t, d(v)=r}} x_v$ where $0\leq r\leq H$, and $1\leq t\leq m$. In particular, $A_0[t]=x_{p_t}$. For the query vertex $u$ at depth $d(u) = q$ and anchor $s$ we define the global \textit{backbone interaction table} as:
\vspace{-5pt}
\begin{align}
\label{eq:sumofFFTterms}
G_q[s] := \sum_{r=0}^{H} \sum_{t=1}^{m} f\bigl(q+|s-t|+r\bigr)\,A_r[t].
\end{align}

\begin{proposition}[FFT acceleration of the computation of $G$]
\label{prop:stad-1d-fft}
Let $H\geq0$ and $m\geq1$, and suppose that each kernel value $f(d)$ can be evaluated in constant time. Given the aggregates $A_r[t]$, brute-force evaluation of all entries $G_q[s]$, for $0\leq q,r\leq H$ and $1\leq s,t\leq m$, requires $\Theta((H+1)^2m^2)$ time. Using fast Fourier transform (FFT) accelerates this computation 
to $O((H+1)^2m\log m)$\footnote{The criterion to choose backbone (in Footnote \ref{ftnt:diameterbackbonechoice}) is directly motivated by the runtime of Algo. \ref{alg:global-anchor}. Therefore, $H$ controls the number of depth levels that appear in the convolutional computation. Choosing a backbone with a small $H$ reduces the size of the anchor-depth arrays and keeps the cross-interaction term efficient.}.
\end{proposition}

Moreover, if $u\in B_\alpha$, then the global sum $G_q[s]$ in Eq. \ref{eq:sumofFFTterms} also includes artificial within-component interactions i.e. the source vertices residing in $B_\alpha$ itself. We, therefore, need to add a correction for each component $B_\alpha$. Hence, locally, for $0\le r,q\le H_\alpha$, we define the depth profile $b_\alpha[r]$ and the correction term $E_\alpha[q]$ as follows:
\vspace{-8pt}
\begin{equation}
 b_\alpha[r]=\sum_{\substack{v\in V(B_\alpha)\\d(v)=r}}x_v,
 \qquad E_\alpha[q]=\sum_{r=1}^{H_\alpha}f(q+r)b_\alpha[r].
 \label{eq:stad-E}
 \vspace{-5pt}
\end{equation}

\begin{lemma}[Cross-component correction]
\label{lem:stad-cross}
For each $u\in V(S)$, we compute cross-interactions as
\vspace{-6pt}
\begin{equation}
\operatorname{Cross}_P(u)=
\begin{cases}
G_0[s],
& u=p_s,\\[3pt]
G_{d(u)}[a(\alpha)]-E_\alpha[d(u)],
& u\in V(B_\alpha).
\end{cases}
\label{eq:stad-cross-correction}
\vspace{-7pt}
\end{equation}
\end{lemma}

Subtracting $E_\alpha$ removes exactly the sources in $B_\alpha$ from the surrogate global sum. Recursion restores their contributions using the actual internal distances. 

\subsection{\STADTFI (Backbone with FFT) : Decomposition-Dependent Runtime Bound}
\label{sec:stad-complexity}

Following Proposition \ref{prop:stad-1d-fft}, Algo.~\ref{alg:global-anchor} performs a separate one-dimensional convolution for each depth pair $(q,r)$. These computations can be combined because the kernel depends on depth through $q+r$ and on backbone position through $s-t$. To express both the dependencies through differences, reverse the source-depth index $r$, by substituting $r=H-\rho$ which makes $A_r[t] := A_{H-\rho}[t] = \widetilde A[\rho,t]$. 
Now, define the two-dimensional kernel $K[a,b] := f(H+a+|b|)$, for  $-H\leq a\leq H$, and $|b|\leq m-1$ and extend both $K$ and $\widetilde A$ by zero outside their specified ranges. Then Eq.~\ref{eq:sumofFFTterms} becomes
\vspace{-8pt}
\begin{equation}
\begin{aligned}
G_q[s]
&=
\sum_{\rho=0}^{H}\sum_{t=1}^{m}
K[q-\rho,s-t]\widetilde A[\rho,t]
=
(K*\widetilde A)[q,s].
\end{aligned}
\label{eq:stad-2d-convolution}
\vspace{-3pt}
\end{equation}
Thus, all query depths and backbone positions can be processed jointly through one two-dimensional (2D) convolution, removing the explicit iteration over depth pairs (as described in Algo. \ref{alg:global-anchor-2dfft}).


\begin{lemma}[Joint 2D-FFT acceleration of $G$]
\label{lem:stad-2d-fft}
Under the assumptions of Proposition~\ref{prop:stad-1d-fft}, and the joint 2D convolution form of all entries of $G$, its FFT evaluation takes  quasi-linear arithmetic operations $O(M\log M)$  where $M=(H+1)m$ is the number of entries in the table $G$.
\end{lemma}

Lemma~\ref{lem:stad-cross} establishes the exact cross-component contributions, while Lemma~\ref{lem:stad-2d-fft} give the costs of computing the backbone interaction table. Combining these results gives a decomposition-dependent bound for the recursive Algo.~\ref{alg:stad-solve}.


\begin{theorem}[Exactness and decomposition-dependent cost]
\label{thm:stad-decomposition-cost}
Let $T$ be an unweighted tree and assume constant-time kernel evaluation. At each recursive call on a connected subtree $S\subseteq T$, assume that a nonempty backbone is selected in $O(|V(S)|+b_S)$ time. Let $\mathcal R$ denote the set of resulting recursive calls, including singleton base cases, and write $M_S=(H_S+1)m_S$ and $H_{S,\alpha}$ for the depth of component $B_\alpha$. Using the exact component correction Eq.~\ref{eq:stad-E} and the joint two-dimensional FFT in Lemma~\ref{lem:stad-2d-fft}, the algorithm computes the tree-field integral exactly in exact arithmetic with total running time $O\!\left(\sum_{S\in\mathcal R}|V(S)| +\sum_{S\in\mathcal R}M_S\log(2M_S) +\sum_{S\in\mathcal R}\sum_{\alpha=1}^{k_S}  H_{S,\alpha}\log(H_{S,\alpha}+1) \right)$. 
\end{theorem}

Backbone selection, thus, influences both the work performed at the current call and the subproblems left for subsequent processing. A longer backbone removes more vertices, but its table size $M_S=(H_S+1)m_S$ also depends on the maximum off-backbone depth. The joint FFT does not, by itself, guarantee $M_S=O(|V(S)|)$.  Appendix~\ref{app:stad-structural-cases} theoretically examines these effects on representative tree families and the following example is one of those cases.
\begin{proposition}[Backbone choice on equal-arm spiders]
\label{prop:stad-spider-comparison}
Let $T$ be a spider with $s\geq3$ arms of $h\geq1$ edges each, so that $n=1+sh$. Suppose that each remaining arm is processed using its entire path as the backbone. Choosing the singleton center at the initial call gives running time $O\!\left(n\log(h+1)\right),$ whereas choosing a diameter backbone gives $O\!\left((n+h^2)\log(h+1)\right).$ For fixed $s$, the center construction therefore takes $O(n\log n)$ time, while the diameter construction requires $\Omega(n^2)$ time and storage merely to materialize its dense interaction table.
\end{proposition}

Both constructions have maximum off-backbone depth $h$ and at most two recursion levels. Their initial interaction tables, however, contain $h+1$ and $(h+1)(2h+1)$ entries, respectively. Thus, minimizing $H =h$ alone does not control the cost. These observations motivate a selection that accounts for (i) backbone length, (ii) component depths, and (iii) remaining component sizes.

\subsection{\STADTFI{} (Adaptive with FFT): Near-Linear Runtime}
\label{sec:stad-adaptive}
\vspace{-5pt}
To control the decomposition-dependent cost in Theorem~\ref{thm:stad-decomposition-cost}, we now select a backbone using both its local computational cost and the sizes of the remaining components. For each connected subtree $S$ with $|V(S)|\geq2$, we consider two candidate backbones $\mathcal P(S)=\{P_{\mathrm{diam}},(c_S)\},$ where $P_{\mathrm{diam}}$ is a diameter path of $S$ and ($c_S$) is a backbone consisting of a single centroid vertex. The diameter candidate removes an extended path through thte tree, while deleting the centroid leaves components containing at most $|V(S)|/2$ vertices.

\vspace{-5pt}
\hspace{20pt} For a candidate backbone $P \in \mathcal P(S)$, we write $m=|V(P)|$, $H=H_P(S)$ and $M=(H+1)m$. Let $B_1,\ldots,B_k$ be the connected components of $S\setminus V(P)$, with $n_\alpha=|V(B_\alpha)|$. Motivated by the computational cost in Theorem~\ref{thm:stad-decomposition-cost}, and using natural logarithm, we assign the following \textit{score} to the candidate $P$: 
\vspace{-7pt}
\begin{equation}
\begin{aligned}
\mathcal C(S,P)
={}&|V(S)|
+M\log\!\bigl(\max\{M,2\}\bigr) 
+\sum_{\alpha=1}^{k}H_\alpha\log\!\bigl(H_\alpha+1\bigr)
+\sum_{\alpha=1}^{k} n_\alpha\log(n_\alpha+1).
\end{aligned}
\label{eq:stad-score}
\vspace{-11pt}
\end{equation}

The first three terms model the nonrecursive computation: processing vertices, computing $G$ by the 2D convolution of Lemma~\ref{lem:stad-2d-fft}, and evaluating the component corrections. The final term penalizes large remaining components to account approximately for subsequent recursive work. Therefore, we select the backbone according to : $P_S=(c_S)$ if $\mathcal C(S,(c_S))<\mathcal C(S,P_{\mathrm{diam}})$, and $P_S=P_{\mathrm{diam}}$ otherwise.
The rule is applied at every nontrivial recursive call where the selection depends only on the tree structure: evaluating the candidates requires their geometric statistics, and FFT computations are performed only after the backbone has been selected.

\begin{lemma}[Adaptive selection and local cost]
\label{lem:stad-selection}
The two candidates, their scores Eq.~\ref{eq:stad-score}, and the choice can be computed in $O(|V(S)|+b_S)$ time. The selected backbone satisfies $M_S=O(|V(S)|)$. Consequently, all nonrecursive work at the call, including selection and original-adjacency scanning, takes $O(|V(S)|\log(|V(S)|)+b_S)$ time. 
\end{lemma}



Controlling the total cost additionally requires accounting for repeated processing across recursive calls. Define $\ell_S=\max\!\left( \{0\}\cup\{|V(B)|:B\text{ is a child of }S\}  \right)$, and  $r_S=|V(S)|-\ell_S$. Thus $r_S$ counts the vertices outside one largest child, including the vertices removed in the backbone.

\begin{theorembox}
\begin{theorem}[Adaptive worst-case guarantee in near-linear time]
\label{thm:stad-adaptive-complexity}
Let $T$ be an unweighted tree with $n$ vertices. At each nontrivial recursive call, select the backbone by using Eq.~\ref{eq:stad-score}, compute its interaction table using the joint 2D FFT Eq.\ref{eq:stad-2d-convolution}, and compute each correction at its component-local depth Eq.\ref{eq:stad-E}. Under the stated computational assumptions, the algorithm computes the integral Eq.\ref{eq:stad-target} in a total running time of $O\!\left(n\log^2 n\right)$.
\end{theorem} 
\end{theorembox}

The key step, proved in Lemma~\ref{lem:stad-call-charge} in Appendix~\ref{app:proofs}, is the stronger local estimate of $|V(S)|+M_S\log(2M_S) +\sum_{\alpha=1}^{k_S}H_{S,\alpha}\log(H_{S,\alpha}+1)  =O\!\left(r_S\log(|V(S)|)\right).$ It follows from the comparison of the complete candidate scores.
This universal bound can be refined by retaining the geometry of the backbones selected throughout the recursion.
\begin{theorembox}
\begin{theorem}[Adaptive complexity in terms of backbone geometry]
\label{thm:stad-adaptive-geometry}
Under the assumptions of Theorem~\ref{thm:stad-adaptive-complexity}, let $\mathcal R$ be the resulting recursive calls and write $m_S=|V(P_S)|$, $H_S=H_{P_S}(S)$, and $M_S=(H_S+1)m_S$, with $m_S=1$ and $H_S=0$ for singleton calls. The total running time satisfies $ T_{\mathrm{adaptive}} = O\!\left(\left(n+\sum_{S\in\mathcal R}m_SH_S\right)\log n \right).$
\end{theorem}
\end{theorembox}

The quantity $\sum_S m_SH_S$ retains the contribution of each selected backbone instead of replacing all depths by their maximum. In particular, the bound becomes $O(n\log n)$ whenever this sum is $O(n)$. Experiments in Section~\ref{sec:synthetic-benchmarks} provide empirical support for this analysis by examining runtime across tree families and illustrate the practical speedups.

\begin{corollary}[Adaptive integration on equal-arm spiders]
\label{cor:stad-adaptive-equal-spider}
Let $T$ be a spider with $s\geq3$ arms of $h\geq1$ edges each, so that $n=1+sh$. The adaptive rule selects the singleton center at the initial call and the entire path at each remaining arm. Its total running time is $O\!\left(n\log(h+1)\right)\subseteq O\!\left(n\log n\right)$.
\end{corollary}

For kernels of the form : sums of $J$ exponentials i.e. $f(d)=\sum_{\ell=1}^{J}c_\ell\lambda_\ell^d$, Appendix~\ref{sec:stad-specialized} replaces FFT with sweep recurrences and yields an $O(Jn\log n)$ adaptive bound.

%% file: sections/synthetic_experiments.tex
We benchmark scalar tree-field integration with the kernel $f(d)=1/(1+d)$. We compare brute force evaluation and the centroid-based FTFI method~\cite{choromanski2024fast} against the variants of \STADTFI defined in Section \ref{sec:algo}. Fig.~\ref{fig:tree-family-runtime} shows four representative tree families chosen to isolate the effects of (i) the backbone decomposition, (ii) the 2D convolution, and (iii) adaptive separator selection. Appendix~\ref{app:synthetic-benchmarks} shows complementary results for other trees.

\vspace{-7pt}
\begin{figure}[htbp]
    \centering
    \includegraphics[width=0.9\linewidth]{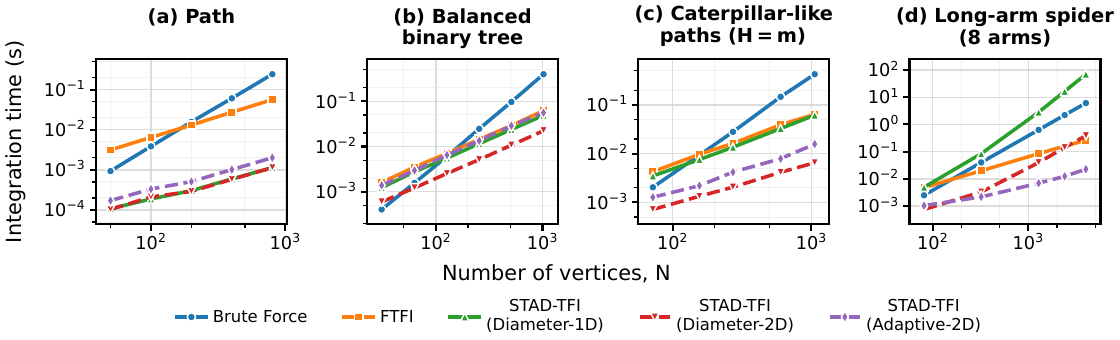}
    \vspace{-8pt}
    \caption{Integration time as a function of number of vertices for kernel $f(d)=1/(1+d)$. Diameter-1D versus Diameter-2D isolates the convolution strategy; Diameter-2D versus Adaptive-2D measures the effect of separator selection. See Table \ref{tab:synthetic-largest} in Appendix \ref{app:synthetic-largest} for detailed numbers.}
    \label{fig:tree-family-runtime}
    \vspace{-12pt}
\end{figure}
\begin{enumerate}[noitemsep, topsep=0pt, label=\roman*)]
    \item \textit{Effect of backbone decomposition.} Fig.\ref{fig:tree-family-runtime}(a) shows that on a path of length $N=800$, Diameter-2D reduces runtime from $59.31$\,ms for FTFI to $1.15$\,ms, a $51.42\times$ speedup. Since Diameter-2D is only  $1.03\times$ faster then Diameter-1D, nearly all of the improvement comes from exploiting the full path as a backbone. 
    \item \textit{Effect of 2D convolution.} On the largest balanced binary tree ($N=1023$) Fig.\ref{fig:tree-family-runtime}(b) and caterpillar with long hanging paths ($N=1056$) Fig.\ref{fig:tree-family-runtime}(c), replacing 1D convolution with the joint 2D FFTs yields $2.08\times$ and $9.40\times$ speedups, respectively.These comparisons isolate the computational gain from 2D convolution.
    \item \textit{Effect of adaptive separator selection.} On the eight-arm spider Fig.\ref{fig:tree-family-runtime}(d) with $N=4001$: Adaptive-2D takes $20.33$\,ms, compared with $259.52$\,ms for FTFI, giving $12.77\times$ speedups, attributable to choosing a more suitable separator. However this benefit is structure dependent: on the largest path, balanced binary tree, and caterpillar, it takes $1.81\times$, $2.38\times$, and $2.50\times$ as long as Diameter-2D, respectively, exposing the constant overhead of adapative selection when the diameter backbone is already well suited to the tree. 
\end{enumerate}

%% file: sections/sinkhornexp.tex
In this section, we study a practical application in which repeated kernel-vector products are a computational bottleneck: \textit{entropically regularized optimal transport (OT) on weighted mesh graphs}~\citep{montesuma2024recent}. For a mesh graph $G$ with $n$ vertices and Euclidean edge weights, the transportation cost is its graph-geodesic distance $d_G$. We denote $D_G=\left[d_G(i,j)\right]_{i,j=1}^{n}$ as the corresponding all-pairs graph geodesic distance matrix. The corresponding Sinkhorn kernel is $K_G(i,j)=\exp(-d_G(i,j)/\varepsilon)$, where $\varepsilon>0$ is the regularization parameter. Each Sinkhorn iteration requires two kernel-vector products~\citep{cuturi2013sinkhorn}. Although $G$ is sparse, $K_G$ is dense, requiring $O(n^2)$ storage and $O(n^2)$ work per iteration under explicit evaluation~\citep{choromanski2026near}.

Replacing $d_G$ with the distance $d_T$ on a spanning tree $T\subseteq G$ turns each kernel-vector product into $(K_Tx)_i = \sum_{j=1}^{n} \exp(-d_T(i,j)/\varepsilon)x_j,$ which can be evaluated using our tree-field integrators. This replacement also changes the transportation geometry and can affect both the transport plan ($\Pi$) and its cost ($D$). We therefore evaluate separately \textit{(i) the accuracy of the spanning-tree metric approximation} and \textit{(ii) the runtime of different integrators within Sinkhorn on a fixed tree}. Our experiments use 28 Thingi10K meshes from the collection evaluated by GenusSink~\citep{choromanski2026near}, containing $2{,}771$-$12{,}472$ vertices. Appendix~\ref{app:thingi-details} provides the optimal-transport formulation, Sinkhorn updates, distribution construction, and additional experimental details and results.

\subsubsection{Thingi10K: Tree-Metric Approximation}
\label{thingi_treeapprox}
\vspace{-5pt}
For every mesh graph $G$, we compare seven spanning-tree constructions $T \subseteq G$, described in Appendix \ref{app:thingi-trees}. Within each mesh, all tree constructions are evaluated using the same source ($a$) and target ($b$) distributions and regularization parameter ($\varepsilon$); their construction is detailed in Appendix \ref{app:thingi-data}. Since a spanning tree retains only $n-1$ of the original graph edges, we have $d_T(i,j) \geq d_G(i,j)$ for all vertex pairs $i,j \in V$. 

\vspace{-5pt}
\hspace{20pt }To evaluate the optimal transport approximation results due to the shift in geometry from a mesh to a tree, we compute the reference plan $\Pi_G$ using $D_G$ and, independently for each tree, $\Pi_T$ using $D_T$. We evaluate both plans under the original mesh metric: $C_G(\Pi)=\langle \Pi,D_G\rangle$, with $C_{\mathrm{ref}}=C_G(\Pi_G)$. We report:
\vspace{-5pt}
\begin{itemize}[noitemsep, topsep=0pt]
    \item \textbf{Transport-cost error:} $E_{\mathrm{cost}}(T)=|C_G(\Pi_T)-C_{\mathrm{ref}}|$, and the relative OT error is  $E_{\mathrm{rel}}(T)=E_{\mathrm{cost}}(T)/C_{\mathrm{ref}}$. Evaluating $\Pi_T$ under $D_G$ measures whether optimizing with the approximate tree metric produces a good transport plan for the original mesh problem.; normalization supports comparisons across geometric scales.

    \item \textbf{Plan total variation:} $E_{\mathrm{TV}}(T)= \frac{1}{2}\sum_{i,j}|(\Pi_T)_{ij}-(\Pi_G)_{ij}|$. For unit-mass plans, this lies in $[0,1]$, with zero indicating identical plans.  Since both \(P_T\) and \(P_G\) are probability transport plans with total mass one, the factor \(1/2\) avoids double-counting the mass that must be reassigned when moving from one plan to the other.

    \item \textbf{Relative Sinkhorn-kernel error:} $E_K(T)= \|K_T-K_G\|_F/\|K_G\|_F$.
\end{itemize}

\vspace{-1pt}
\begin{table}
\centering
\caption{
Tree approximations for geodesic Sinkhorn on Thingi10K. Results are reported as mean $\pm$ standard deviation across the evaluated meshes. Lower is better for all metrics.
}
\label{tab:thingi_ot_tree}
\vspace{-10pt}
\small
\setlength{\tabcolsep}{4.5pt}
\renewcommand{\arraystretch}{1.15}

\begin{tabular}{lcccc}
\toprule
\textbf{Tree Approximation}
& \shortstack{\textbf{Rel. OT Error}}
& \shortstack{\textbf{Plan TV}}
& \shortstack{\textbf{Kernel Error}} \\
\midrule

Minimum Spanning Tree (MST)
& $0.028 \pm 0.027$
& $0.329 \pm 0.155$
& $0.680 \pm 0.135$\\

Random Shortest-Path Tree (Random SPT)
& $0.033 \pm 0.047$
& $0.183 \pm 0.082$
& $0.677 \pm 0.050$ \\



Diameter/Backbone Spanning Tree
& $0.021 \pm 0.020$
& $\mathbf{0.124 \pm 0.063}$
& $0.599 \pm 0.155$ \\

\citet{alon1995graph} (AKPW) Tree
& $\mathbf{0.018 \pm 0.025}$
& $0.236 \pm 0.101$
& $\mathbf{0.573 \pm 0.100}$ \\

\citet{abraham2012using} Tree
& $0.026 \pm 0.048$
& $0.211 \pm 0.100$
& $0.677 \pm 0.072$ \\

\bottomrule
\end{tabular}
\vspace{-15pt}
\end{table}


Table~\ref{tab:thingi_ot_tree} shows that Alon-Karp-Peleg-West (AKPW) based tree approximation achieves the lowest mean \textit{\textbf{relative transport-cost error}}, $1.8\%$ versus $2.8\%$ for the MST, an approximately $35.7\%$ reduction. The Backbone Tree produces the \textit{\textbf{closest transport plans}}: its mean Plan TV is $0.124$, compared with $0.329$ for the MST and $0.236$ for AKPW, reductions of approximately $62.3\%$ and $47.5\%$, respectively, while maintaining a low mean relative cost error of $2.1\%$. AKPW also gives the \textit{\textbf{most accurate kernel approximation}}, reducing the mean kernel error from $0.680$ for the MST to $0.573$, a $15.7\%$ improvement. The Backbone Tree has the second-lowest mean kernel error at $0.599$, an $11.9\%$ improvement over the MST.\footnote{Kernel error measures discrepancies across all vertex pairs, whereas transport cost depends on the mass allocated after Sinkhorn scaling. Thus, large kernel errors can coexist with small transport-cost errors.} 

\vspace{-5pt}
\hspace{20pt} The MST's combination of $2.8\%$ cost error, $0.329$ Plan TV and a kernel error of $0.680$ demonstrates that similar total transport costs $C$ can conceal substantial differences in mass allocation and kernel approximation. Since we focus on transport-cost minimization and all tested trees achieve comparable OT error, we choose the MST for downstream runtime experiments to assess whether \STADTFI[s] can achieve this accuracy with reduced computation time.


\vspace{-1pt}
\subsubsection{Thingi10K: Sinkhorn Runtime Scaling}
\label{sec:thingi_runtime}
\vspace{-5pt}
We now benchmark Sinkhorn on MSTs of all 28 Thingi10K meshes to isolate the effect of the field integrator from tree construction. We compare eight implementations: two quadratic brute-force baselines; three FTFI-based variants, FTFI~(\textit{1D-FFT}), FTFI~(\textit{SpclK-Centroid}) \citep{choromanski2024fast} and FTFI~(\emph{SpclK-LinearExactDP}) \citep[Lemma~3.5]{choromanski2022block}; and three \STADTFI{} variants, \emph{Adaptive-2D-FFT}, \emph{SpclK-Diameter}, and \emph{SpclK-Adaptive} (Appendix~\ref{app:thingi-runtime}). The centroid baseline uses recursive decomposition, whereas the linear exact-DP baseline computes the integration through two passes over the entire tree in $O(N)$ time. 

\vspace{-6pt}
\hspace{20pt}Both FFT-based methods quantize (Appendix~\ref{app:fft_quantization}) the real-valued edge weights onto an integer grid to express the cross-interaction sums as discrete convolutions that can be evaluated using FFTs. The specialized-kernel methods operate directly on the original weighted tree metric, since the Sinkhorn kernel $f_\varepsilon(d)=e^{-d/\varepsilon}$ is the $J=1$ case of Appendix~\ref{sec:stad-specialized}, with $c_1=1$ and $\lambda_1=e^{-1/\varepsilon}$.




\begin{figure}[htbp]
    \centering
    \includegraphics[width=0.9\linewidth, trim = 0 31 0 10, clip]{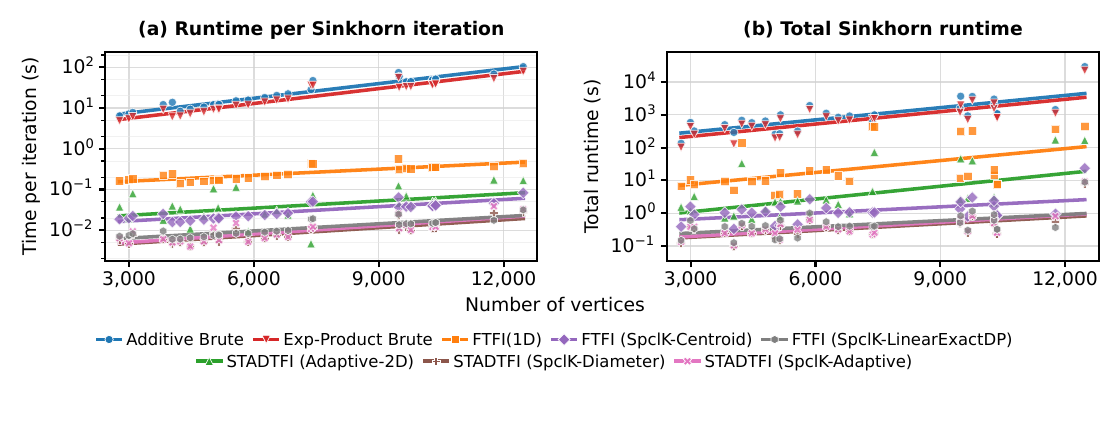}
    \vspace{-7pt}
    \caption{Runtime comparison as a function of the number of vertices. Speedup Numbers are in Table \ref{tab:sinkhorn-speedups} in Appendix~\ref{app:thingi-runtime}. \STADTFI(SpclK-Diameter) is $\approx3.35\times$ its FTFI counterpart.}
    \label{fig:sinkhorn_runtime_comparison}
    \vspace{-10pt}
\end{figure}

Fig.~\ref{fig:sinkhorn_runtime_comparison} reports per-iteration and total Sinkhorn runtimes with the corresponding speedup numbers summarized in Table~\ref{tab:sinkhorn-speedups}. The brute-force methods are slowest; exponential factorization improves their constant factor but preserves quadratic complexity. \STADTFI{}~(\emph{Adaptive-2D-FFT}) achieves a geometric-mean speedup of $6.32\times$ in total Sinkhorn runtime over FTFI~(\emph{1D-FFT}), reflecting the combined effects of separator selection and convolution strategy. The exponential-specific methods are generally fastest, using the factorization $e^{-(a+b)/\varepsilon} =e^{-a/\varepsilon}e^{-b/\varepsilon}$, where \STADTFI~(\emph{SpclK}) variants further reduce runtime relative to the FTFI~(\emph{SpclK}) variants. Total runtimes broadly preserve this ordering which also reflect iteration counts, that varies with quantization.

%% file: sections/vit.tex
\providecommand{\ftfi}{\textsc{FTFI}}
\providecommand{\dtree}{d_T}
\providecommand{\pending}[1]{\textcolor{red}{#1}}
\providecommand{\STADTFI}[1][]{\textsc{StAd-TFI}#1\xspace}
\vspace{-5pt}
We next study whether the structural advantages of \STADTFI{} translate into efficient attention for vision models. Topological Vision Transformers \citep{choromanski2024fast} incorporate spatial information through a learnable mask $M_{ij}=\exp\!\left(u\,\dtree(i,j)^2+v\,\dtree(i,j)+w\right)$, where $\dtree(i,j)$ is the hop distance between tokens $i$ and $j$ on a spanning tree of the patch grid, and $(u,v,w)$ are learned separately for each layer. The tree determines the distances used by the model and the integrator determines  the cost of evaluating the resulting mask. We, therefore, examine both (i) \textit{the predictive performance of different tree masks} and (ii) \textit{the computational cost of applying them}. For Performer attention, stacking the features $x_j=[\phi(k_j)v_j^\top\,|\,\phi(k_j)]$ allows the masked numerator and normalizer to be computed jointly through a single tree-field integration (Eq.~\ref{eq:stad-target}). The resulting field has multiple channels ($C$) per head (in our experiments, $C=4{,}160$). Consequently, for a fixed $C$, a subquadratic integrator yields masked linear attention with subquadratic complexity in the sequence length $N$ (Appendix~\ref{app:reduction}).

\vspace{-5pt}
\hspace{20pt} We conduct image-classification experiments on two datasets, CIFAR-100 and Tiny-ImageNet ($200$ classes). We \textbf{train} ViT-S models \textbf{from scratch} for 100 epochs using one fixed recipe at $224\times224$ resolution. For CIFAR-100, we evaluate patch sizes $16$ and $7$, corresponding to $N=196$ and $N{=}1{,}024$ tokens respectively. For Tiny-ImageNet, we use patch size $16$, giving $N=196$. The linear-attention baseline uses $\phi(x)=\mathrm{elu}(x)+1$ \citep{choromanski2021performer}. Following the \emph{synced} parameterization of \citet{choromanski2024fast}, mask parameters are shared across heads within each layer and initialized to zero, so that masked and unmasked models coincide at initialization. Full training details and the seed protocol appear in Appendix~\ref{app:training}.

\vspace{-5pt}
\hspace{20pt} We compare the configuration of \citet{choromanski2024fast}, which pairs \ftfi{} with a random spanning tree, with \STADTFI{} applied to a serpentine path (Fig.~\ref{fig:trees}).  Both span the unweighted $4$-neighbor patch grid, but induce different pairwise distances. The serpentine path is computationally attractive since it allows the entire tree to serve as the backbone, eliminating further recursions. Its suitability for prediction must nevertheless be assessed empirically.


\subsubsection{Accuracy Measurements}
\label{sec:vit-acc}
\vspace{-5pt}
\begin{table}[t]
\centering
\caption{Top-1 accuracy (\%, mean~$\pm$~standard deviation over
seeds) for ViT-S, \textbf{trained from scratch} under a single fixed recipe. Seed counts and protocol are given in Appendices~\ref{app:training} and~\ref{app:vit-tree-choice}.}
\vspace{-10pt}
\label{tab:vit-acc}
\begin{tabular}{lccc}
\toprule
& \multicolumn{2}{c}{CIFAR-100} & Tiny-ImageNet \\
\cmidrule(lr){2-3}\cmidrule(lr){4-4}
Configuration & $N=196$ & $N=1{,}024$ & $N=196$ \\
\midrule
Linear, unmasked            & $49.94\pm0.26$ & $51.89$
  & $36.77\pm0.49$ \\
Softmax, unmasked           & $51.71\pm0.26$ & $53.19\pm0.49$
  & $37.96\pm0.78$ \\
Linear, masked (random)     & $\mathbf{55.27\pm0.68}$
  & $\mathbf{58.02\pm0.08}$ & $\mathbf{40.14\pm0.33}$ \\
Linear, masked (serpentine) & $\mathbf{54.14\pm1.37}$
  & $\mathbf{57.84\pm0.89}$ & $\mathbf{41.35\pm0.99}$ \\
\bottomrule
\end{tabular}
\vspace{-2pt}
\end{table}


Table~\ref{tab:vit-acc} shows that both the masked linear-attention models outperform the unmasked linear and softmax baselines in all three settings. In particular, the serpentine model exceeds unmasked softmax by $2.43\%$ and $4.65\%$ on CIFAR-100 at $N=196$ and $N=1{,}024$, respectively, and by $3.39\%$ on Tiny-ImageNet. Compared with the random tree, the accuracy trails by a mere $\approx 0.18\%$-$1.13\%$ on CIFAR-100, but leads by $1.21\%$ on Tiny-ImageNet. Thus, these comparisons support three observations: (i) \textit{higher accuracy does not require a more complex spanning tree}, as demonstrated by varying the tree while keeping linear attention fixed; (ii) \textit{masked attention outperforms unmasked baselines} as shown here, and is consistent across multiple tree choices reported  in Appendix~\ref{app:vit-tree-choice}, and (iii) under the same mask, \textit{linear and softmax attention achieve comparable accuracy} (Appendix~\ref{app:training}), supporting the use of subquadratic attention.

\subsubsection{Runtime Comparisons}
\label{sec:vit-runtime}
\vspace{-5pt}
We next measure the cost of applying these masks by comparing the two integrator-tree configurations motivated above. 
The large-$N$ experiments extend the masked-attention workload; the accuracy results remain those at the trained resolutions.

Table~\ref{tab:vit-runtime} reports the runtime of the masked-attention workload with $C=4{,}160$ channels per head, ranging from the trained-model scale of $N=1{,}024$ tokens to $N{=}262{,}144$ tokens, corresponding to a $512\times512$ patch grid. We make two observations: 
\vspace{-5pt}
\begin{enumerate}[noitemsep, topsep=0pt]
    \item Dense evaluation is fastest at the trained scale, where masking costs are negligible, but it scales poorly: at $N{=}16{,}384$ it is already $4.5\times$ slower than \STADTFI{} ($36.1$ versus $8.0$\,s), and beyond approximately $4\times10^4$ tokens its $N\times N$ mask exceeds commodity memory. These cases are marked ``---'' in Table~\ref{tab:vit-runtime}. For tree-based masks, which extend beyond grids to meshes and general graphs, TFIs are consequently the only exact option at large $N$.\footnote{For masks defined by grid displacement, block-Toeplitz masking \citep{choromanski2022block} is an efficient alternative, does not extend to tree metrics.}
    \item Among tree-field integrators, \STADTFI{} on the serpentine path is the fastest configuration at every $N{\geq}4{,}096$, running $2.6$-$3.1\times$ faster than \ftfi{} on the latter's most favorable tree. On the path, \STADTFI{} evaluates the mask in $O(CN\log N)$ time, a multiplicative $O(\log N)$ overhead relative to unmasked linear attention. Table~\ref{tab:vit-crossover} in Appendix~\ref{sec:tc-runtime} compares runtimes across alternative tree structures using identical inputs.
\end{enumerate}



On a path, the distance-kernel matrix is Toeplitz, and admits established $O(N\log N)$ FFT evaluation per channel, including in attention mechanisms \citep{luo2021stable,choromanski2022block}. This makes paths a natural choice for efficient masking; among the paths screened, the serpentine traversal also attained the lowest stretch (complete tree study in Appendix~\ref{app:vit-tree-choice}). Therefore, the key contribution of \STADTFI{} is to recover this efficient path computation within a single adaptive integrator that also handles branching trees efficiently. 


\begin{table}[t]
\centering
\caption{Runtime (seconds) of the masked-attention workload as a function of $N$ (median of at least five runs;
protocol in Appendix~\ref{app:runtime}). \ftfi{} uses its most
favorable configuration (random tree).}
\label{tab:vit-runtime}
\vspace{-10pt}
\begin{tabular}{lccccc}
\toprule
Method (tree) & $1{,}024$ & $4{,}096$ & $16{,}384$ & $65{,}536$ &
$262{,}144$ \\
\midrule
Dense ($\Theta(N^2)$)   & \textbf{0.15} & 2.24 & 36.07 & --- & --- \\
\ftfi{} (random)        & 1.10 & 4.93 & 25.20 & 99.99 & 424.97 \\
\STADTFI{} (serpentine)  & 0.39 & \textbf{1.76}
  & \textbf{8.02} & \textbf{37.82} & \textbf{165.67} \\
\bottomrule
\end{tabular}
\vspace{-5pt}
\end{table}


%% file: sections/conclusion.tex
\vspace{-5pt}
We proposed in this paper the Structure-Adaptive Tree Field Integrators (\STADTFI) algorithm, providing exact integration on trees and favorable time complexity, as compared to previously most computationally efficient methods. We show a rich range of applications of \STADTFI, from geodesic Sinkhorn methods for solving Optimal Transport problems on meshes and point clouds to relative positional encoding mechanism for Vision Transformers. Our empirical evaluation was complemented with detailed algorithmic analysis of \STADTFI, confirming our empirical claims.

%% file: sections/appendix.tex
\section{Algorithmic Details and Deferred Proofs}
\label{app:stad}

\subsection{Pseudocode for Backbone Decomposition}
\label{app:stad-backbonedecomposition}

Algo.~\ref{alg:decompose-components} retains the adjacency
lists of the original tree $T$ and filters neighbors using
constant-time membership tests for $V(S)$ and $V(P)$.
Let $b_S$ denote the number of original-tree edges with exactly
one endpoint in $V(S)$. Given the backbone $P$, the algorithm
takes $O(|V(S)|+b_S)$ time and $O(|V(S)|)$ auxiliary space.

\begin{algorithm}[H]
\caption{\textsc{DecomposeOffPathComponents}$(S,P)$}
\label{alg:decompose-components}
\begin{algorithmic}[1]
\Require Connected unweighted subtree $S$, backbone
path $P=(p_1,\ldots,p_m)$, and original-tree adjacency lists
\Ensure Components $B_1,\ldots,B_k$, anchor indices $a$,
projections $\pi$, depths $d$, maximum depth $H$, and
component depths $(H_\alpha)_{\alpha=1}^{k}$

\State Initialize per-vertex records on $V(S)$; mark backbone
vertices and store their indices
\Comment{$O(|V(S)|)$}

\State Set $\pi(p_j)\gets p_j$ and $d(p_j)\gets 0$
for all $p_j\in P$
\Comment{$O(m)$}

\State Scan the original adjacency lists of backbone vertices.
Record the $k$ edges $(p_{a(\alpha)},r_\alpha)$ with
$r_\alpha\in V(S)\setminus V(P)$ and their anchor indices
$a(\alpha)$
\Comment{$O(m+\sum_{v\in V(P)}\deg_T(v))$}

\For{$\alpha=1,\ldots,k$}
    \State Run BFS from $r_\alpha$ to obtain component $B_\alpha$ and distances $\ell_\alpha(u)$ from $r_\alpha$
    \Comment{$O(\sum_{u\in V(B_\alpha)}\deg_T(u))$}

    \State Set $\pi(u)\gets p_{a(\alpha)}$ and
    $d(u)\gets\ell_\alpha(u)+1$ for all $u\in V(B_\alpha)$
    \Comment{$O(|V(B_\alpha)|)$}

    \State Record
    $H_\alpha\gets\max_{u\in V(B_\alpha)}d(u)$
    during the same pass
    \Comment{$O(|V(B_\alpha)|)$}
\EndFor

\State
$H\gets\max\bigl(\{0\}\cup
\{H_\alpha:1\leq\alpha\leq k\}\bigr)$
\Comment{$O(k+1)$}

\State \Return
$(B_1,\ldots,B_k,a,\pi,d,H,(H_\alpha)_{\alpha=1}^{k})$

\end{algorithmic}
\end{algorithm}

Each component $B_\alpha$ has exactly one edge connecting it to the backbone, with endpoints $r_\alpha$ and $p_{a(\alpha)}$. Otherwise, the connected component and backbone would form a cycle. The BFS from $r_\alpha$ rejects neighbors outside $S$ and vertices on $P$, and computes $\ell_\alpha(u)=\dist_{B_\alpha}(r_\alpha,u)$. Hence, $d(u)=\ell_\alpha(u)+1$, where the additional unit accounts for the edge connecting $r_\alpha$ to the backbone. The recorded $H_\alpha$ is therefore the maximum depth of $B_\alpha$ relative to $P$.

Different components may share the same anchor, but each BFS explores a distinct component. The backbone scan and these BFS traversals together examine the original adjacency list of every vertex in $S$ once. Since $S$ is a connected subtree, the total number of adjacency entries examined is
\[
\sum_{v\in V(S)}\deg_T(v)
=
2(|V(S)|-1)+b_S.
\]
Internal edges contribute two entries, while edges leaving $S$ contribute one. All remaining operations take $O(|V(S)|)$ time, because $m+\sum_{\alpha=1}^{k}|V(B_\alpha)|=|V(S)|$. Thus the total running time is $O(|V(S)|+b_S)$. The component lists, per-vertex records, and BFS queues require $O(|V(S)|)$ auxiliary space.

\subsection{PseudoCode for computing the global term $G$ for evaluating Cross Interactions}
\label{app:stad-fft}

\begin{algorithm}[H]
\caption{\textsc{ComputeGTerm}$(A,f,H,m)$}
\label{alg:global-anchor}
\begin{algorithmic}[1]
\Require Aggregates $A_r[t]$, kernel $f$, maximum depth $H$,
         and number of backbone vertices $m$
\Ensure Backbone interaction table $G_q[s]$

\State Initialize $G_q[s]\gets0$ for all $q,s$
       \Comment{$O((H+1)m)$}
\State Choose $L=2^{\lceil\log_2(3m-2)\rceil}$

\For{$q=0,\ldots,H$}
    \For{$r=0,\ldots,H$}
        \State Form $a[j]\gets A_r[j+1]$,
               $j=0,\ldots,m-1$
               \Comment{$O(m)$}
        \State Form
               $b[j]\gets f(q+r+|j-(m-1)|)$,
               $j=0,\ldots,2m-2$
               \Comment{$O(m)$}
        \State Compute linear convolution
               $y=a*b$ using length-$L$ FFTs with zero padding
               \Comment{$O(m\log(2m))$}
        \State Update
               $G_q[s]\gets G_q[s]+y[s+m-2]$
               for all $s=1,\ldots,m$
               \Comment{$O(m)$}
    \EndFor
\EndFor

\State \Return $G$
\end{algorithmic}
\end{algorithm}


\begin{algorithm}[H]
\caption{\textsc{ComputeGTerm2DFFT}$(A,f,H,m)$}
\label{alg:global-anchor-2dfft}
\begin{algorithmic}[1]
\Require Aggregates $A_r[t]$, kernel $f$, maximum depth $H$,
         and number of backbone vertices $m$
\Ensure Backbone interaction table $G_q[s]$

\State Form the depth-reversed array
       $\overline A[\rho,j]\gets A_{H-\rho}[j+1]$
       for $0\leq\rho\leq H$, $0\leq j\leq m-1$
       \Comment{$O((H+1)m)$}

\State Form the shifted kernel
       $\overline K[i,j]\gets f(i+|j-(m-1)|)$
       for $0\leq i\leq2H$, $0\leq j\leq2m-2$
       \Comment{$O((H+1)m)$}

\State Zero-pad both arrays to power-of-two dimensions
       at least $(3H+1)\times(3m-2)$
       \Comment{$O((H+1)m)$}

\State Compute $Y\gets\overline K*\overline A$
       using two-dimensional FFT convolution
       \Comment{$O((H+1)m\log(2(H+1)m))$}

\State Extract $G_q[s]\gets Y[q+H,s+m-2]$
       for $0\leq q\leq H$, $1\leq s\leq m$
       \Comment{$O((H+1)m)$}

\State \Return $G$
\end{algorithmic}
\end{algorithm}

\subsection{PseudoCode for computing the Cross Terms of Eq. \ref{eq:stad-split}}
\label{app:stad-cross}

\begin{algorithm}[H]
\caption{\textsc{CrossSubRoutine}$(S,P,x,f)$}
\label{alg:cross-subroutine}
\begin{algorithmic}[1]
\Require Connected unweighted subtree $S$,
         nonempty backbone $P=(p_1,\ldots,p_m)$,
         field $x$ on $V(S)$, and kernel $f$
\Ensure $\operatorname{Cross}[u]=\operatorname{Cross}_P(u)$
        for all $u\in V(S)$, and components $B_1,\ldots,B_k$

\State $(B_1,\ldots,B_k,a,\pi,d,H,(H_\alpha)_{\alpha=1}^{k})
       \gets\textsc{DecomposeOffPathComponents}(S,P)$
       \Comment{$O(|V(S)|+b_S)$}

\State Initialize $A_r[t]\gets0$ for
       $0\leq r\leq H$, $1\leq t\leq m$
       \Comment{$O((H+1)m)$}
\State Set $A_0[t]\gets x_{p_t}$ for $t=1,\ldots,m$
       \Comment{$O(m)$}
\State Scan all components and accumulate
       $A_{d(u)}[a(\alpha)]
       \gets A_{d(u)}[a(\alpha)]+x_u$
       for each $u\in V(B_\alpha)$
       \Comment{$O(|V(S)|)$}

\State $G\gets\textsc{ComputeGTerm2DFFT}(A,f,H,m)$
       \Comment{$O((H+1)m\log(2(H+1)m))$}

\State Set $\operatorname{Cross}[p_t]\gets G_0[t]$
       for $t=1,\ldots,m$
       \Comment{$O(m)$}

\For{$\alpha=1,\ldots,k$}
    \State Initialize $b_\alpha[r]\gets0$
           for $r=0,\ldots,H_\alpha$
           \Comment{$O(H_\alpha+1)$}
    \State Accumulate
           $b_\alpha[d(u)]\gets b_\alpha[d(u)]+x_u$
           for each $u\in V(B_\alpha)$
           \Comment{$O(|V(B_\alpha)|)$}

    \State Compute
           $E_\alpha[q]=\sum_{r=1}^{H_\alpha}f(q+r)b_\alpha[r]$,
           $q=0,\ldots,H_\alpha$,
           using zero-padded 1D FFT convolution after reversing
           the component's depth profile
           \Comment{$O(H_\alpha\log(H_\alpha+1))$}

    \State Set
           $\operatorname{Cross}[u]
           \gets G_{d(u)}[a(\alpha)]-E_\alpha[d(u)]$
           for each $u\in V(B_\alpha)$
           \Comment{$O(|V(B_\alpha)|)$}
\EndFor

\State \Return $(\operatorname{Cross},B_1,\ldots,B_k)$
\end{algorithmic}
\end{algorithm}

Algo.~\ref{alg:cross-subroutine} implements Lemma~\ref{lem:stad-cross}. The aggregate array $A$ includes all source vertices, while each profile $b_\alpha$ contains only the sources in $B_\alpha$. For the correction step, reversing the profile converts the dependence on $q+r$ into a convolutional difference, allowing all required values $E_\alpha[q]$, for $0\leq q\leq H_\alpha$, to be computed together by a one-dimensional FFT. Both $b_\alpha$ and $E_\alpha$ have length $H_\alpha+1$, with $b_\alpha[0]=0$; the global arrays $A$ and $G$ still use $H$. Every nonempty off-backbone component satisfies $1\leq H_\alpha\leq |V(B_\alpha)|$, since $S$ is unweighted and the path from its anchor to a deepest vertex contains $H_\alpha$ vertices in the component. Thus,
\begin{equation}
\sum_{\alpha=1}^{k}H_\alpha\leq |V(S)|-m,
\qquad
\sum_{\alpha=1}^{k}(H_\alpha+1)\leq 2(|V(S)|-m).
\label{eq:stad-component-depth-volume}
\end{equation}
Initializing and accumulating all profiles and assigning the outputs therefore take $O(|V(S)| +b_S)$ time. Including decomposition, the computation of $G$, and all component corrections, the total cost is $O\!\left(|V(S)|+b_S +(H+1)m\log\bigl(2(H+1)m\bigr) +\sum_{\alpha=1}^{k}H_\alpha\log(H_\alpha+1)\right)$ (See Proof of Theorem \ref{thm:stad-decomposition-cost}), assuming constant-time kernel evaluation and membership tests. If $k=0$, the correction sum is empty and $H=0$.

\subsection{PseudoCode for \STADTFI(Backbone) computing complete Eq. \ref{eq:stad-split} involving Recursion}
\label{app:stad-recursion}

\begin{algorithm}[H]
\caption{\textsc{Solve}$(S,x,f)$}
\label{alg:stad-solve}
\begin{algorithmic}[1]
\Require Nonempty connected unweighted subtree $S$,
         field $x$ on $V(S)$, and kernel $f$
\Ensure $w_u^{(S)}
        =\sum_{v\in V(S)}f(\dist_S(u,v))x_v$
        for all $u\in V(S)$

\If{$|V(S)|=1$}
    \State Let $u$ be the unique vertex of $S$
    \State Set $w_u^{(S)}\gets f(0)x_u$
    \State \Return $w^{(S)}$
\EndIf

\State Choose a nonempty backbone path $P\subseteq S$
       using the prescribed selection rule

\State $(\operatorname{Cross},B_1,\ldots,B_k)
       \gets\textsc{CrossSubRoutine}(S,P,x,f)$

\State Set $w_u^{(S)}\gets\operatorname{Cross}[u]$
       for all $u\in V(S)$
       \Comment{$O(|V(S)|)$}

\For{$\alpha=1,\ldots,k$}
    \State $w^{(B_\alpha)}
           \gets\textsc{Solve}
           (B_\alpha,x|_{V(B_\alpha)},f)$
    \State Update
           $w_u^{(S)}\gets w_u^{(S)}+w_u^{(B_\alpha)}$
           for all $u\in V(B_\alpha)$
           \Comment{$O(|V(B_\alpha)|)$}
\EndFor

\State \Return $w^{(S)}$
\end{algorithmic}
\end{algorithm}

Algo.~\ref{alg:stad-solve} applies the decomposition in Eq.~\ref{eq:stad-split}. Backbone vertices receive their complete contribution from $\operatorname{Cross}$. For each off-backbone vertex, recursion supplies the remaining within-component contribution. Because $P$ is nonempty, every recursive component is strictly smaller than $S$, ensuring termination. At the root call $S=T$, the returned vector is the desired field $w$.








\subsection{PseudoCode for Backbone Selection in \STADTFI(Adaptive with FFT)}

\begin{algorithm}[H]
\caption{\textsc{ChooseBackbone}$(S)$}
\label{alg:choose_backbone}
\begin{algorithmic}[1]
\Require Connected unweighted subtree $S$ with
         $|V(S)|\geq2$
\Ensure Backbone $P^\star$, either a diameter path
        or a singleton centroid

\State Find a diameter path $P_{\mathrm{diam}}$
       using two BFS traversals and parent pointers
       \Comment{$O(|V(S)|+b_S)$}

\State Find a centroid $c_S$ using a rooted DFS
       and subtree-size computation
       \Comment{$O(|V(S)|+b_S)$}

\ForAll{$P\in\{P_{\mathrm{diam}},(c_S)\}$}
       \Comment{Two candidates}
    \State $(B_1,\ldots,B_k,a,\pi,d,H,(H_\alpha)_{\alpha=1}^{k})\gets\textsc{DecomposeOffPathComponents}(S,P)$
           \Comment{$O(|V(S)|+b_S)$}

    \State Set $m\gets|V(P)|$ and $M\gets(H+1)m$
           \Comment{$O(1)$}

    \State Obtain component sizes
           $n_\alpha\gets|V(B_\alpha)|$
           for $\alpha=1,\ldots,k$
           \Comment{$O(k+1)$}

    \State Evaluate $\mathcal C(S,P)$ using
           Eq.~\ref{eq:stad-score}
           \Comment{$O(k+1)$}
\EndFor

\If{$\mathcal C(S,(c_S))
      <\mathcal C(S,P_{\mathrm{diam}})$}
      \Comment{$O(1)$}
    \State \Return $(c_S)$
           \Comment{$O(1)$}
\Else
    \State \Return $P_{\mathrm{diam}}$
           \Comment{$O(1)$; diameter wins ties}
\EndIf

\end{algorithmic}
\end{algorithm}

Each traversal scans the adjacency lists of vertices in $V(S)$ and excludes neighbors outside $S$.  Only two candidates are evaluated, and their scores use component sizes, the global depth $H$, and each component depth $H_\alpha$. Since $k\leq|V(S)|-1$, the total selection cost is $O(|V(S)|+b_S)$ time and $O(|V(S)|)$ auxiliary space, as stated in Lemma~\ref{lem:stad-selection}.

\subsection{PseudoCode for \STADTFI(Specialized Kernels)}
\begin{algorithm}[H]
\caption{\textsc{CrossSubRoutineExp}$(S,P,x,f)$}
\label{alg:stad-specialized-cross}
\begin{algorithmic}[1]
\Require Connected subtree $S$, backbone $P=(p_1,\ldots,p_m)$,
         field $x$, and supplied kernel
         $f(d)=\sum_{\ell=1}^{J}c_\ell\lambda_\ell^d$
\Ensure Cross-component contributions $\operatorname{Cross}[u]$
        and components $B_1,\ldots,B_k$

\State Obtain $(B_1,\ldots,B_k,a,\pi,d,H,(H_\alpha)_{\alpha=1}^{k})$ using
       \textsc{DecomposeOffPathComponents}$(S,P)$
       \Comment{$O(|V(S)|+b_S)$}
\State Initialize $\operatorname{Cross}[u]\gets0$
       for all $u\in V(S)$
       \Comment{$O(|V(S)|)$}

\For{$\ell=1,\ldots,J$}
    \State Generate $\lambda_\ell^0,\ldots,\lambda_\ell^H$
           by successive multiplication
           \Comment{$O(H+1)$}

    \State Initialize the working row by vertex accumulation:
           $G_0[s]\gets
           \sum_{\pi(v)=p_s}\lambda_\ell^{d(v)}x_v$
           for all $s$
           \Comment{$O(|V(S)|)$}

    \State Accumulate
           $E_\alpha[0]\gets
           \sum_{v\in V(B_\alpha)}
           \lambda_\ell^{d(v)}x_v$
           for all components
           \Comment{$O(|V(S)|)$}

    \State Forward sweep:
           $G_0[s]\gets G_0[s]+\lambda_\ell G_0[s-1]$,
           in order $s=2,\ldots,m$
           \Comment{$O(m)$}

    \State Backward sweep:
           $G_0[s]\gets
           (1-\lambda_\ell^2)G_0[s]
           +\lambda_\ell G_0[s+1]$,
           in order $s=m-1,\ldots,1$
           \Comment{$O(m)$}

    \State For every $p_s\in V(P)$, add
           $c_\ell G_0[s]$ to $\operatorname{Cross}[p_s]$
           \Comment{$O(m)$}

    \State For every $u\in V(B_\alpha)$, add
           $c_\ell\lambda_\ell^{d(u)}
           \bigl(G_0[a(\alpha)]-E_\alpha[0]\bigr)$
           to $\operatorname{Cross}[u]$
           \Comment{$O(|V(S)|)$}
\EndFor

\State \Return $(\operatorname{Cross},B_1,\ldots,B_k)$
\end{algorithmic}
\end{algorithm}

For each exponential term, the algorithm aggregates sources at their anchors, applies two backbone sweeps, and subtracts within-component contributions, requiring $O(J|V(S)|+b_S)$ total time.

\subsection{Deferred Proofs}
\label{app:proofs}
\begin{proof}[Proof of Lemma~\ref{lem:stad-distance}]
Connectivity gives an attachment from each $B_\alpha$ to $P$. Two attachment edges would form a cycle using a path in $B_\alpha$ and a segment of $P$; this also excludes two edges to the same anchor. Thus the attachment is unique. For vertices in different components, the unique connecting path runs from $u$ to $\pi(u)$, along $P$ to $\pi(v)$, and then to $v$. The same argument applies when an endpoint lies on $P$, with the corresponding depth equal to zero. The backbone segment has $|s-t|$ edges, proving the identity.
\end{proof}

\begin{proof}[Proof of Proposition~\ref{prop:stad-1d-fft}]
View $G$ as a matrix whose rows index query depths $q=0,\ldots,H$ and whose columns index backbone positions $s=1,\ldots,m$, as illustrated in Figure~\ref{fig:stad-G}.

Eq.~\ref{eq:sumofFFTterms} gives a direct way to compute this matrix: for each query coordinate $(q,s)$, sum the contributions from all source coordinates $(r,t)$. After initializing $G$ to zero, this produces four nested loops:
\[
\begin{array}{l}
\textbf{for } q=0,\ldots,H:\\
\quad \textbf{for } s=1,\ldots,m:\\
\qquad \textbf{for } r=0,\ldots,H:\\
\qquad\quad \textbf{for } t=1,\ldots,m:\\
\qquad\qquad
G_q[s]\gets G_q[s]
+f(q+r+|s-t|)A_r[t].
\end{array}
\]
There are $(H+1)m$ query coordinates, and each requires summing over $(H+1)m$ source coordinates. Under the constant-time kernel assumption, the total arithmetic cost is therefore $\Theta\!\left((H+1)^2m^2\right)$.

The above brute-force computation evaluates separately a length-$m$ sum over $t$ for every query position $s$ and depth pair $(q,r)$. To accelerate it, we fix $(q,r)$ and compute these sums jointly for all $s$. For each fixed pair $(q,r)$, the inner sum over $t$ in Eq. \ref{eq:sumofFFTterms} has the form
\begin{equation}
C_{q,r}[s]
:=
\sum_{t=1}^{m}
f(q+r+|s-t|)A_r[t],
\qquad 1\leq s\leq m.
\label{eq:inbetweencross}
\end{equation}
Then $G_q[s] = \sum_{r=0}^{H}C_{q,r}[s]$. Consequently, we may process one pair $(q,r)$ at a time, compute the entire vector $C_{q,r}$ (line no. 7 in Algo. \ref{alg:global-anchor}, explained later by Eq.~\ref{eq:stad-fft-crop}), and add it to the output row $G_q[\cdot]$ (line no. 8 in Algo. \ref{alg:global-anchor}). This corresponds to interchanging the $s$ and $r$ loops and replacing the two position loops $s$ and $t$ by a structured matrix vector multiplication, namely one FFT convolution. The following steps establish its cost using the standard \textit{Toeplitz-to-convolution reduction}; see also \citet[Section~3.2]{luo2021stable}.

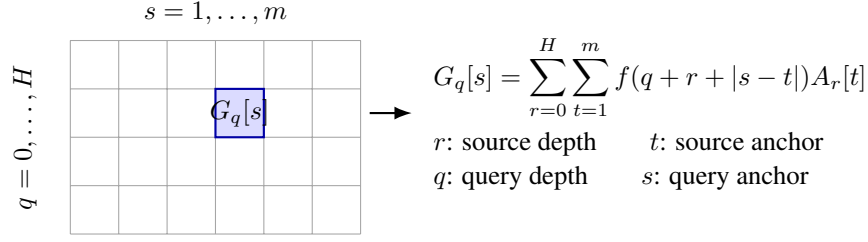
\begin{figure}[t]
\centering
\begin{tikzpicture}[scale=0.8]
    \fill[blue!14] (2.4,1.6) rectangle (3.2,2.4);
    \draw[step=0.8cm,gray!70,thin] (0,0) grid (4.8,3.2);
    \draw[blue!65!black,thick]
        (2.4,1.6) rectangle (3.2,2.4);
    \node at (2.8,2.0) {$G_q[s]$};

    \node at (2.4,3.65) {$s=1,\ldots,m$};
    \node[rotate=90] at (-0.75,1.6) {$q=0,\ldots,H$};

    \draw[-{Latex},thick] (4.95,2.0)--(5.65,2.0);
    \node[anchor=west,align=left] at (5.85,2.0) {
        $\displaystyle
        G_q[s]=\sum_{r=0}^{H}\sum_{t=1}^{m}
        f(q+r+|s-t|)A_r[t]$\\[4pt]
        $r$: source depth\qquad$t$: source anchor\\[2pt]
        $q$: query depth\qquad$s$: query anchor
    };
\end{tikzpicture}
\caption{The schematic grid presents a matrix where each entry $G_q[s]$ sums over source depths $r$ and anchors $t$. }
\label{fig:stad-G}
\end{figure}

For fixed $(q,r)$, define the matrix
\[
T^{(q,r)}_{s,t}
:=
f(q+r+|s-t|),
\qquad 1\leq s,t\leq m.
\]
Its entries are constant along each diagonal, since
\[
T^{(q,r)}_{s+1,t+1}
=
f(q+r+|(s+1)-(t+1)|)
=
T^{(q,r)}_{s,t}
\]
whenever both entries are defined.
Thus, $T^{(q,r)}$ is a Toeplitz matrix, and
\[
\begin{pmatrix}
C_{q,r}[1]\\
\vdots\\
C_{q,r}[m]
\end{pmatrix}
=
T^{(q,r)}
\begin{pmatrix}
A_r[1]\\
\vdots\\
A_r[m]
\end{pmatrix}.
\]

Equivalently, define the signed-offset kernel $K_{q,r}[z]:= f(q+r+|z|)$ for $-(m-1)\leq z\leq m-1$. Every offset $s-t$ lies in this range, so
\begin{equation}
C_{q,r}[s]
=
\sum_{t=1}^{m}K_{q,r}[s-t]A_r[t].
\label{eq:stad-offset-convolution}
\end{equation}
This is a convolution indexed by the relative displacement $s-t$. To specify the FFT indexing precisely, define two zero-indexed arrays $a$ and $b$ as:
\[
a[j]:=A_r[j+1],
\qquad 0\leq j\leq m-1,
\]
\[
b[\ell]
:=
K_{q,r}[\ell-(m-1)]
=
f\!\left(q+r+|\ell-(m-1)|\right),
\qquad 0\leq\ell\leq2m-2.
\]
Extend both arrays by zero outside these ranges. Their linear convolution is
\[
y[n]
:=
(a*b)[n]
=
\sum_{j\in\mathbb Z}a[j]b[n-j].
\]
The convolution has length $m+(2m-1)-1=3m-2$, with potentially nonzero indices $0,\ldots,3m-3$. For a desired query position $s$, evaluate this convolution at index $n=s+m-2$:
\begin{align}
y[s+m-2]
&=
\sum_{j=0}^{m-1}a[j]b[s+m-2-j]
\nonumber =
\sum_{j=0}^{m-1}
A_r[j+1]K_{q,r}[s-1-j]
\nonumber=
\sum_{t=1}^{m}
A_r[t]K_{q,r}[s-t]
\nonumber\\
&=
C_{q,r}[s].
\label{eq:stad-fft-crop}
\end{align}
The third equality substitutes $t=j+1$. Therefore, one linear convolution supplies every entry of $C_{q,r}$ through the crop
\[
C_{q,r}[s]=y[s+m-2],
\qquad s=1,\ldots,m.
\]

Next, we compute the linear convolution using FFTs. We choose the transform length
$L = 2^{\lceil\log_2(3m-2)\rceil}$. Then $L\geq3m-2$ and $L=O(m)$.
Zero-pad $a$ and $b$ to length $L$, obtaining $\widetilde a$ and $\widetilde b$. Let $\omega_L=\exp(-2\pi\mathrm{i}/L)$. Their discrete Fourier transforms are
\[
\widehat a[\ell]
=
\sum_{i=0}^{L-1}\widetilde a[i]\omega_L^{i\ell},
\qquad
\widehat b[\ell]
=
\sum_{j=0}^{L-1}\widetilde b[j]\omega_L^{j\ell}.
\]
Multiply these transforms entrywise and apply the inverse transform: $Y[n] = \frac{1}{L} \sum_{\ell=0}^{L-1} \widehat a[\ell]\widehat b[\ell]\omega_L^{-n\ell}$. Expanding the two transforms gives
\begin{align*}
Y[n]
&=
\sum_{i=0}^{L-1}\sum_{j=0}^{L-1}
\widetilde a[i]\widetilde b[j]
\left(
\frac{1}{L}
\sum_{\ell=0}^{L-1}
\omega_L^{\ell(i+j-n)}
\right).
\end{align*}
The expression in parentheses satisfies
\[
\frac{1}{L}
\sum_{\ell=0}^{L-1}
\omega_L^{\ell(i+j-n)}
=
\begin{cases}
1, & i+j\equiv n\pmod L,\\
0, & \text{otherwise}.
\end{cases}
\]
Indeed, it is either a sum of $L$ ones or a geometric series whose sum is zero. Hence, $Y$ is the circular convolution of the padded arrays. However, their nonzero entries satisfy
\[
0\leq i+j\leq(m-1)+(2m-2)=3m-3<L.
\]
Therefore, for $0\leq n<L$, the congruence $i+j\equiv n\pmod L$ reduces to the equality $i+j=n$. There is no wraparound, and
\[
Y[n]
=
\sum_{i+j=n}a[i]b[j]
=
y[n].
\]
Thus, $y = \operatorname{IFFT}_L \left( \operatorname{FFT}_L(\widetilde a) \odot \operatorname{FFT}_L(\widetilde b) \right)$ computes the required linear convolution exactly in exact arithmetic. Here, $\odot$ denotes entrywise multiplication.

Constructing the arrays $a$ and $b$ and padding them to length $L$ takes $O(m)$ time. The two forward FFTs and one inverse FFT require $O(L\log(2L))= O(m\log(2m))$ arithmetic operations. Entrywise multiplication and extraction of the $m$ desired entries take $O(m)$ additional operations. Therefore, the entire vector $C_{q,r}$ is computed in $O(m\log(2m))$ time, rather than the $\Theta(m^2)$ operations required
to evaluate its $m$ length-$m$ sums separately.

Finally, we accumulate over all depth pairs. Initialize every entry of $G$ to zero. For each $q=0,\ldots,H$ and $r=0,\ldots,H$, compute $C_{q,r}$ and update
\[
G_q[s]\gets G_q[s]+C_{q,r}[s],
\qquad s=1,\ldots,m.
\]
After processing source depths $0,\ldots,r$, the row satisfies $G_q[s] = \sum_{\rho=0}^{r}C_{q,\rho}[s]$. This follows immediately by induction on $r$.  There are $(H+1)^2$ depth pairs (line no 3 and 4 in Algo. \ref{alg:global-anchor}). Each requires $O(m\log(2m))$ operations for convolution and $O(m)$ operations for accumulation. Including initialization, the total cost is
\begin{align*}
& O((H+1)m)
  +(H+1)^2
  \bigl(O(m\log(2m))+O(m)\bigr)\\
&\qquad=
O\!\left((H+1)^2m\log(2m)\right).
\end{align*}
\end{proof}

\begin{proof}[Proof of Lemma~\ref{lem:stad-cross}]
Fix a recursive subtree $S$ and its backbone $P=(p_1,\ldots,p_m)$. Suppose $u\in V(B_\alpha)$, and write $q=d(u)$ and $s=a(\alpha)$, so that $\pi(u)=p_s$. Eq.~\ref{eq:stad-split} gives the within-component\footnote{The within-component term is evaluated recursively. This is valid because $B_\alpha$ is a connected subtree: for $u,v\in V(B_\alpha)$, their unique connecting path lies entirely in $B_\alpha$, and hence $\dist_{B_\alpha}(u,v)=\dist_S(u,v)$.} and cross-component split i.e., $w_u^{(S)} = w_u^{(B_\alpha)} + \operatorname{Cross}_P(u)$, where
\begin{equation}
\operatorname{Cross}_P(u) = \sum_{v\in V(S)\setminus V(B_\alpha)}
f\!\bigl(\dist_S(u,v)\bigr)x_v.
\label{eq:app-cross-original}
\end{equation}

\noindent
Now, we would like to group the cross-component sources by the anchor position $t$ on the backbone and depth $r$. For $0\leq r\leq H$ and $1\leq t\leq m$, we define $\mathcal{V}_{r,t}:=\{v\in V(S):\pi(v)=p_t,\ d(v)=r\}$. Every source vertex has exactly one projection and one depth. Consequently, these sets are pairwise disjoint and partition $V(S)$ i.e.,  $V(S)= \bigsqcup_{r=0}^{H}\bigsqcup_{t=1}^{m} \mathcal{V}_{r,t}$. Some sets may be empty; their sums are understood to be zero.

For a source $v\in\mathcal V_{r,t}\setminus V(B_\alpha)$, Lemma~\ref{lem:stad-distance} applies because $v$ lies outside the query's component. Therefore, $\dist_S(u,v) = d(u)+|s-t|+d(v) = q+|s-t|+r$. In particular, every source in this group receives the same kernel weight $f(q+|s-t|+r)$. Partitioning the source sum in \eqref{eq:app-cross-original} into these groups gives
\begin{align}
\operatorname{Cross}_P(u)
&=
\sum_{r=0}^{H}\sum_{t=1}^{m}
\sum_{v\in\mathcal V_{r,t}\setminus V(B_\alpha)}
f\!\bigl(\dist_S(u,v)\bigr)x_v
\nonumber\\
&=
\sum_{r=0}^{H}\sum_{t=1}^{m}
\sum_{v\in\mathcal V_{r,t}\setminus V(B_\alpha)}
f(q+|s-t|+r)x_v
\nonumber\\
&=
\sum_{r=0}^{H}\sum_{t=1}^{m}
f(q+|s-t|+r)
\left(
\sum_{v\in\mathcal V_{r,t}\setminus V(B_\alpha)}x_v
\right).
\label{eq:app-cross-grouped}
\end{align}
The last equality factors out a coefficient that is constant within each group. Thus, aggregation replaces a collection of source values by their sum without changing its weighted contribution. No linearity or positivity assumption on $f$ is needed.

The inner sum in Eq.~\ref{eq:app-cross-grouped} \textbf{excludes $B_\alpha$}, so it depends on which component contains the query. To obtain aggregates shared by \textbf{all queries}, we use the same  $A_r[t]:=\sum_{v\in\mathcal V_{r,t}}x_v$. These aggregates include every source in the corresponding group. In particular, $\mathcal V_{0,t}=\{p_t\}$, so
$A_0[t]=x_{p_t}$.

The disjoint union $\mathcal V_{r,t} = \bigl(\mathcal V_{r,t}\setminus V(B_\alpha)\bigr) \sqcup \bigl(\mathcal V_{r,t}\cap V(B_\alpha)\bigr)$  implies
\[
\sum_{v\in\mathcal V_{r,t}\setminus V(B_\alpha)}x_v =A_r[t]- \sum_{v\in\mathcal V_{r,t}\cap V(B_\alpha)}x_v.
\]
Substituting this identity into Eq.~\ref{eq:app-cross-grouped} and distributing the kernel weight yields
\begin{align}
\operatorname{Cross}_P(u)
&=
\underbrace{
\sum_{r=0}^{H}\sum_{t=1}^{m}
f(q+|s-t|+r)A_r[t]
}_{\text{First term is exactly $G_q[s]$ in Eq.~\ref{eq:sumofFFTterms}}}
\quad- \quad
\underbrace{\sum_{r=0}^{H}\sum_{t=1}^{m}
f(q+|s-t|+r)
\sum_{v\in\mathcal V_{r,t}\cap V(B_\alpha)}x_v}_{\text{Second term must be equivalent to component correction}}.
\label{eq:app-cross-before-correction}
\end{align}

To make the meaning of $G_q[s]$  explicit, expand the aggregates:
\begin{align}
G_q[s]
&=
\sum_{r=0}^{H}\sum_{t=1}^{m}
\sum_{v\in\mathcal V_{r,t}}
f(q+|s-t|+r)x_v
=
\sum_{v\in V(S)}
f\!\bigl(
q+\dist_S(p_s,\pi(v))+d(v)
\bigr)x_v.
\label{eq:app-G-expanded}
\end{align}
Thus, $G_q[s]$ includes every source, using the length of the route from the query to its backbone projection, along the backbone to the source projection, and then to the source.

For sources outside $B_\alpha$, this route is the unique query-to-source path, so its length is the true distance. For sources inside $B_\alpha$, it is a surrogate route through their shared anchor. Their surrogate contributions are precisely the second term of Eq.~\ref{eq:app-cross-before-correction}. Every vertex of $B_\alpha$ has projection $p_{a(\alpha)}$. Therefore,
\[
\mathcal V_{r,t}\cap V(B_\alpha)
=
\begin{cases}
\{v\in V(B_\alpha):d(v)=r\},
& t=a(\alpha),\\[3pt]
\varnothing,
& t\neq a(\alpha).
\end{cases}
\]
Moreover, every vertex of $B_\alpha$ has depth between $1$ and $H_\alpha$. Thus, the intersections above are empty for $r=0$ and for $r>H_\alpha$. Using the depth profile
\[
b_\alpha[r]
=
\sum_{\substack{v\in V(B_\alpha)\\d(v)=r}}x_v,
\qquad 0\leq r\leq H_\alpha,\qquad b_\alpha[0]=0,
\]
the term to be subtracted becomes
\begin{align}
&\sum_{r=0}^{H}\sum_{t=1}^{m}
f(q+|s-t|+r)
\sum_{v\in\mathcal V_{r,t}\cap V(B_\alpha)}x_v
\nonumber\\
&\qquad=
\sum_{r=1}^{H_\alpha}
f(q+|s-a(\alpha)|+r)b_\alpha[r]
\nonumber\\
&\qquad=
\sum_{r=1}^{H_\alpha}f(q+r)b_\alpha[r]
=
E_\alpha[q].
\label{eq:app-component-correction}
\end{align}
The second equality uses $s=a(\alpha)$. Since $q=d(u)\leq H_\alpha$, the locally computed correction contains every required query value.  Combining Eq.~\ref{eq:app-cross-before-correction}
and Eq.~\ref{eq:app-component-correction} gives
\[
\operatorname{Cross}_P(u) = G_q[s]-E_\alpha[q] = G_{d(u)}[a(\alpha)]-E_\alpha[d(u)].
\]

This subtraction removes only sources belonging to $B_\alpha$. Sources in another component sharing the same anchor remain in the cross-component sum.

Now suppose $u=p_s$, so $d(u)=0$. By definition, every source belongs to $\operatorname{Cross}_P(p_s)$. Lemma~\ref{lem:stad-distance} applies to every such source. Grouping them as above therefore gives
\begin{align*}
\operatorname{Cross}_P(p_s)
&=
\sum_{v\in V(S)}
f\!\bigl(\dist_S(p_s,v)\bigr)x_v\\
&=
\sum_{r=0}^{H}\sum_{t=1}^{m}
\sum_{v\in\mathcal V_{r,t}}
f(|s-t|+r)x_v\\
&=
\sum_{r=0}^{H}\sum_{t=1}^{m}
f(|s-t|+r)A_r[t]\\
&=
G_0[s].
\end{align*}
No component is excluded, so no correction is required.

Together, the two cases establish Eq.~\ref{eq:stad-cross-correction}. The actual within-component contribution remains $w_u^{(B_\alpha)}$ in Eq.~\ref{eq:stad-split} and is computed recursively.
\end{proof}

\begin{proof}[Proof of Lemma~\ref{lem:stad-2d-fft}]
We first verify the convolution identity. For $r=H-\rho$, we now have:
\begin{equation}
A_r[t] := A_{H-\rho}[t] = \widetilde A[\rho,t],
\qquad
0\leq\rho\leq H,\quad 1\leq t\leq m.
\label{eq:stad-depth-reversal}
\end{equation}
This gives $K[q-\rho,s-t]\widetilde A[\rho,t] = f(q+r+|s-t|)A_r[t]$. Since $\rho\mapsto H-\rho$ is a bijection on $\{0,\ldots,H\}$, summing over $\rho$ and $t$ proves Eq.\ref{eq:stad-2d-convolution}.

To implement this convolution with nonnegative indices,
define
\[
\overline A[\rho,j]
=
\widetilde A[\rho,j+1],
\qquad
\overline K[i,j]
=
K[i-H,j-(m-1)].
\]
The arrays $\overline A$ and $\overline K$ have dimensions
$(H+1)\times m$ and $(2H+1)\times(2m-1)$, respectively.
For their linear convolution $Y=\overline K*\overline A$,
the index shifts give
\begin{align*}
Y[q+H,s+m-2]
&=
\sum_{\rho=0}^{H}\sum_{t=1}^{m}
K[q-\rho,s-t]\widetilde A[\rho,t]
=
G_q[s].
\end{align*}

The full convolution has dimensions
$(3H+1)\times(3m-2)$.
Choose
\[
L_1=2^{\lceil\log_2(3H+1)\rceil},
\qquad
L_2=2^{\lceil\log_2(3m-2)\rceil},
\]
and zero-pad both arrays to $L_1\times L_2$. These dimensions prevent circular wraparound in either coordinate. By the two-dimensional convolution theorem, two forward FFTs, entrywise multiplication, and one inverse FFT recover $Y$ in exact arithmetic. The required entries of $G$ are then extracted using the indices above.

A two-dimensional FFT applies one-dimensional transforms along both coordinates~\citep{frigo2005design}, requiring $O\!\left( L_1L_2[\log(2L_1)+\log(2L_2)] \right)$ arithmetic operations. Since $L_1=O(H+1)$ and $L_2=O(m)$, this is
\[
O\!\left(
(H+1)m\log\bigl(2(H+1)m\bigr)
\right).
\]
Constructing the arrays, multiplying the transforms entrywise, and extracting the output each require $O((H+1)m)$ additional operations.
\end{proof}

\begin{proof}[Proof of Theorem~\ref{thm:stad-decomposition-cost}]
For each recursive subproblem $S$, we have
\[
m_S=|V(P_S)|,
\qquad
H_S=H_{P_S}(S),
\qquad
k_S=\text{number of components of }
S\setminus V(P_S).
\]
Write $H_{S,\alpha}$ for the component depth $H_\alpha$ at this call, measured relative to $P_S$. For singleton base cases, set $m_S=1$, $H_S=0$, and $k_S=0$. We first establish the nonrecursive cost of a call and then sum the FFT costs over the recursion.

Fix a subproblem $S$. Selecting the backbone, decomposing $S$, and computing projections and depths take $O(|V(S)|+b_S)$ time (Algo. \ref{alg:decompose-components} in Appendix \ref{app:stad-backbonedecomposition}). Accumulating vertex values into the aggregate table also takes $O(|V(S)|)$ time, while initializing this dense table requires $O((H_S+1)m_S)$ time. By Lemma~\ref{lem:stad-2d-fft}, computing the backbone interaction table $G$ costs $O\!\left((H_S+1)m_S\log\bigl(2(H_S+1)m_S\bigr)\right)$. This bound absorbs initialization of the aggregate table.

For each off-backbone component $B_\alpha$, initializing its depth profile costs $O(H_\alpha+1)$, and accumulating its vertex values costs $O(|V(B_\alpha)|)$. The correction terms are
\[
E_\alpha[q] = \sum_{r=1}^{H_\alpha}f(q+r)b_\alpha[r],
\qquad 0\leq q\leq H_\alpha.
\]
Using zero-based indexing, define
\[
\beta_\alpha[j]=b_\alpha[H_\alpha-j],
\quad 0\leq j\leq H_\alpha,
\qquad
\kappa_\alpha[t]=f(t),
\quad 0\leq t\leq 2H_\alpha,
\]
and extend both arrays by zero outside these ranges. For $0\leq q\leq H_\alpha$, the index $q+H_\alpha-j$ lies in $[0,2H_\alpha]$ for every $0\leq j\leq H_\alpha$. Hence,
\begin{align*}
(\beta_\alpha*\kappa_\alpha)[H_\alpha+q]=\sum_{j=0}^{H_\alpha} b_\alpha[H_\alpha-j]f(q+H_\alpha-j)=\sum_{r=0}^{H_\alpha}b_\alpha[r]f(q+r)
=E_\alpha[q],
\end{align*}
where the last equality uses $b_\alpha[0]=0$. The input lengths are $H_\alpha+1$ and $2H_\alpha+1$, so their full linear convolution has length $3H_\alpha+1$. Zero-padding to a power of two at least this length and extracting entries $H_\alpha,\ldots,2H_\alpha$ computes all required correction values in $O(H_\alpha\log(H_\alpha+1))$ operations. Each component has $H_\alpha\geq1$; if $H=0$, there are no components and no corrections.

By Eq.~\ref{eq:stad-component-depth-volume}, all profile initialization and vertex accumulation together take $O(|V(S)|)$ time. Consequently, constructing the profiles and computing all corrections costs $O(|V(S)|+ \sum_{\alpha=1}^{k}H_\alpha\log(H_\alpha+1))$. Assigning the cross-component contributions using Lemma~\ref{lem:stad-cross} and adding the recursive within-component outputs take $O(|V(S)|)$ additional time.

Consequently, the nonrecursive cost at $S$ is
\begin{equation}
O\!\left(
|V(S)| +b_S
+(H_S+1)m_S
 \log\bigl(2(H_S+1)m_S\bigr)
+\sum_{\alpha=1}^{k_S}H_{S,\alpha}\log(H_{S,\alpha}+1)
\right).
\label{eq:stad-local-2d}
\end{equation}

We next show that the total boundary-scanning cost is absorbed by the repeated vertex-processing cost. Let $n=|V(T)|$, and let $j_S$ denote the depth of call $S$ in the "recursion tree", with the root at depth zero. For bookkeeping, regard a singleton base case as removing its sole vertex as its backbone. Thus every call removes a nonempty connected backbone, and $m_S=|V(P_S)|\geq1$.

Consider an edge $\{u,v\}$ with $u\in V(S)$ and $v\notin V(S)$. Along the recursion path from $T$ to $S$, consider the first backbone removal after which $v$ no longer belongs to the child containing $u$. The vertex $v$ must belong to that removed backbone: otherwise, the edge $\{u,v\}$ would keep both vertices in the same remaining connected component. Hence every edge leaving $S$ has its other endpoint on an ancestor backbone. Moreover, each ancestor backbone has at most one edge to $S$. Indeed, both the ancestor backbone and $S$ are connected and vertex-disjoint. Two distinct edges between them, together with the paths joining their endpoints within these two connected sets, would form a cycle in $T$. Since there are exactly $j_S$ strict ancestors, $b_S\leq j_S$.

To sum these bounds, observe that every vertex is removed in exactly one backbone. Therefore, the backbone vertex sets partition $V(T)$, giving $\sum_{S\in\mathcal R}m_S=n$.

For each $v\in V(T)$, let $R(v)$ denote its removal call. The vertex $v$ belongs to exactly the calls on the root-to-$R(v)$ path, including both endpoints. Consequently, it belongs to $j_{R(v)}+1$ calls. Double-counting vertex-call pairs yields
\begin{align}
 \sum_{S\in\mathcal R}|V(S)|
 &=
 \sum_{S\in\mathcal R}
 \sum_{v\in V(T)}
 \mathbf{1}_{\{v\in V(S)\}}
 \notag\\
 &=
 \sum_{v\in V(T)}\bigl(j_{R(v)}+1\bigr)
 \notag\\
 &=
 \sum_{A\in\mathcal R}m_A(j_A+1)
 \notag\\
 &=
 n+\sum_{A\in\mathcal R}m_Aj_A.
 \label{eq:stad-vertex-call-count}
\end{align}
The third equality groups vertices by their removal call; exactly $m_A$ vertices are removed at call $A$. Using $m_S\geq1$, we obtain
\begin{equation}
 \sum_{S\in\mathcal R}b_S
 \leq
 \sum_{S\in\mathcal R}j_S
 \leq
 \sum_{S\in\mathcal R}m_Sj_S
 =
 \sum_{S\in\mathcal R}|V(S)|-n.
 \label{eq:stad-boundary-volume}
\end{equation}
In particular,
\begin{equation}
 \sum_{S\in\mathcal R}\bigl(|V(S)|+b_S\bigr)
 \leq
 2\sum_{S\in\mathcal R}|V(S)|-n
 \leq
 2\sum_{S\in\mathcal R}|V(S)|.
 \label{eq:stad-boundary-absorption}
\end{equation}

Finally, summing the nonrecursive costs in
Eq.~\ref{eq:stad-local-2d} over all calls and applying
Eq.~\ref{eq:stad-boundary-absorption} gives the total
running time
\begin{equation}
\begin{aligned}
\label{eq:summingrunning}
 T_{\mathrm{total}}
 =O\!\Biggl(\sum_{S\in\mathcal R}|V(S)|
 +\sum_{S\in\mathcal R}M_S\log(2M_S)+\sum_{S\in\mathcal R}\sum_{\alpha=1}^{k_S}
 H_{S,\alpha}\log(H_{S,\alpha}+1)
 \Biggr).
\end{aligned}
\end{equation}
This proves the claimed running-time bound, including
scans of the original adjacency lists.
\end{proof}

\begin{corollary}[A bound in terms of recursion depth and
off-backbone depth]
\label{cor:stad-depth-runtime}
Under the assumptions of Theorem~\ref{thm:stad-decomposition-cost}, let $n=|V(T)|$, let $L$ be the maximum number of calls on a root-to-leaf path in the recursion tree, and define $\overline H=\max_{S\in\mathcal R}H_S$, with $H_S=0$ for singleton base cases. Then the total running time is $ O\!\left( nL+(\overline H+1)n\log(2n)\right)$.
\end{corollary}

\begin{proof}
We bound the three terms in Theorem~\ref{thm:stad-decomposition-cost}. We use the following three properties of $\mathcal R$.
\begin{enumerate}[noitemsep, topsep = 0pt]
    \item First, the backbones removed at different recursive calls are disjoint and together cover $V(T)$. Hence, $\sum_{S\in\mathcal R}m_S=n.$
    \item Second, each component produces exactly one child in the recursion tree. Therefore, $\sum_{S\in\mathcal R}k_S = |\mathcal R|-1 \leq n-1,$  where $|\mathcal R|\leq n$ follows because every call removes a nonempty backbone.
    \item Third, subproblems at the same recursion level have disjoint vertex sets. Since there are at most $L$ levels,
    \begin{equation}
    \sum_{S\in\mathcal R}|V(S)|\leq nL.
    \label{eq:stad-recursion-volume}
    \end{equation}
\end{enumerate}

For the two-dimensional FFT terms, Property 1 gives
\begin{align*}
\sum_{S\in\mathcal R}
(H_S+1)m_S
\log\bigl(2(H_S+1)m_S\bigr)
&\quad\leq
(\overline H+1)
\log\bigl(2(\overline H+1)n\bigr)
\sum_{S\in\mathcal R}m_S\\
&\quad=
(\overline H+1)n
\log\bigl(2(\overline H+1)n\bigr)\\
&\quad=
O\!\left((\overline H+1)n\log(2n)\right), \quad \text{using $\overline H+1\leq n$}.
\end{align*}

For the correction computations, use $H_{S,\alpha}\leq H_S$ and Property 2 to obtain
\begin{align*}
\sum_{S\in\mathcal R}\sum_{\alpha=1}^{k_S}
H_{S,\alpha}\log(H_{S,\alpha}+1)
&\leq\sum_{S\in\mathcal R}k_SH_S\log(H_S+1)\\
&\leq
(n-1)\overline H\log(\overline H+1)\\
&=
O\!\left((\overline H+1)n\log(2n)\right).
\end{align*}
The intermediate expression with $k_SH_S$ is only an upper bound; the algorithm allocates and computes each correction at its own depth $H_{S,\alpha}$. Combining these inequalities with Eq.~\ref{eq:stad-recursion-volume} proves the total time complexity for computations using 2D FFTs.
\end{proof}

\begin{proof}[Proof of Proposition \ref{prop:stad-spider-comparison}]
For the center backbone, $H=h$ and $sh=n-1$. Lemma~\ref{lem:stad-spider-center} gives $O\!\left(n+sh\log(2h)\right) = O(n\log(2h))$.

For the diameter backbone, $m=2h+1$, $H=h$, $k=s-2$. The interaction table therefore has $M=(h+1)(2h+1)=\Theta(h^2)$ entries. Lemma~\ref{lem:stad-spider-diameter} gives
\begin{align*}
\operatorname{Time}(T)
&=
O\!\left(
n+h^2\log(2h)+(s-2)h\log(2h)
\right)\\
&=
O\!\left((n+h^2)\log(2h)\right).
\end{align*}

If $s$ is fixed, then $h=\Theta(n)$. Hence the diameter construction has an $O(n^2\log(2n))$ upper bound. Moreover, Algo.~\ref{alg:global-anchor-2dfft} explicitly constructs a table with $\Theta(h^2)=\Theta(n^2)$ entries, requiring $\Omega(n^2)$ time and storage. This establishes an asymptotic separation from the $O(n\log(2n))$ center construction for thecspecified dense implementation.
\end{proof}

\begin{proof}[Proof of Lemma~\ref{lem:stad-selection}]
Algo. \ref{alg:choose_backbone} shows the total computation time. 
A constant number of breadth-first traversals finds diameter endpoints and recovers the diameter path. A centroid is obtained by rooting $S$, computing rooted subtree sizes, and finding a vertex whose deletion leaves no component larger than $|V(S)|/2$. Both computations satisfy the same traversal bound. Besides this, decomposing $S$, and computing projections and depths take $O(|V(S)|+b_S)$ time (Algo. \ref{alg:decompose-components} in Appendix \ref{app:stad-backbonedecomposition}). It also obtains the component sizes and number. These statistics determine every term in Eq.~\ref{eq:stad-score}. Since there are two candidates, evaluating their scores and selecting between them takes $O(|V(S)|+b_S)$ time in total.

To prove, $M_S=O(|V(S)|)$, Write $s=|V(S)|\geq2$. 
Since the selection rule chooses a candidate with minimum score, the selected backbone $P_S$ satisfies $\mathcal C(S,P_S) = \min\!\left\{ \mathcal C(S,P_{\mathrm{diam}}), \mathcal C(S,(c_S)) \right\} \leq \mathcal C(S,(c_S))$. Thus, an upper bound on the centroid candidate's score also bounds the selected score, regardless of which candidate is chosen. The centroid candidate satisfies $M_c=H_c+1\leq s$, $\sum_\alpha H_{c,\alpha}\leq s-1$, and $\sum_\alpha n_{c,\alpha}=s-1$. Since $H_{c,\alpha}+1\leq s$ and $n_{c,\alpha}+1\leq s$, its score obeys  $\mathcal C(S,(c_S))\leq s+(3s-2)\log s$. The selected score is no larger. 
In particular,
\[
 M_S\log\!\bigl(\max\{M_S,2\}\bigr)
 \leq s+(3s-2)\log s.
\]
If $M_S\geq s$, division by $\log s$ gives
$M_S\leq3s+s/\log s=O(s)$; if $M_S<s$, this conclusion is
immediate.

We now bound the FFT and correction terms in
Eq.~\ref{eq:stad-local-2d}. For a nontrivial call,
$s\geq2$, and the previous estimate gives
\[
 M_S
 \leq 3s+\frac{s}{\log s}
 \leq \left(3+\frac{1}{\log2}\right)s
 =O(s).
\]
Consequently, $\log(2M_S)=O(\log(2s))$, and hence
\begin{align*}
 (H_S+1)m_S\log\bigl(2(H_S+1)m_S\bigr)=M_S\log(2M_S)=O\!\left(s\log(2s)\right).
\end{align*}

For each off-backbone component $B_{S,\alpha}$, the path from the backbone to a vertex at depth $H_{S,\alpha}$ contains $H_{S,\alpha}$ vertices inside that component. Hence, $H_{S,\alpha}\leq |V(B_{S,\alpha})|.$ Since the backbone is nonempty, we also have
\[
 |V(B_{S,\alpha})|+1\leq |V(S)|,
 \qquad
 \sum_{\alpha=1}^{k_S}|V(B_{S,\alpha})|
 =|V(S)|-m_S.
\]
Consequently,
\begin{align*}
 \sum_{\alpha=1}^{k_S}
 H_{S,\alpha}\log(H_{S,\alpha}+1)
 &\leq
 \sum_{\alpha=1}^{k_S}
 |V(B_{S,\alpha})|\log|V(S)|\\
 &=(|V(S)|-m_S)\log|V(S)|\\
 &\leq |V(S)|\log|V(S)|\\
 &=O\!\left(|V(S)|\log(2|V(S)|)\right).
\end{align*}

Substituting these estimates into
Eq.~\ref{eq:stad-local-2d} gives
\begin{align*}
 T_{\mathrm{local}}(S)
 &=O\!\Bigl(
 |V(S)|+b_S
 +|V(S)|\log(2|V(S)|)
 +|V(S)|\log(2|V(S)|)
 \Bigr)\\
 &=O\!\left(
 |V(S)|\log(2|V(S)|)+b_S
 \right).
\end{align*}
\end{proof}

The next scalar estimate controls the component-size penalties in the score comparison. It is used only in the proof.

\begin{lemma}[Merging component penalties]
\label{lem:stad-merge-penalty}
Let $h\geq1$, and let positive integers $t_1,\ldots,t_k$ satisfy $\sum_{j=1}^{k}t_j=N$. Define $\varphi(t)=t\log(t+1)$, where $\varphi(0)=0,$ and $\Delta_h = \sum_{j=1}^{k} \left( \min\{t_j,h\}\log(h+1)+\varphi(t_j)\right)-\varphi(N).$ Then $\Delta_h\leq(h+1)^2$.  

If, additionally, $\max_j t_j\leq\frac{2N}{3},$ and $N\geq16(h+1)^2$, then $\Delta_h\leq-\frac{N}{8}.$
\end{lemma}

\begin{proof}
For this proof, write $L=\log(h+1)$ and define, for $x>0$,
\[
 g(x)
 =
 \min\{1,h/x\}L+\log(x+1)-\log(N+1).
\]
Since $\sum_j t_j=N$, we have $\varphi(N) =N\log(N+1)=\sum_{j=1}^{k}t_j\log(N+1).$ Also, $\min\{t_j,h\}=t_j\min\{1,h/t_j\}$. Consequently,
\begin{align*}
 \Delta_h
 &=
 \sum_{j=1}^{k}
 \left(
 \min\{t_j,h\}L
 +t_j\log(t_j+1)
 -t_j\log(N+1)
 \right)\\
 &=
 \sum_{j=1}^{k}t_j
 \left(
 \min\{1,h/t_j\}L
 +\log(t_j+1)-\log(N+1)
 \right)\\
 &=
 \sum_{j=1}^{k}t_jg(t_j).
\end{align*}
Thus, if $g(t_j)\leq B$ for every part, then $\Delta_h \leq B\sum_{j=1}^{k}t_j =NB$. We therefore seek uniform upper bounds on $g$ over the possible partition sizes.

For $0<x\leq h$, we have $g(x)=L+\log(x+1)-\log(N+1),$ increasing because its derivative is $1/(x+1)>0$. For $x\geq h$, we have $g(x)=\frac{hL}{x}+\log(x+1)-\log(N+1)$, we have
\[
 g'(x)
 =
 -\frac{hL}{x^2}+\frac{1}{x+1}
 =
 \frac{x^2-hLx-hL}{x^2(x+1)}.
\]
The denominator is positive. The numerator is a quadratic with one negative root and one positive root, since the product of its roots is $-hL<0$. Its sign on the positive axis therefore changes only from negative to positive. Accordingly, on $[h,\infty)$ the function $g$ can decrease and then increase, but has no interior maximum.

Combining these observations, for every $U\geq h$, $\max_{1\leq x\leq U}g(x)= \max\{g(h),g(U)\}$. Indeed, the maximum on $[1,h]$ occurs at $h$, while the maximum on $[h,U]$ occurs at one of its endpoints.

We first prove the bound $\Delta_h\leq(h+1)^2$.
Suppose that $N<h$. Every part satisfies $t_j\leq N<h$, so $\min\{t_j,h\}=t_j$. Moreover,
\[
 \sum_{j=1}^{k}\varphi(t_j)
 =
 \sum_{j=1}^{k}t_j\log(t_j+1)
 \leq
 \sum_{j=1}^{k}t_j\log(N+1)
 =
 \varphi(N).
\]
Hence, it follows that
\[
 \Delta_h
 \leq\sum_{j=1}^{k}t_jL
 =NL
 \leq hL
 \leq h^2
 \leq(h+1)^2,
\]
where we used $\log(h+1)\leq h$.

Now suppose that $N\geq h$. Every part lies in $[1,N]$,
so the endpoint bound gives $g(t_j)\leq\max\{g(h),g(N)\}$. The two endpoint values are
\[
 g(h)
 =
 \log\frac{(h+1)^2}{N+1},
 \qquad
 g(N)
 =
 \frac{h\log(h+1)}{N}.
\]
Therefore,
\begin{align*}
 \Delta_h
 &\leq
 \sum_{j=1}^{k}t_j\max\{g(h),g(N)\}\\
 &=
 N\max\{g(h),g(N)\}\\
 &=
 \max\left\{
 N\log\frac{(h+1)^2}{N+1},
 \ h\log(h+1)
 \right\}.
\end{align*}
The second expression satisfies
\[
 h\log(h+1)\leq h^2\leq(h+1)^2.
\]
To bound the first, recall that, for every $A_0>0$, the function $x\log(A_0/x)$ on $x>0$ has derivative $\log(A_0/x)-1$. It attains its maximum at $x=A_0/e$, where its value is $A_0/e$. Hence $x\log(A_0/x)\leq A_0/e$. Taking $A_0=(h+1)^2$ yields
\[
 N\log\frac{(h+1)^2}{N+1}
 \leq
 N\log\frac{(h+1)^2}{N}
 \leq
 \frac{(h+1)^2}{e}
 \leq(h+1)^2.
\]
Both endpoint contributions are therefore at most $(h+1)^2$,
which proves the first assertion.

For the second assertion, assume $\max_j t_j\leq{2N}/{3}$, and  $N\geq16(h+1)^2$. These assumptions imply $N\geq3$ and $h<2N/3$. All partition sizes now lie in $[1,2N/3]$, so the same endpoint argument gives $g(t_j)\leq\max\{g(h),g(2N/3)\}$. At the first endpoint,
\[
 g(h)
 =
 \log\frac{(h+1)^2}{N+1}
 \leq\log\frac1{16}
 <-\frac18.
\]
At the other endpoint,
\[
 g(2N/3)
 =
 \frac{3hL}{2N}
 +\log\frac{2N/3+1}{N+1}.
\]
We bound these two terms separately. Since
$hL\leq h^2\leq(h+1)^2$ and $N\geq16(h+1)^2$,
\[
 \frac{3hL}{2N}\leq\frac{3}{32}.
\]
Also, $N\geq3$ implies
\[
 \frac{2N/3+1}{N+1}\leq\frac34,
\]
because
\[
 4(2N/3+1)\leq3(N+1)
 \quad\Longleftrightarrow\quad N\geq3.
\]
Using $\log(1-z)\leq-z$ with $z=1/4$, we obtain
\[
 g(2N/3)
 \leq\frac{3}{32}+\log(3/4)
 \leq\frac{3}{32}-\frac14
 =-\frac{5}{32}
 <-\frac18.
\]
Thus every part satisfies $g(t_j)\leq-1/8$, and consequently
\[
 \Delta_h
 =
 \sum_{j=1}^{k}t_jg(t_j)
 \leq-\frac18\sum_{j=1}^{k}t_j
 =-\frac{N}{8}.
\]
\end{proof}

\begin{lemma}[Amortized work at a selected call]
\label{lem:stad-call-charge}
For every call using Eq.~\ref{eq:stad-score} for selection,
\[
 |V(S)|+M_S\log(2M_S)
 +\sum_{\alpha=1}^{k_S}H_{S,\alpha}\log(H_{S,\alpha}+1)
 =O(r_S\log(2|V(S)|)).
\]
\end{lemma}
\begin{proof}
Suppress the subscript $S$ and write $s=|V(S)|$, $\ell=\ell_S$, and $r=s-\ell$. Singleton calls are immediate. If $\ell\leq3s/4$, then $r\geq s/4$, and the claim follows from Lemma~\ref{lem:stad-selection} and its proof, excluding the separate boundary-scan term $b_S$.

Suppose now that $\ell>3s/4$. If the selected backbone were the singleton centroid, every remaining component would have at most $s/2$ vertices, contradicting $\ell>3s/4$. Therefore, the selected backbone must be a diameter path $P$.

Let $c$ be the centroid used in the score comparison.
We claim that $c\notin V(P)$. Suppose, for a contradiction, that $c\in V(P)$. Let $C_1,\ldots,C_q$ be the connected components of $S\setminus\{c\}$. By the centroid property, for every $i$, $|V(C_i)|\leq {s}/{2}$. Since $c\in V(P)$, deleting $P$ is equivalent to first deleting $c$ and then deleting the remaining vertices of $P$. These additional deletions can only shrink or split the components $C_i$; they cannot join distinct components. Thus, every connected component $D$ of $S\setminus V(P)$ satisfies $V(D)\subseteq V(C_i)$, for some $i$ and consequently
\[
 |V(D)|\leq |V(C_i)|\leq \frac{s}{2}.
\]
The largest component after deleting $P$ therefore has size $\ell\leq s/2$, contradicting $\ell>3s/4$. Hence $c\notin V(P)$.

Since $c\notin V(P)$ and $P$ is connected, the entire path $P$ lies in one component of $S\setminus\{c\}$; denote this component by $A$. Let $C_1,\ldots,C_q$ be the other components, and write $a=|V(A)|$ and $a_i=|V(C_i)|$. The centroid property gives $a\leq s/2$ and $a_i\leq s/2$ for every $i$.

Separately, consider deleting $P$ from $S$, and let $B$ be the remaining component containing $c$. Because $V(P)\subseteq V(A)$, each $C_i$ remains intact and connected to $c$, so $B$ contains $c$ and all the components $C_i$. Every other component of $S\setminus V(P)$ therefore lies inside $A$ and has size at most $s/2$. Since the largest component has size $\ell>3s/4$, it must be $B$. Thus $B$ is the unique largest component and $|V(B)|=\ell$.

The sets $V(A)$ and $V(B)$ arise from different deletions and may overlap (for example, any vertices strictly between \(c\) and \(P\) remain connected to \(c\) after deleting \(P\)). Define $u=|V(A)\cap V(B)|$. All vertices outside $B$ lie in $A$, so
\[
 V(A)\setminus V(B)=V(S)\setminus V(B),
 \qquad
 |V(A)\setminus V(B)|=s-\ell=r.
\]
Hence $A$ consists of its $u$ shared vertices with $B$ and the $r$ vertices outside $B$. Meanwhile, $B$ consists of the centroid $c$, these $u$ shared vertices, and the components $C_1,\ldots,C_q$. These groups are disjoint, giving
\begin{align}
a=u+r,
 \qquad
 \ell=1+u+\sum_{i=1}^{q}a_i.
 \label{eq:sizeofAandB}
\end{align}
Finally, the centroid property implies $a_i\leq s/2$ for every $i$ and $V(A)\cap V(B)\subseteq V(A)$ implies $u\leq a\leq s/2$.

Now, remember, we write $P=(p_1,\ldots,p_m)$, with  $m=|V(P)|$, $h=H_P(S)$, and $M=(h+1)m$. Since $P$ contains $m$ vertices, its length is $m-1$ edges. Moreover, because $P$ is a diameter, the distance between any two vertices of $S$ is at most $m-1$. 

Choose a vertex $v$ whose distance from $P$ is $h$, and let $p_j$ be its projection onto $P$. Since $S$ is a tree, the paths from $v$ to the two endpoints of $P$ pass through $p_j$. Consequently, $\operatorname{dist}_S(v,p_1)=h+j-1,$ and $\operatorname{dist}_S(v,p_m)=h+m-j$. Both distances are bounded by the diameter length, so we get $h+j-1\leq m-1,$ and  $h+m-j\leq m-1$ respectively. Hence, summing the two inequalities, we therefore get $h\leq(m-1)/2$. It follows that $h+1\leq\frac{m+1}{2}\leq m.$ Multiplying the last inequality by $h+1$ yields $(h+1)^2\leq(h+1)m=M$. Thus,
\begin{equation}
 h\leq\frac{m-1}{2},
 \qquad
 (h+1)^2\leq M.
 \label{eq:stad-diameter-height}
\end{equation}

We next bound the table and correction terms for the centroid candidate. Its maximum off-backbone depth is its eccentricity, namely $H_c=\max_{v\in V(S)}\operatorname{dist}_S(c,v).$ Every distance in $S$ is at most $m-1$, so $H_c\leq m-1$. Since the centroid backbone contains only one vertex, its table size satisfies $M_c=(H_c+1)\cdot1\leq m$. Hence, using $M_c\leq m$ and $\max\{M_c,2\}\leq m+1$, its table term in the selection score obeys 
\[
M_c\log\!\bigl(\max\{M_c,2\}\bigr) \leq m\log(m+1).
\]

Similarly, the depth of the centroid component $A$, measured from $c$, is  $H_{c,A} =\max_{v\in V(A)}\operatorname{dist}_S(c,v) \leq m-1$. Its correction term therefore satisfies
\[
 H_{c,A}\log(H_{c,A}+1)
 \leq(m-1)\log m
 \leq m\log(m+1).
\]
Combining these two estimates gives
\begin{align}
 M_c\log\!\bigl(\max\{M_c,2\}\bigr)
 +H_{c,A}\log(H_{c,A}+1)
 \leq2m\log(m+1).
 \label{eq:twoscoreterms}
\end{align}

It remains to bound the corrections for the other centroid components $C_i$. Since $P$ lies in $A$, every path from a vertex $v\in V(C_i)$ to $P$ passes through $c$. Therefore, $\operatorname{dist}_S(v,P)$ = $\operatorname{dist}_S(v,c) +\operatorname{dist}_S(c,P)$. By the definition of $h$, the left side is at most $h$. Since distances are nonnegative, this implies $ \operatorname{dist}_S(v,c)\leq h$. There is also a bound in terms of the component size $a_i$. The path from $c$ to $v$ uses one edge to enter $C_i$ and at most $a_i-1$ edges within $C_i$. Thus, $\operatorname{dist}_S(c,v)\leq1+(a_i-1)=a_i$. Taking the maximum over $v\in V(C_i)$ in both bounds gives:  $H_{c,i} :=\max_{v\in V(C_i)}\operatorname{dist}_S(c,v) \leq\min\{a_i,h\}$. In particular, $\log(H_{c,i}+1)\leq\log(h+1)$. Both factors are nonnegative, so the correction term satisfies
\[
 H_{c,i}\log(H_{c,i}+1)
 \leq\min\{a_i,h\}\log(h+1).
\]

Altogether, the centroid candidate's table and correction terms are bounded by
 \begin{align}
M_c\log\!\bigl(\max\{M_c,2\}\bigr)
 +H_{c,A}\log(H_{c,A}+1)
 +\sum_i H_{c,i}\log(H_{c,i}+1)\notag \leq
 2m\log(m+1)
 +\sum_i\min\{a_i,h\}\log(h+1).
 \end{align}

We now compare the recursive-size penalties of the two candidates. Recall that we are considering the case $\ell>3s/4$, in which the selected backbone $P$ is a diameter and the comparison centroid $c$ lies outside $P$. The preceding decomposition gives Eq.\ref{eq:sizeofAandB} i.e. $a=u+r,$ and $\ell=1+u+\sum_i a_i,$ alongside $a\leq s/2,$ and $a_i\leq s/2.$ 

Let $\varphi(t)=t\log(t+1)$, with $\varphi(0)=0$. If $n_\alpha$ denotes the size of an off-backbone component under the diameter deletion, then one such component has size $\ell$. Since all recursive-size penalties are nonnegative, $ \sum_\alpha\varphi(n_\alpha)\geq\varphi(\ell).$ For the centroid candidate, the component sizes are $a,a_1,a_2,\ldots$, so its recursive-size penalty is exactly $\varphi(a)+\sum_i\varphi(a_i).$

We first relate $\varphi(a)$ to $\varphi(u)$. We know that $u<a$. Differentiation gives $\varphi'(t) = \log(t+1)+\frac{t}{t+1}$. For every $t\in[u,a]$, we have $0\leq t\leq s$, and hence $\varphi'(t)\leq\log(s+1)+1.$ Integrating this bound and using $a-u=r$, we obtain
\begin{align*}
 \varphi(a)-\varphi(u)
 &=
 \int_u^a\varphi'(t)\,dt\\
 &\leq
 \int_u^a\bigl(\log(s+1)+1\bigr)\,dt\\
 &=
 (a-u)\bigl(\log(s+1)+1\bigr)\\
 &=
 r\bigl(\log(s+1)+1\bigr).
\end{align*}
Equivalently, $\varphi(a) \leq \varphi(u)+r\bigl(\log(s+1)+1\bigr).$

Since $c\notin V(P)$, we have $h\geq1$. Eq.~\ref{eq:stad-diameter-height} then gives $m\geq2h+1\geq3$, and therefore $M=(h+1)m\geq6$. Thus the diameter's table term in the score is $M\log M$, and its score satisfies
\begin{align}
 \mathcal C(S,P)
 &=
 s+M\log M
 +\sum_\alpha H_\alpha\log(H_\alpha+1)
 +\sum_\alpha\varphi(n_\alpha) \notag\\
 &\geq
 s+M\log M
 +\sum_\alpha H_\alpha\log(H_\alpha+1)
 +\varphi(\ell).
 \label{eq:diameterscorebig}
\end{align}
The preceding bounds for the centroid's table and correction terms give
\begin{align}
\label{eq:centroidscorebig}
 \mathcal C(S,(c))
 \leq{}&
 s+2m\log(m+1)
 +\sum_i\min\{a_i,h\}\log(h+1)+\varphi(a)+\sum_i\varphi(a_i).
\end{align}

Because the diameter was selected by using Eq.~\ref{eq:stad-score}, $\mathcal C(S,P)\leq\mathcal C(S,(c)).$ Combining the lower bound for the diameter score with the upper bound for the centroid score, and subtracting $s+\varphi(\ell)$ from both sides, yields
\begin{align*}
 &M\log M+\sum_\alpha H_\alpha\log(H_\alpha+1)\\
 &\quad\leq
 2m\log(m+1)
 +\sum_i\min\{a_i,h\}\log(h+1)
 +\varphi(a)+\sum_i\varphi(a_i)-\varphi(\ell)\\
 &\quad\leq
 2m\log(m+1)
 +r\bigl(\log(s+1)+1\bigr)
 +\sum_i\min\{a_i,h\}\log(h+1)
 +\varphi(u)+\sum_i\varphi(a_i)-\varphi(\ell).
\end{align*}

Every vertex of $P$ lies outside $B$, and $r=s-\ell$ counts all vertices outside $B$. Hence $m\leq r$. Also, since $m\leq s$ and $s\geq2$,
\[
 \log(m+1)\leq\log(2s),
 \qquad
 \log(s+1)+1\leq2\log(2s).
\]
Consequently,
\begin{align*}
 2m\log(m+1)+r\bigl(\log(s+1)+1\bigr)
 &\leq
 2r\log(2s)+2r\log(2s)\\
 &=4r\log(2s).
\end{align*}
Thus, for an absolute constant $K$ (one may take $K=4$),
\begin{align*}
 &M\log M+\sum_\alpha H_\alpha\log(H_\alpha+1)\\
 &\quad\leq
 Kr\log(2s)
 +\sum_i\min\{a_i,h\}\log(h+1)
 +\varphi(u)+\sum_i\varphi(a_i)-\varphi(\ell).
\end{align*}

To bound the remaining terms, apply Lemma~\ref{lem:stad-merge-penalty} with $N=\ell$ and partition parts $ 1,\ u,\ a_1,a_2,\ldots,$ omitting $u$ when it is zero. These parts sum to $\ell$ because $\ell=1+u+\sum_i a_i$. Let $\Delta_h$ denote the quantity defined in that lemma for this partition. Expanding its definition gives
\begin{align*}
 \Delta_h
 ={}&
 \sum_i\min\{a_i,h\}\log(h+1)
 +\varphi(u)+\sum_i\varphi(a_i)-\varphi(\ell)\\
 &+
 \underbrace{
 \min\{1,h\}\log(h+1)+\varphi(1)
 +\min\{u,h\}\log(h+1)
 }_{\geq0}.
\end{align*}
This identity also holds when $u=0$, since the terms involving $u$ then vanish. The additional terms on the second line are nonnegative. Therefore,
\begin{align}
 &M\log M+\sum_\alpha H_\alpha\log(H_\alpha+1)
 \notag\\
 &\quad\leq
 Kr\log(2s)
 +\sum_i\min\{a_i,h\}\log(h+1)
 +\varphi(u)+\sum_i\varphi(a_i)-\varphi(\ell)
 \notag\\
 &\quad\leq Kr\log(2s)+\Delta_h.
 \label{eq:stad-key-score-comparison}
\end{align}

The first assertion of Lemma~\ref{lem:stad-merge-penalty} always applies to this partition. Together with Eq.~\ref{eq:stad-diameter-height}, it gives
\[
 \Delta_h\leq(h+1)^2\leq M.
\]
Substituting M into Eq. \ref{eq:stad-key-score-comparison}, we obtain
\[
 M\log M+\sum_\alpha H_\alpha\log(H_\alpha+1)
 \leq Kr\log(2s)+M.
\]
Because $M\leq(M\log M)/\log6$ and $1/\log6<1$, we can absorb this $M$ term into the left side. Additionally $M= O(r\log(2s)).$ Therefore,
\begin{equation}
 M\log M+\sum_\alpha H_\alpha\log(H_\alpha+1)
 =O(r\log(2s)).
 \label{eq:stad-selected-transform-charge}
\end{equation}

It remains to bound the cost of processing the $s$ vertices at the call. Since $\ell>3s/4$, we have $s<\frac{4\ell}{3}.$ We distinguish two cases. First, suppose that $\ell<16(h+1)^2$. Then
\[
 s
 <\frac{4\ell}{3}
 <\frac{64}{3}(h+1)^2
 \leq\frac{64}{3}M
 =O(r\log(2s)).
\]

We also record when the stronger, negative bound applies. Every part of the partition is at most $s/2$: this holds for the $a_i$ by the centroid property, for $u$ because $u\leq a\leq s/2$, and for the part $1$ because $s\geq2$. Moreover,
\[
 \ell>\frac{3s}{4}
 \quad\Longrightarrow\quad
 \frac{s}{2}<\frac{2\ell}{3}.
\]
Thus the largest partition part is at most $2\ell/3$. Now, if $\ell\geq16(h+1)^2$, the second assertion of Lemma~\ref{lem:stad-merge-penalty} therefore gives $ \Delta_h\leq-\frac{\ell}{8}.$ Returning to Eq.~\ref{eq:stad-key-score-comparison},
we then obtain
\[
 0
 \leq
 M\log M+\sum_\alpha H_\alpha\log(H_\alpha+1)
 \leq
 Kr\log(2s)-\frac{\ell}{8}.
\]
The right side must consequently be nonnegative, which implies $\ell\leq8Kr\log(2s)$. Using $s=\ell+r$ and $\log(2s)\geq1$ for $s\geq2$, we obtain
\begin{align}
\label{eq:boundons}
 s
 =\ell+r
 \leq8Kr\log(2s)+r
 \leq(8K+1)r\log(2s)
 =O(r\log(2s)).   
\end{align}
Thus both cases give the required bound on vertex processing.

Combining this estimate with Eq.~\ref{eq:stad-selected-transform-charge} and the bound
on $s$ yields
\[
 s+M\log(2M)
 +\sum_\alpha H_\alpha\log(H_\alpha+1)
 =O(r\log(2s)).
\]
This establishes the claimed bound in the remaining case $\ell>3s/4$. Together with the previously treated case $\ell\leq3s/4$ and the singleton calls, it proves the lemma.
\end{proof}

\begin{proof}[Proof of  Theorem~\ref{thm:stad-adaptive-complexity}]
Exactness in exact arithmetic follows from
Theorem~\ref{thm:stad-decomposition-cost}.
We therefore prove the running-time bound. By Eq.~\ref{eq:summingrunning}, the total running time is
\[
 \begin{aligned}
 T_{\mathrm{total}}
 =O\!\Biggl(
 \sum_{S\in\mathcal R}
 \Bigl[
 |V(S)|+b_S+M_S\log(2M_S)
 +\sum_{\alpha=1}^{k_S}
 H_{S,\alpha}\log(H_{S,\alpha}+1)
 \Bigr]
 \Biggr).
 \end{aligned}
\]
Applying Lemma~\ref{lem:stad-call-charge} at every call gives $T_{\mathrm{total}}= O\!\left(\sum_{S\in\mathcal R}r_S\log(2|V(S)|)+\sum_{S\in\mathcal R}b_S \right)$.

From the boundary-edge estimate in Eq.~\ref{eq:stad-boundary-volume}, we get  $\sum_{S\in\mathcal R}b_S \leq \sum_{S\in\mathcal R}|V(S)|-n \leq \sum_{S\in\mathcal R}|V(S)|$. Additionally from Eq.~\ref{eq:boundons} in Lemma~\ref{lem:stad-call-charge}, we know that $|V(S)|=O\!\left(r_S\log(2|V(S)|)\right)$ uniformly over all calls. Summing this bound shows that $\sum_{S\in\mathcal R}b_S = O\!\left( \sum_{S\in\mathcal R}r_S\log(2|V(S)|) \right)$. Consequently, $T_{\mathrm{total}}= O\!\left( \sum_{S\in\mathcal R}r_S\log(2|V(S)|) \right)$.

It remains to bound this sum $\sum_{S\in\mathcal R}r_S\log(2|V(S)|)$. We use a charging argument: we assign numerical costs to vertices and bound the total cost assigned to each vertex. This assignment is only an accounting device for the proof.

At every call with children, fix one largest child, resolving ties arbitrarily. At call $S$, assign a charge of $\log(2|V(S)|)$ to every vertex outside that chosen largest child. If $S$ has no children, charge every vertex of $S$. Since the largest child contains $\ell_S$ vertices, the number of charged vertices is exactly $r_S=|V(S)|-\ell_S.$ The total charge assigned at call $S$ is therefore
\[
 \underbrace{r_S}_{\text{number of charged vertices}}
 \underbrace{\log(2|V(S)|)}_{\text{charge per vertex}}.
\]

We next show why this choice of charged vertices is useful. Fix a vertex $v$ charged at a call $S$. If $v$ belongs to the selected backbone, it is removed at this call and receives no subsequent charges. Otherwise, $v$ belongs to a child $B$ different from the chosen largest child $B_{\max}$.

Since $B_{\max}$ is a largest child, $|V(B)|\leq|V(B_{\max})|.$ The two children are disjoint, and their vertices are among those remaining after removal of the $m_S$ backbone
vertices. Consequently,
\[
 \begin{aligned}
 2|V(B)|
 &\leq |V(B)|+|V(B_{\max})|\\
 &\leq |V(S)|-m_S \\
 &<|V(S)|,
 \end{aligned}
\]
In particular, $|V(B)|\leq\frac{|V(S)|}{2}.$ Thus, whenever a charged vertex survives the call, it enters a child containing at most half the vertices of the current subproblem. The vertex may pass through several calls without receiving another charge, by remaining in the chosen largest child at those calls. Nevertheless, every subsequent subproblem containing $v$ is contained in $B$. Therefore, at the next call where $v$ is charged, the subproblem size is at most $|V(S)|/2$.

We now bound the number of charges received by any fixed
vertex $v$. Suppose that it receives $q$ charges, and
let $N_1,\ldots,N_q$ be the subproblem sizes at those
charged calls, in chronological order. The first size
is at most $n$, and the preceding argument gives
\[
 N_1\leq n,
 \qquad
 N_{j+1}\leq\frac{N_j}{2}
 \quad\text{for }j=1,\ldots,q-1.
\]
Repeated substitution yields
\[
 N_2\leq\frac n2,\qquad
 N_3\leq\frac n{2^2},\qquad
 \ldots,\qquad
 N_q\leq\frac n{2^{q-1}}.
\]
There are $q-1$ halvings between $q$ charged calls.
Since the last subproblem contains $v$, its size is
at least one. Hence
\[
 1\leq N_q\leq\frac n{2^{q-1}}.
\]
Rearranging and taking base-two logarithms gives
\[
 2^{q-1}\leq n,
 \qquad
 q-1\leq\log_2 n.
\]
Because $q$ is an integer, $ q\leq1+\lfloor\log_2 n\rfloor.$

Every subproblem has at most $n$ vertices, so each charge
is at most $\log(2n)$. The total charge assigned to any
one vertex is consequently at most $\bigl(1+\lfloor\log_2 n\rfloor\bigr)\log(2n).$

We now count the same total charge by recursive calls
and by vertices. For each call $S$ with children, let
$B_{\max,S}$ be its chosen largest child. For a call
without children, interpret
$V(B_{\max,S})=\varnothing$.
Define the charge assigned to vertex $v$ at call $S$ by
\[
 a_{S,v}
 =
 \begin{cases}
 \log(2|V(S)|),
 & v\in V(S)\setminus V(B_{\max,S}),\\
 0,
 & \text{otherwise}.
 \end{cases}
\]
These quantities can be viewed as entries of a table
whose rows correspond to recursive calls and whose
columns correspond to vertices.

At call $S$, exactly $r_S$ vertices receive a nonzero charge. Thus the sum along its row is $\sum_{v\in V(T)}a_{S,v} = r_S\log(2|V(S)|)$. For a fixed vertex $v$, the sum along its column is the total charge it receives throughout the recursion. The preceding halving argument shows that $v$ receives at most $1+\lfloor\log_2 n\rfloor$ charges, each at most $\log(2n)$. Therefore,
$ \sum_{S\in\mathcal R}a_{S,v} \leq \bigl(1+\lfloor\log_2 n\rfloor\bigr)\log(2n).$

Since both index sets are finite, we may interchange
the order of summation. Counting the charges first
by calls and then by vertices gives
\begin{align*}
 \sum_{S\in\mathcal R}r_S\log(2|V(S)|)
 &=
 \sum_{S\in\mathcal R}
 \sum_{v\in V(T)}a_{S,v}\\
 &=
 \sum_{v\in V(T)}
 \sum_{S\in\mathcal R}a_{S,v}\\
 &\leq
 \sum_{v\in V(T)}
 \bigl(1+\lfloor\log_2 n\rfloor\bigr)\log(2n)\\
 &=
 n\bigl(1+\lfloor\log_2 n\rfloor\bigr)\log(2n)\\
 &=O\!\left(n\log^2(2n)\right).
\end{align*}
One logarithmic factor bounds the number of charges per vertex, and the other bounds the amount of each charge.
\end{proof}

\begin{corollary}[An off-backbone-depth bound for adaptive selection]
\label{cor:stad-adaptive-depth-runtime} Under the assumptions of Theorem~\ref{thm:stad-decomposition-cost}, suppose that every nontrivial recursive call selects either a diameter path or a singleton centroid as its backbone. Let $n=|V(T)|$, let $L$ be the maximum number of calls on a root-to-leaf path in the recursion tree, and define $\overline H=\max_{S\in\mathcal R}H_S$, with $H_S=0$ for singleton base cases. Then $L\leq\overline H+\lfloor\log_2 n\rfloor+1,$ and the total running time is $O\!\left((\overline H+1)n\log(2n)\right)$. In particular, this bound applies to the adaptive selection rule which uses Eq.~\ref{eq:stad-score}.
\end{corollary}

\begin{proof}
By Lemma~\ref{lem:stad-selection}, adaptive backbone selection takes $O(|V(S)|)$ time at each recursive call. Therefore, Corollary~\ref{cor:stad-depth-runtime} applies to the resulting decomposition and gives $O\!\left( nL+(\overline H+1)n\log(2n) \right)$. It remains to bound the repeated-processing term $nL$. We show that the adaptive choice between a diameter path and a singleton centroid ensures
\[
L\leq
\overline H+\lfloor\log_2 n\rfloor+1.
\]

\textit{\textbf{(i) Centroid deletion reduces subtree size.}}
By the definition of a centroid, deleting the selected centroid of a subtree $S$ leaves components containing at most $|V(S)|/2$ vertices each. Other recursive transitions cannot increase subtree size. Hence, after $c$ centroid transitions along a recursive chain, the current subtree has at most $n/2^c$ vertices. Since it is nonempty, $\boxed{c\leq\lfloor\log_2 n\rfloor}.$

\textit{\textbf{(ii) Diameter-path deletion reduces diameter.}}
Suppose that the selected backbone $P$ is a
diameter path of $S$, with endpoints $a,b$ and
length $D=\operatorname{diam}(S)$.
Let $B$ be a component of $S\setminus V(P)$,
let $p\in V(P)$ be its attachment vertex, and
let $r\in V(B)$ be adjacent to $p$.

For every $u\in V(B)$, the paths from $u$ to $a$ and $b$ pass through $p$. Because neither distance exceeds $D$,
\[
\dist_S(u,p)+\dist_S(p,a)\leq D,
\qquad
\dist_S(u,p)+\dist_S(p,b)\leq D.
\]
Since $\dist_S(p,a)+\dist_S(p,b)=D,$ these inequalities imply
\[
\dist_S(u,p)
\leq
\min\{\dist_S(p,a),\dist_S(p,b)\}
\leq \frac{D}{2}.
\]
For any $u,v\in V(B)$, the triangle inequality within $B$ therefore gives
\begin{align*}
\dist_B(u,v)
&\leq \dist_B(u,r)+\dist_B(r,v)\\
&=\dist_S(u,p)+\dist_S(v,p)-2\\
&\leq D-2.
\end{align*}
Hence every surviving component satisfies
\begin{equation}
\operatorname{diam}(B)
\leq \operatorname{diam}(S)-2.
\label{eq:stad-diameter-decrease1}
\end{equation}

Consider any root-to-leaf chain of recursive calls. If the root has no children, this chain contains only one call. Otherwise, let $B$ be the first child on the chain, let $p$ be its attachment vertex on the root backbone, and let $r\in V(B)$ be adjacent to $p$. By the definition of $\overline H$,
\[
\dist_B(u,r)
=
\dist_T(u,p)-1
\leq \overline H-1
\qquad\text{for every }u\in V(B).
\]
Therefore, 
\begin{equation}
    \operatorname{diam}(B) \leq 2(\overline H-1).
    \label{eq:diameter-ineq2}
\end{equation}

Let $d$ denote the number of diameter-backbone calls after the root that produce another child on this chain. Each such transition decreases the diameter by at least two (by Eq \ref{eq:stad-diameter-decrease1}), while intervening centroid transitions cannot increase it. Hence, if $S_{\mathrm{last}}$ is the subtree processed by the final call, then using Eq. \ref{eq:stad-diameter-decrease1} first and then Eq. \ref{eq:diameter-ineq2}, we get
\[
0
\leq \operatorname{diam}(S_{\mathrm{last}})
\leq \operatorname{diam}(B)-2d
\leq 2\overline H-2-2d.
\]
The nonnegativity of the diameter therefore implies $2d\leq 2\overline H-2$, and hence $\boxed{d\leq\overline H-1}.$

\paragraph{Bounding the number of recursive calls.}
Every intermediate call is either a diameter-backbone call or a centroid call. Therefore, the number $\ell$ of calls on the chain satisfies
\begin{align*}
\ell
&=
\underbrace{1}_{\text{root}}
+
\underbrace{d}_{\substack{\text{diameter}\\\text{calls}}}
+
\underbrace{c}_{\substack{\text{centroid}\\\text{calls}}}
+
\underbrace{1}_{\substack{\text{terminal}\\\text{call}}}\\
&\leq
1+(\overline H-1)+\lfloor\log_2 n\rfloor+1\\
&=
\overline H+\lfloor\log_2 n\rfloor+1.
\end{align*}
The terminal call is counted separately because it produces no child and therefore contributes no additional recursive transition. Since this bound holds for every root-to-leaf chain, $L\leq\overline H+\lfloor\log_2 n\rfloor+1.$ If the root itself is terminal, then $L=1$, and the same bound holds. It follows that
\begin{align*}
    nL &= O\!\left( n(\overline H+1)+n\log(2n)\right)\\
       &\le O\!\left( n(\overline H + \log(2n))\right)\\
       &\le O\!\left( n(\overline H\log(2n) + \log(2n))\right)\\
       &= O\!\left( n(\overline H+1)\log(2n)\right).
\end{align*}

Substituting into the bound from Corollary~\ref{cor:stad-depth-runtime} yields the claimed total running time $O\!\left( (\overline H+1)n\log(2n) \right)$.
\end{proof}

\begin{proof}[Proof of
Theorem~\ref{thm:stad-adaptive-geometry}]
First, we prove the explicit vertex-processing bound
\begin{equation}
 \sum_{S\in\mathcal R}|V(S)| \leq n\log_2(n+1).
 \label{eq:stad-specialized-volume}
\end{equation}

Set $\varphi(t)=t\log(t+1)$, with $\varphi(0)=0$, and define $D_S= \varphi(|V(S)|) - \sum_{B\text{ child of }S}\varphi(|V(B)|)$. We will show that at every call:
\begin{equation}
 D_S\geq |V(S)|\log 2
 \label{eq:stad-specialized-potential-decrease}
\end{equation}

Fix a call and write $s=|V(S)|$. For a singleton, $D_S=\varphi(1)=\log 2$, so Eq.~\ref{eq:stad-specialized-potential-decrease} holds with equality. Suppose henceforth that $s\geq2$. Let $P$ be the selected backbone and let $c$ be the centroid used in the score comparison.

\medskip
\noindent\textit{Case 1: $c\in V(P)$.}
Every child after deleting $P$ is contained in a component after deleting $c$. Thus every child size $n_\alpha$ is at most $s/2$. Moreover, at least one vertex is removed, so $\sum_\alpha n_\alpha\leq s-1$. Consequently,
\begin{align*}
 D_S \geq
 s\log(s+1)-(s-1)\log(s/2+1)=
 s\log 2
 +\log\frac{s+1}{2}
 -(s-1)\log\left(1+\frac{1}{s+1}\right).
\end{align*}
The elementary inequalities
$\log(1+x)\leq x$ for $x\geq0$ and
$\log y\geq1-1/y$ for $y>0$ give
\[
 (s-1)\log\left(1+\frac{1}{s+1}\right)
 \leq\frac{s-1}{s+1}
 \leq\log\frac{s+1}{2}.
\]
Therefore $D_S\geq s\log 2$.
This case includes selection of the singleton centroid.

\medskip
\noindent\textit{Case 2: $c\notin V(P)$.}
The selected backbone is then a diameter path. Write $m=|V(P)|$,  $h=H_P(S)$ and $M=(h+1)m$. Let $A$ be the component of $S\setminus\{c\}$ containing $P$, let $a=|V(A)|$, and let $a_i$ be the sizes of the other centroid components. Then
\[
 s=1+a+\sum_i a_i,
 \qquad
 m\leq a\leq s/2,
 \qquad
 a_i\leq s/2.
\]

Here $h\geq2$. Indeed, $s\geq3$, and some centroid component other than $A$ must exist because $a\leq s/2$. A neighbor of $c$ in that component has distance at least two from $P$. Eq.~\ref{eq:stad-diameter-height} therefore gives
\[
 m\geq2h+1\geq5,
 \qquad
 M=(h+1)m\geq15.
\]

Write $t_1,\ldots,t_q$ for the partition parts $1,a,a_1,a_2,\ldots$, which sum to $s$, and every part is at most $s/2$. and define
\[
 \widehat\Delta_h
 =
 \sum_{j=1}^{q}
 \left(
 \min\{t_j,h\}\log(h+1)+\varphi(t_j)
 \right)
 -
 \varphi(s).
\]

We first compare the candidate scores. As established in the proof of Lemma~\ref{lem:stad-call-charge}, the centroid's table term and correction for $A$ total at most $2m\log(m+1)$ (from Eq.\ref{eq:twoscoreterms}). The correction for each other centroid component is at most $\min\{a_i,h\}\log(h+1)$. It follows from Eq.\ref{eq:centroidscorebig} that
\begin{align*}
 \mathcal C(S,(c))
 \leq{}&
 s+2m\log(m+1)
 +\sum_i\min\{a_i,h\}\log(h+1)+\varphi(a)+\sum_i\varphi(a_i)\\
 \leq{}&
 s+2m\log(m+1)+\varphi(s)+\widehat\Delta_h.
\end{align*}

For the selected diameter, nonnegativity of its correction terms gives
\[
 \mathcal C(S,P)
 \geq
 s+M\log M+\varphi(s)-D_S.
\]
Since the selection rule ensures
$\mathcal C(S,P)\leq\mathcal C(S,(c))$, we obtain
\begin{equation}
 D_S
 \geq
 M\log M-2m\log(m+1)-\widehat\Delta_h.
 \label{eq:stad-specialized-score-decrease}
\end{equation}

We now sharpen the merging estimate using the fact that every partition part is at most $s/2$. Set
\[
 Q=(h+1)^2,
 \qquad
 g(x)
 =
 \min\{1,h/x\}\log(h+1)
 +\log(x+1)-\log(s+1).
\]
Then $\widehat\Delta_h=\sum_{j=1}^{q}t_jg(t_j).$ The endpoint argument in the proof of Lemma~\ref{lem:stad-merge-penalty} shows that $g$ is increasing on $[1,h]$ and has no interior maximum on $[h,s/2]$. Since $h<s/2$, this implies $\widehat\Delta_h \leq s\max\{g(h),g(s/2)\}.$

At the first endpoint, using $\log x\leq x-1$,
\begin{align*}
 s g(h)
 = s\log\frac{Q}{s+1} \leq
 s\left(\frac{Q}{s+1}-1\right) \leq Q-s \leq Q-s\log 2.
\end{align*}

For the second endpoint, observe that $\log(h+1)\leq(h+1)/2$ for $h\geq2$. Hence, $2h\log(h+1)+1 \leq h(h+1)+1\leq Q.$ Using $\log(1+x)\leq x$, we therefore obtain
\begin{align*}
 s g(s/2)
 &=
 2h\log(h+1)
 +s\log\frac{s/2+1}{s+1}\\
 &=
 2h\log(h+1)-s\log 2
 +s\log\left(1+\frac{1}{s+1}\right)\\
 &\leq
 2h\log(h+1)+1-s\log 2\\
 &\leq Q-s\log 2.
\end{align*}
Thus
\begin{equation}
 \widehat\Delta_h\leq Q-s\log 2.
 \label{eq:stad-specialized-balanced-merge}
\end{equation}

It remains to compare $Q$ with the difference between the diameter's table term and the centroid estimate. Since $M\geq m+1$ and $h\geq2$,
\[
 2m\log(m+1)
 \leq \frac{2}{h+1}M\log M
 \leq \frac23 M\log M.
\]
Furthermore,
\[
 \frac{M}{Q}
 =
 \frac{m}{h+1}
 \geq\frac{2h+1}{h+1}
 \geq\frac53,
\]
and $M\geq15$ implies $\log M\geq2$.
Consequently,
\begin{align*}
 M\log M-2m\log(m+1) \geq \frac13 M\log M \geq \frac23 M \geq \frac{10}{9}Q \geq Q.
\end{align*}
Combining this inequality with Eqs.~\ref{eq:stad-specialized-score-decrease} and~\ref{eq:stad-specialized-balanced-merge} yields $D_S \geq Q-\bigl(Q-s\log 2\bigr) =s\log 2$. This completes the proof of Eq.~\ref{eq:stad-specialized-potential-decrease}.

\medskip
Finally, the potential differences telescope: every nonroot call appears once in the first sum and once as a child in the second sum, so its contributions cancel. Only the root call remains; it contains all $n$ vertices and has no parent. Therefore,
\begin{align*}
 \sum_{S\in\mathcal R}D_S
 &=
 \sum_{S\in\mathcal R}\varphi(|V(S)|)
 -
 \sum_{S\in\mathcal R}
 \sum_{B\text{ child of }S}\varphi(|V(B)|) =\varphi(n).
\end{align*}
Using Eq. \ref{eq:stad-specialized-potential-decrease} and substituting the above we get
\[
 \sum_{S\in\mathcal R}|V(S)|
 \leq
 \frac{1}{\log 2}\sum_{S\in\mathcal R}D_S
 =
 \frac{n\log(n+1)}{\log 2}
 =
 n\log_2(n+1).
\]

First, the potential-decrease estimate
Eq.~\ref{eq:stad-specialized-potential-decrease}, with
$\varphi(t)=t\log(t+1)$, gives
\[
 |V(S)|\log 2
 \leq
 \varphi(|V(S)|)
 -
 \sum_{B\text{ child of }S}\varphi(|V(B)|).
\]
Summing over calls cancels every nonroot potential term.
Therefore,
\begin{equation}
 \sum_{S\in\mathcal R}|V(S)|
 \leq
 \frac{\varphi(n)}{\log 2}
 =
 n\log_2(n+1).
 \label{eq:stad-adaptive-volume}
\end{equation}

Next, we bound each component correction by the interaction-table cost of its child call. Let $B$ be a component created at call $S$, and let $r\in V(B)$ be its vertex adjacent to the parent backbone. Counting the edges in the longest path of $B$, we denote $D_B=\operatorname{diam}(B)$. Therefore, for the component $B$, following notation in Eq.~\ref{eq:stad-depth}, we have:
\[
 H_{S,B}
 =
 1+\max_{v\in V(B)}\dist_B(r,v)
 \leq D_B+1.
\]

If the child call selects a diameter backbone, then the number of vertices in the backbone, $m_B=D_B+1$, and hence
\[
 H_{S,B}+1
 \leq D_B+2
 =m_B+1
 \leq2M_B.
\]
If it selects a singleton centroid $c_B$, then $M_B=H_B+1$ and $D_B\leq2H_B$, giving
\[
 H_{S,B}+1
 \leq D_B+2
 \leq2(H_B+1)
 =2M_B.
\]
Therefore, the same inequality holds for a singleton child.

Consequently,
\[
 H_{S,B}\log(H_{S,B}+1)
 \leq2M_B\log(2M_B).
\]
Every nonroot call occurs exactly once as a child, so
\begin{equation}
 \sum_{S\in\mathcal R}\sum_{\alpha=1}^{k_S}
 H_{S,\alpha}\log(H_{S,\alpha}+1)
 \leq
 2\sum_{B\in\mathcal R\setminus\{T\}}
 M_B\log(2M_B).
 \label{eq:stad-corrections-child-tables}
\end{equation}

Applying Theorem~\ref{thm:stad-decomposition-cost}
and Eq.~\ref{eq:stad-adaptive-volume} now yields
\[
 T_{\mathrm{adaptive}}
 =
 O\!\left(
 n\log(2n)
 +\sum_{S\in\mathcal R}M_S\log(2M_S)
 \right).
\]

Finally, Lemma~\ref{lem:stad-selection} gives
$M_S=O(|V(S)|)=O(n)$, so
$\log(2M_S)=O(\log(2n))$.
The removed backbones partition the vertices, giving
$\sum_S m_S=n$. Thus,
\[
 \sum_{S\in\mathcal R}M_S
 =
 \sum_{S\in\mathcal R}(H_S+1)m_S
 =
 n+\sum_{S\in\mathcal R}m_SH_S.
\]
Substitution proves the Theorem.
\end{proof}

\begin{proof}[Proof of Corollary \ref{cor:stad-adaptive-equal-spider}]
The center $c$ is the unique centroid. For the diameter candidate, we know $M=(h+1)(2h+1).$

Using the component-local correction terms in
Eq.~\ref{eq:stad-score}, the two scores are
\begin{align*}
 \mathcal C(T,(c))
 &=
 n+(h+1)\log(h+1)+2sh\log(h+1),\\
 \mathcal C(T,P_{\mathrm{diam}})
 &=
 n+M\log M+2(s-2)h\log(h+1).
\end{align*}
Their difference is
\[
 \mathcal C(T,P_{\mathrm{diam}})
 -\mathcal C(T,(c))
 =
 M\log M-(5h+1)\log(h+1)>0.
\]
Indeed,
\[
 M-(5h+1)=2h(h-1)\geq0,
 \qquad M>h+1.
\]
The adaptive rule therefore selects the center.

It remains to verify the adaptive choice on each child. Let $B$ be an arm component with $t=|V(B)|$ vertices. If $t=1$, the algorithm returns its singleton base case. Suppose that $t\geq2$. The diameter backbone is the entire path $B$. Thus $m=t$, $H=0$, and $M=t$, with no remaining components. Its score is therefore
\[
 \mathcal C(B,P_{\mathrm{diam}})
 =t+t\log t.
\]

For the singleton centroid candidate, let $a$ and $b$ be the numbers of vertices on its two sides, ordered so that $a\leq b\leq a+1,$ and $a+b=t-1$. A side of size zero is omitted from the decomposition. Since each nonempty side is a path attached at an endpoint, its component depth equals its size. Consequently, the centroid candidate has $H=b$,
$M=b+1$, and score
\[
 \mathcal C(B,(c_B))
 =
 t+(b+1)\log(b+1)
 +2a\log(a+1)+2b\log(b+1).
\]
The factor $2$ accounts for the component correction and the recursive-size penalty, which coincide for each side.

If $t=2$, then $a=0$ and $b=1$, giving
\[
 \mathcal C(B,(c_B))
 =2+4\log2
 >
 2+2\log2
 =\mathcal C(B,P_{\mathrm{diam}}).
\]

For $t\geq3$, we have $a\geq1$ and $b\geq1$.
Since $\log(b+1)\geq\log(a+1)$,
\begin{align*}
 \mathcal C(B,(c_B))-t
 &\geq (2a+3b+1)\log(a+1)\\
 &\geq 2t\log(a+1)\\
 &\geq t\log t.
\end{align*}
The second inequality follows from $2a+3b+1=2t+(b-1)\geq2t.$ For the last inequality, use $b\leq a+1$ and $a\geq1$ to obtain $t=a+b+1\leq2(a+1)\leq(a+1)^2,$ and hence $\log t\leq2\log(a+1)$.

Thus the entire-path candidate has score no larger than the centroid candidate. Since ties favor the diameter, the adaptive rule selects the entire path in every nontrivial child call. Each arm is therefore resolved in one call, with no further recursion. The selected decomposition consequently satisfies
\[
 \sum_{S\in\mathcal R}m_SH_S
 =
 1\cdot h+\sum_{\alpha=1}^{s}h\cdot0
 =
 h.
\]
Theorem~\ref{thm:stad-adaptive-geometry} therefore gives
$O((n+h)\log(2n))=O(n\log(2n))$.

More precisely, this is exactly the center construction
of Proposition~\ref{prop:stad-spider-comparison}, whose
cost is $O(n\log(h+1))$.
Evaluating the two candidate scores adds only $O(n)$
work over this decomposition; the rejected diameter
table is never constructed.
\end{proof}

\subsection{\STADTFI\ (Specialized Kernels) : Near Linear Run Time}
\label{sec:stad-specialized}
\vspace{-5pt}
For exponential kernels, linear-time integration on trees is already established by \citet[Lemma~3.5]{choromanski2022block}. We show how this structure fits into our backbone aggregation and correction formulas, allowing both $G$ and $E_\alpha$. Specifically, we consider kernels supplied as $J-$finite sums of exponentials, with $c_\ell\in\mathbb R,$ and $\lambda_\ell>0$: $f(d)=\sum_{\ell=1}^{J}c_\ell\lambda_\ell^d$.


\begin{lemma}[Cross interactions for exponential sums]
\label{lem:stad-specialized-cross}
For a kernel supplied in the form $f(d)=\sum_{\ell=1}^{J}c_\ell\lambda_\ell^d$, all cross-component contributions at a recursive call on $S$ can be computed exactly in $O(J|V(S)|+b_S)$ time by Algo.~\ref{alg:stad-specialized-cross}.
\end{lemma}


\begin{proof}[Proof of Lemma~\ref{lem:stad-specialized-cross}]
First consider a single exponential term $f(d)=\lambda^d$. Substituting $f(d)=\lambda^d$ into Eq.~\ref{eq:sumofFFTterms} gives
\begin{align*}
G_q[s]=\sum_{r=0}^{H}\sum_{t=1}^{m}\lambda^{q+|s-t|+r}A_r[t]= \lambda^q \sum_{r=0}^{H}\sum_{t=1}^{m} \lambda^{|s-t|+r}A_r[t]=\lambda^q G_0[s].
\end{align*}
Thus, computing the row $G_0[\cdot]$ suffices to evaluate every required entry of $G$.

Similarly, substituting into Eq.\ref{eq:stad-E}, evaluates the error term as
\begin{align*}
E_\alpha[q] = \sum_{r=1}^{H_\alpha}\lambda^{q+r}b_\alpha[r] = \lambda^q
\sum_{r=1}^{H_\alpha}\lambda^r b_\alpha[r] =\lambda^q E_\alpha[0],
\qquad 0\leq q\leq H_\alpha.
\end{align*}
Consequently,
\begin{equation}
G_q[s]=\lambda^q G_0[s],
\qquad
E_\alpha[q]=\lambda^q E_\alpha[0].
\label{eq:stad-specialized-factorization}
\end{equation}

Now lets try to calculate the row $G_0[\cdot]$. The question is can we do it efficiently without using FFTs? Reordering the finite sums in the definition of $G_0[s]$ yields
\begin{equation}
G_0[s] = \sum_{t=1}^{m}\lambda^{|s-t|} \left(\sum_{r=0}^{H}\lambda^r A_r[t]\right).
\label{eq:stad-specialized-G0}
\end{equation}

Using the definition of $A_r[t]$, its inner sum is
\begin{align*}
\sum_{r=0}^{H}\lambda^r A_r[t] \quad = \quad
\sum_{r=0}^{H} \sum_{\substack{v\in V(S)\\ 
\pi(v)=p_t,\ d(v)=r}} \lambda^r x_v
\quad = \quad \sum_{\substack{v\in V(S)\\\pi(v)=p_t}} \lambda^{d(v)}x_v.
\end{align*}
Therefore, these weighted sums can be accumulated directly from the vertices. Likewise, the definition of $b_\alpha[r]$ gives $E_\alpha[0] = \sum_{r=1}^{H}\lambda^r b_\alpha[r] = \sum_{v\in V(B_\alpha)} \lambda^{d(v)}x_v$.
Only one scalar correction value is needed per component.

\paragraph{Two sweeps for computing $G_0$.}
To compute $G_0[\cdot]$, first initialize the contribution of sources attached to each backbone vertex: $G_0[s]\gets \sum_{r=0}^{H}\lambda^r A_r[s],$ for $s=1,\ldots,m$.
At this stage, the entry at position $s$ contains only the contribution of sources whose projection is $p_s$, corresponding to the $t=s$ term in Eq.~\ref{eq:stad-specialized-G0}. The following two sweeps update this row in place to incorporate sources attached to the other backbone vertices. Once both sweeps finish, the row contains the required values $G_0[s]$.
\begin{enumerate}
    \item First perform the forward updates for all $s=2,\ldots,m$: $G_0[s]\gets G_0[s]+\lambda G_0[s-1]$. After this sweep, the stored entry at position $s$ equals $\sum_{r=0}^{H}\sum_{t=1}^{s} \lambda^{r+s-t}A_r[t]$. Indeed, this holds at $s=1$, and each update gives
    \begin{align*}
    \sum_{r=0}^{H}\lambda^r A_r[s] + \lambda \sum_{r=0}^{H}\sum_{t=1}^{s-1} \lambda^{r+s-1-t}A_r[t] = \sum_{r=0}^{H}\sum_{t=1}^{s} \lambda^{r+s-t}A_r[t].
    \end{align*}
    In particular, the entry at $s=m$ already equals the required value $G_0[m]$.

    \item Next perform the backward updates for all $s=m-1,\ldots,1$ : $G_0[s]\gets (1-\lambda^2)G_0[s]+\lambda G_0[s+1]$. To verify this update, suppose that the entry at $s+1$ already has its final value. Multiplying that value by $\lambda$ gives
    \begin{align*}
        &\lambda
        \sum_{r=0}^{H}\sum_{t=1}^{m}
        \lambda^{r+|s+1-t|}A_r[t] =
        \lambda^2
        \sum_{r=0}^{H}\sum_{t=1}^{s}
        \lambda^{r+s-t}A_r[t]
        +
        \sum_{r=0}^{H}\sum_{t=s+1}^{m}
        \lambda^{r+t-s}A_r[t].
    \end{align*}
    The first sum is $\lambda^2$ times the partial sum currently stored at position $s$. Adding $(1-\lambda^2)$ times that partial sum therefore leaves
    \[
    \sum_{r=0}^{H}\sum_{t=1}^{s}
    \lambda^{r+s-t}A_r[t]
    +
    \sum_{r=0}^{H}\sum_{t=s+1}^{m}
    \lambda^{r+t-s}A_r[t]
    =
    G_0[s].
    \]
    Backward induction proves that both sweeps together compute the complete row $G_0[\cdot]$.
\end{enumerate}

\paragraph{Recovering the cross-component contributions.}
Substituting Eq.~\ref{eq:stad-specialized-factorization} into Lemma~\ref{lem:stad-cross} gives
\[
\operatorname{Cross}_P(u)=
\begin{cases}
G_0[s],
&u=p_s,\\[3pt]
\lambda^{d(u)}
\bigl(G_0[a(\alpha)]-E_\alpha[0]\bigr),
&u\in V(B_\alpha).
\end{cases}
\]
Each query therefore requires constant work once $G_0[\cdot]$ and the values $E_\alpha[0]$ have been computed.

\paragraph{Computational cost.}
Decomposition and depth computation take $O(|V(S)|+b_S)$ time. For one exponential term, generate $\lambda^0,\ldots,\lambda^H$ successively and accumulate the weighted vertex contributions. This takes $O(|V(S)|)$ time because $H+1\leq |V(S)|$. Both backbone sweeps together take $O(m)$ time, and assigning the cross-component contributions takes $O(|V(S)|)$ time.

The dense tables $A_r[t]$ and $b_\alpha[r]$ are used only to establish the identities above; the implementation accumulates the required weighted sums directly from vertices. Thus, one exponential term requires $O(|V(S)|)$ arithmetic work after decomposition.

For the supplied exponential sum form $f(d)=\sum_{\ell=1}^{J}c_\ell\lambda_\ell^d$, perform these computations for each $\lambda_\ell^d$, multiply its cross-component contributions by $c_\ell$, and sum over $\ell$. By linearity and reuse of the decomposition, the total cost is $O(J|V(S)|+b_S)$.
\end{proof}

\begin{corollary}[Worst-Case Runtime of \STADTFI(Specialized Kernels)]
\label{cor:stad-specialized-runtime}
Let $T$ be an unweighted tree with $n$ vertices, and let the kernel be $f(d)=\sum_{\ell=1}^{J}c_\ell\lambda_\ell^d$, where  $J\geq1$, $c_\ell\in\mathbb R$ and $\lambda_\ell>0$. Under constant-time arithmetic, \STADTFI\ (SpclK-Adaptive), using the selection rule and the recurrences of Lemma~\ref{lem:stad-specialized-cross}, computes the tree-field integral in $O\!\left(Jn\log(n)\right)$ time.
\end{corollary}

\begin{proof}[Proof of Corollary~\ref{cor:stad-specialized-runtime}]
By Lemmas~\ref{lem:stad-specialized-cross}
and~\ref{lem:stad-selection}, the nonrecursive cost at
call $S$, including selection and accumulation, is
$O(J|V(S)|+b_S)$, since $J\geq1$.
Summing over calls and applying
Eq.~\ref{eq:stad-boundary-volume} gives
\[
 T_{\mathrm{specialized}}
 =
 O\!\left(
 J\sum_{S\in\mathcal R}|V(S)|
 +\sum_{S\in\mathcal R}b_S
 \right)
 =
 O\!\left(J\sum_{S\in\mathcal R}|V(S)|\right).
\]

From, Eq.~\ref{eq:stad-specialized-volume}, we know that:  $\sum_{S\in\mathcal R}|V(S)|
 \leq n\log_2(n+1).$

Substituting into the specialized cost bound gives $ T_{\mathrm{specialized}} =  O\!\left(J\sum_{S\in\mathcal R}|V(S)|\right) = O\!\left(Jn\log(2n)\right)$.
\end{proof}




\section{Additional Structural Tree Cases}
\label{app:stad-structural-cases}

We specialize the preceding analysis to trees for which the backbone geometry yields explicit runtime bounds. Throughout, kernel values can be evaluated in constant time, and the stated backbones are supplied or selected in linear time.

\subsection{Complexity Analysis for Paths}
\label{app:paths}
\begin{proposition}[Path]
\label{prop:stad-path}
Let $S$ be a path with $n$ vertices. Choosing the entire path as its backbone allows Algo.~\ref{alg:stad-solve} to compute all field values in $O(n\log(2n))$ time using either the one-dimensional or two-dimensional FFT implementation.
\end{proposition}

\begin{proof}
Choose $P=S$. Then $m=n$, $H=0$, and there are no off-backbone components. The only aggregate row is $A_0[t]=x_{p_t}$, for  $1\leq t\leq n$, and the interaction table reduces to $G_0[s]= \sum_{t=1}^{n}f(|s-t|)A_0[t]$. Proposition~\ref{prop:stad-1d-fft}, with $H=0$ and $m=n$, evaluates this row using one convolution in $O(n\log(2n))$ time. Lemma~\ref{lem:stad-2d-fft} gives the same bound.

By Lemma~\ref{lem:stad-cross}, every backbone vertex satisfies $w_{p_s}^{(S)}=G_0[s]$. No corrections or recursive calls are required. The remaining initialization and output assignment take $O(n)$ time.
\end{proof}

\subsection{Complexity Analysis for Caterpillar paths.}
\label{app:caterpillar}

\begin{proposition}[Caterpillars]
\label{prop:stad-caterpillar}
Let $T$ be a caterpillar with $n$ vertices, and let $P=(p_1,\ldots,p_m)$ be a backbone such that every vertex outside $P$ is a leaf attached to $P$. Using this backbone, Algo.~\ref{alg:stad-solve} computes all field values in $O\!\left(n+m\log(2m)\right) \subseteq O(n\log(2n))$ time with either FFT implementation.
\end{proposition}

\begin{proof}
If every vertex lies on $P$, the result follows from Proposition~\ref{prop:stad-path}. Otherwise, the maximum off-backbone depth is $H=1$. The aggregate table has two rows:
\[
A_0[t]=x_{p_t},
\qquad
A_1[t]
=
\sum_{\substack{v\in V(T)\setminus V(P)\\
                 \pi(v)=p_t}}x_v,
\qquad 1\leq t\leq m.
\]
Any number of leaves may share the same anchor; their field values are combined in $A_1[t]$. Constructing these aggregates takes $O(n)$ time. 

There are only four depth pairs, $(q,r)\in\{0,1\}\times\{0,1\}$. Consequently, Proposition~\ref{prop:stad-1d-fft} computes the interaction table using four one-dimensional convolutions, with total cost $O\!\left(4m\log(2m)\right)= O(m\log(2m))$.Alternatively, Lemma~\ref{lem:stad-2d-fft} gives $O\!\left(2m\log(4m)\right)=O(m\log(2m)).$

It remains to account for component corrections and recursion. Each off-backbone component consists of a single leaf, say $B_\alpha=\{u\}$. Its depth profile and correction satisfy
\[
b_\alpha[1]=x_u,
\qquad
E_\alpha[1]=f(2)x_u.
\]
Lemma~\ref{lem:stad-cross} therefore gives $\operatorname{Cross}_P(u)= G_1[a(\alpha)]-f(2)x_u$. The recursive base case contributes $f(0)x_u$, so
the complete field value is
\[
w_u
=
G_1[a(\alpha)]
+
\bigl(f(0)-f(2)\bigr)x_u.
\]
For a backbone vertex, the same lemma gives $w_{p_s}=G_0[s]$ directly.

There are $n-m$ leaf components. Their corrections, base-case evaluations, and output assignments together take $O(n)$ time. Thus, the total running time is $O\!\left(n+m\log(2m)\right)$.Since $m\leq n$, this is $O(n\log(2n))$.
\end{proof}

\begin{proposition}[A spine with equally long pendant paths]
\label{prop:stad-long-pendant-paths}
Let $\ell\geq1$. Construct $T$ from a spine $P=(p_1,\ldots,p_\ell)$ by attaching to each $p_t$ a path of $\ell$ edges, introducing $\ell$ new vertices per attachment. Then $n=|V(T)|=\ell+\ell^2$. Choose $P$ as the initial backbone and choose each entire remaining path as the backbone of its recursive subproblem. With the joint two-dimensional FFT implementation, Algo.~\ref{alg:stad-solve} computes all field values in $O(\ell^2\log(2\ell))=O(n\log(2n))$.
\end{proposition}

\begin{proof}
At the initial call, the backbone has $\ell$ vertices, the maximum off-backbone depth is $\ell$, and there are $\ell$ components: $m=\ell$, $H=\ell$, and $k=\ell$. Each component contains $\ell$ vertices and is a path with $\ell-1$ internal edges.

We first bound the work at the initial call. Decomposition, depth computation, and accumulation of vertex values take $O(n)$ time. The aggregate table has $(H+1)m=(\ell+1)\ell=n$ entries, so its initialization also takes $O(n)$ time.

By Lemma~\ref{lem:stad-2d-fft}, the interaction table is computed in
\begin{align*}
O\!\left(
(H+1)m\log\bigl(2(H+1)m\bigr)
\right)
&=
O\!\left(
\ell(\ell+1)\log\bigl(2\ell(\ell+1)\bigr)
\right)\\
&=
O(\ell^2\log(2\ell)).
\end{align*}

For each component, its depth profile has $\ell$ possible nonzero depth bins. Reversing the source-depth index expresses its correction
\[
E_\alpha[q]
=
\sum_{r=1}^{\ell}f(q+r)b_\alpha[r],
\qquad 1\leq q\leq\ell,
\]
as a one-dimensional convolution. As in the correction analysis leading to Eq.~\ref{eq:stad-local-2d}, all these values can be computed in $O(\ell\log(2\ell))$ time per component. There are $\ell$ components, so the total correction cost is
\[
\ell\cdot O(\ell\log(2\ell))
=
O(\ell^2\log(2\ell)).
\]
Profile construction and output assignment add only $O(n)$ time. Lemma~\ref{lem:stad-cross} then gives all exact cross-component contributions.

Next, each recursive subproblem is a path with $\ell$ vertices. By Proposition~\ref{prop:stad-path}, choosing that entire path as its backbone solves the subproblem in $O(\ell\log(2\ell))$ time, with no further recursion. The total recursive cost is therefore
\[
\ell\cdot O(\ell\log(2\ell))
=
O(\ell^2\log(2\ell)).
\]
Combining the cross-component and recursive contributions is exact by Eq.~\ref{eq:stad-split}. Adding the initial and recursive costs proves $T(n)=O(\ell^2\log(2\ell))$. Finally, $n=\ell(\ell+1)=\Theta(\ell^2)$, yielding $T(n)=O(n\log(2n))$.
\end{proof}

\paragraph{Comparison with separate one-dimensional FFTs.} For the backbone choices in Proposition~\ref{prop:stad-long-pendant-paths}, Proposition~\ref{prop:stad-1d-fft} instead bounds the initial interaction-table computation by
\[
O\!\left((\ell+1)^2\ell\log(2\ell)\right)
=
O(\ell^3\log(2\ell)).
\]
Together with the remaining work, this gives
an $O(n^{3/2}\log(2n))$ upper bound for the
separate one-dimensional FFT implementation.
The joint two-dimensional FFT yields the
$O(n\log(2n))$ bound above because the full
anchor-depth table contains only $n$ entries.

\subsection{Complexity Analysis for Balanced Trees}
\label{app:balancedtree}

\begin{proposition}[Balanced trees with bounded degree]
\label{prop:stad-balanced}
Let $T$ be an unweighted tree with $n$ vertices and maximum degree bounded by a constant $\Delta$. Suppose that every recursive call on a connected subtree $S$ chooses a diameter backbone $P_S$. Assume that there exist constants $c>0$ and $0<\rho<1$, independent of $S$ and $n$, such that $|V(P_S)|\leq c\log(2|V(S)|)$, and  $H_{P_S}(S)\leq c\log(2|V(S)|)$, and every component $B_\alpha$ of $S\setminus V(P_S)$ satisfies $|V(B_\alpha)|\leq \rho |V(S)|$. Then Algo.~\ref{alg:stad-solve} computes all field values exactly in exact arithmetic in $O(n\log(2n))$ time, using either separate one-dimensional FFTs or the joint two-dimensional FFT implementation.
\end{proposition}

\begin{proof}
Correctness follows from
Lemma~\ref{lem:stad-cross} and the recursive splitting identity Eq.~\ref{eq:stad-split}. We establish the running-time bound by analyzing the work within each call and then summing over the recursion.

\paragraph{Work within one recursive call.}
Consider a nontrivial subtree $S$ with $s=|V(S)|$ vertices. Write $m=|V(P_S)|$, and $H=H_{P_S}(S)$, and let $k$ be the number of off-backbone components.

Each component connects to the backbone by exactly one edge. Since every backbone vertex has degree at most $\Delta$,
\[
k
\leq
\sum_{v\in V(P_S)}\deg_S(v)
\leq \Delta m.
\]
Thus, the assumptions imply $m=O(\log s),$$ H=O(\log s),$ and $k=O(\log s)$.

Diameter selection, decomposition, accumulation
of vertex values, and output assignment take
$O(s)$ time.
Initializing the aggregate table additionally
takes $O((H+1)m)$ time, which is absorbed by
the FFT bounds below.

For separate one-dimensional FFTs,
Proposition~\ref{prop:stad-1d-fft} gives
\begin{align*}
\operatorname{Cost}(G)
&=
O\!\left((H+1)^2m\log(2m)\right)\\
&=
O\!\left((\log s)^3\log\log s\right)
\end{align*}
for sufficiently large $s$.

For the joint two-dimensional FFT,
Lemma~\ref{lem:stad-2d-fft} instead gives
\begin{align*}
\operatorname{Cost}(G)
&=
O\!\left(
(H+1)m\log\bigl(2(H+1)m\bigr)
\right)\\
&=
O\!\left((\log s)^2\log\log s\right).
\end{align*}

The component corrections in
Eq.~\ref{eq:stad-E} are evaluated by
one-dimensional FFT convolutions.
Including profile initialization and accumulation,
their total cost is
\begin{align*}
O\!\left(s+kH\log(H+1)\right)
&=
O\!\left(
s+(\log s)^2\log\log s
\right).
\end{align*}
Here, accumulating the profiles over all
components takes $O(s)$ time because the
components are disjoint.

Consequently, the nonrecursive work satisfies
\[
\operatorname{Work}_{\mathrm{1D}}(S)
=
O\!\left(
s+(\log s)^3\log\log s
\right)
=
O(s),
\]
and
\[
\operatorname{Work}_{\mathrm{2D}}(S)
=
O\!\left(
s+(\log s)^2\log\log s
\right)
=
O(s).
\]
These equalities follow because both logarithmic
expressions are $o(s)$.
Increasing the constant covers the remaining
small subproblems and the singleton base case.

\paragraph{Work over the recursion.}
For either implementation, there is therefore
a constant $C>0$ such that
\[
\operatorname{Time}(S)
\leq
Cs+
\sum_{\alpha=1}^{k}
\operatorname{Time}(B_\alpha).
\]
The child subtrees satisfy
\[
\sum_{\alpha=1}^{k}|V(B_\alpha)|
=
s-m
\leq s,
\qquad
|V(B_\alpha)|\leq\rho s.
\]

At any fixed recursion depth $d$, the subproblems
have disjoint vertex sets.
Hence, their total nonrecursive work is at most
\[
C
\sum_{\substack{S\text{ at}\\\text{depth }d}}
|V(S)|
\leq Cn.
\]

Along any root-to-leaf chain, the subtree size
decreases by at least the fixed factor $\rho$
at each step.
Thus, every subproblem at depth $d$ has size
at most $\rho^d n$.
Counting the root as the first level, the number
of recursion levels satisfies
\[
L
\leq
1+
\left\lceil
\frac{\log n}{\log(1/\rho)}
\right\rceil
=
O(\log(2n)).
\]
Summing the $O(n)$ work per level over these
levels gives
\[
\operatorname{Time}(T)
=
O(nL)
=
O(n\log(2n))
\]
for both FFT implementations.
\end{proof}

\subsection{Complexity Analysis for Simple Spiders}
\label{sec:spidercase}

Let $T$ be a spider with center $c$ and $s\geq3$ arms whose lengths, measured in edges from the center, satisfy $h_1\geq h_2\geq\cdots\geq h_s\geq1$. Write $B_\alpha$ for the component corresponding to arm $\alpha$ after removing $c$. Then
\[
|V(B_\alpha)|=h_\alpha,
\qquad
n=1+\sum_{\alpha=1}^{s}h_\alpha.
\]
Spiders with at most two arms are paths and are covered by Proposition~\ref{prop:stad-path}. We analyze the joint two-dimensional FFT implementation. Component corrections use the common depth range of the current recursive call, as in Algo.~\ref{alg:cross-subroutine}.

\begin{lemma}[Diameter backbone]
\label{lem:stad-spider-diameter}
Choose a diameter backbone at the initial call and choose each entire remaining arm as the backbone of its recursive subproblem. Define $M=(h_3+1)(h_1+h_2+1)$. Then the total running time is $O\!\left( n+M\log(2M) +(s-2)h_3\log(2h_3) \right)$.
\end{lemma}

\begin{proof}
A diameter joins the endpoints of two longest arms through the center. We may therefore choose the path containing $B_1$, $c$, and $B_2$. Its number of vertices is $m=h_1+h_2+1.$ Removing this path leaves precisely the components $B_3,\ldots,B_s$. Consequently, $k=s-2$, $H=h_3$, and $(H+1)m=M.$

By Lemma~\ref{lem:stad-2d-fft}, computing the backbone interaction table costs $O(M\log(2M)).$

The $s-2$ component corrections cost $O\!\left((s-2)h_3\log(2h_3)\right)$ in total. Including decomposition, accumulation, table initialization, and output assignment, the initial call therefore costs
\[
O\!\left(
n+M\log(2M)
+(s-2)h_3\log(2h_3)
\right).
\]
Table initialization is absorbed by the
$M\log(2M)$ term.

Lemma~\ref{lem:stad-cross} gives the exact
cross-component contributions.
For each remaining arm, Proposition~\ref{prop:stad-path}
gives the recursive cost
$O(h_\alpha\log(2h_\alpha))$.
Since $h_\alpha\leq h_3$ for $\alpha\geq3$,
\[
\sum_{\alpha=3}^{s}
h_\alpha\log(2h_\alpha)
\leq
(s-2)h_3\log(2h_3).
\]
The recursive work is therefore absorbed into the correction term. Eq.~\ref{eq:stad-split} then completes the exact computation.
\end{proof}

\begin{lemma}[Center backbone]
\label{lem:stad-spider-center}
Choose $P=(c)$ at the initial call and choose each entire arm as the backbone of its recursive subproblem. Writing $H=h_1$, the total running time is $O\!\left(n+sH\log(2H)\right)$.
\end{lemma}

\begin{proof}
At the initial call, $m=1$, the maximum
off-backbone depth is $H$, and there are $s$
components.
The backbone interaction table reduces to
\[
G_q[1]
=
\sum_{r=0}^{H}f(q+r)A_r[1],
\qquad 0\leq q\leq H.
\]
There is only one backbone position, so the joint
convolution of Lemma~\ref{lem:stad-2d-fft}
reduces to a one-dimensional convolution in depth.
Reversing the source-depth index evaluates all
entries in
\[
O\!\left((H+1)\log(2(H+1))\right)
=
O(H\log(2H))
\]
time.

Each component correction is also a one-dimensional depth convolution and costs $O(H\log(2H))$ time. Over all $s$ arms, the corrections therefore cost $O(sH\log(2H))$. Together with decomposition, aggregation, and output assignment, the initial cross-interaction computation costs $O\!\left(n+sH\log(2H)\right)$. Lemma~\ref{lem:stad-cross} ensures that the resulting cross-component contributions are exact.

Each remaining component $B_\alpha$ is a path
with $h_\alpha$ vertices.
By Proposition~\ref{prop:stad-path}, choosing the
entire component as its backbone solves it in
$O(h_\alpha\log(2h_\alpha))$ time.
The total recursive work is bounded by
\begin{align*}
\sum_{\alpha=1}^{s}
h_\alpha\log(2h_\alpha)
&\leq
\left(\sum_{\alpha=1}^{s}h_\alpha\right)
\log(2H)\\
&=
(n-1)\log(2H)\\
&\leq
sH\log(2H).
\end{align*}
Thus, this work is absorbed into the stated bound.
Combining the recursive and cross-component
contributions is exact by
Eq.~\ref{eq:stad-split}.
\end{proof}

\section{Synthetic tree benchmarks: construction and additional results}
\label{app:synthetic-benchmarks}

\subsection{Experimental protocol}
\label{app:synthetic-protocol}

For each tree $T$ with $N$ vertices, we compute $y_i=\sum_{j=1}^{N}\frac{x_j}{1+d_T(i,j)}$ for $i=1,\ldots,N$, where $d_T$ is the hop distance and $x$ is a scalar standard-normal field generated with NumPy's \texttt{default\_rng(123)}. For a fixed $N$, the same field is used across all methods and tree seeds. All trees are unweighted during integration, including those whose topology is obtained from an MST of a randomly weighted graph.

We time five implementations: brute force (BF), the FTFI-style centroid baseline, \STADTFI(Diameter + 1D) (D1), \STADTFI(Diameter + 2D FFT) (D2), and \STADTFI(Adaptive + 2D FFT) (A2). Each timed call includes decomposition or backbone selection, kernel-array construction, convolution, and recursive integration.  

The random recursive, random-MST, and preferential-attachment generators use three tree seeds. For these families, the reported runtime is the arithmetic mean of the three per-seed timing minima. Approximate $95\%$ confidence intervals are $\bar t\pm 4.303\,s_t/\sqrt{3}$, where $s_t$ is the sample standard deviation across the three tree seeds. Thus, the bands summarize variation across seeded instances, not variation across individual timing repetitions. Deterministic families use one tree instance per size and have no estimated confidence interval. On log-scale plots, the lower confidence limit is bounded below by $0.05\bar t$ for display; numerical summaries retain the original interval limits.

Every speedup is a ratio of the runtimes used in the plots: $S_{B\to M}(N)=\frac{\bar t_B(N)}{\bar t_M(N)}.$ In particular, ratios for random families are ratios of mean times, not means of per-seed ratios. Values below one indicate that $M$ is slower than $B$. Small differences near one should be interpreted as observed timing differences.

\subsection{Tree construction and parameter ranges}
\label{app:synthetic-generators}

All generators return undirected adjacency lists. A spine or handle parameter $\ell$ counts vertices; an arm length $L$ counts edges from its attachment vertex and therefore contributes $L$ new vertices.\footnote{For the caterpillar with hanging paths, $H=m=\ell$ holds for the \emph{construction spine}. The selected diameter instead has $m=3\ell$ vertices and $H=\ell$ for the tested values of $\ell$.} The depth of a complete rooted tree counts edges from its root to a leaf. Table~\ref{tab:synthetic-generators} lists the deterministic constructions. For compactness, define
\[
    \mathcal N=\{50,100,200,400,800\},\qquad
    \mathcal L=\{25,50,100,200,400\}.
\]

\begin{table}[t]
\centering
\small
\setlength{\tabcolsep}{3pt}
\caption{Deterministic tree constructions and complete parameter sweeps.
All leaves and hanging paths are distinct.}
\label{tab:synthetic-generators}
\begin{tabular}{@{}>{\raggedright\arraybackslash}p{0.23\linewidth}>{\raggedright\arraybackslash}p{0.35\linewidth}>{\raggedright\arraybackslash}p{0.35\linewidth}@{}}
\toprule
Family & Construction and vertex count & Tested parameters \\
\midrule
Path & One chain, $N$ vertices. & $N\in\mathcal N$. \\
Balanced binary & Complete binary tree; $N=2^{d+1}-1$. & $d\in\{4,5,6,7,8,9\}$. \\
Complete 3-ary & Complete ternary tree; $N=(3^{d+1}-1)/2$. & $d\in\{3,4,5,6\}$. \\
Caterpillar & Three leaves per spine vertex, $N=4\ell$. & $\ell\in\mathcal L$. \\
Caterpillar + hanging paths & One $\ell$-edge path at every spine vertex, $N=\ell+\ell^2$. & $\ell\in\{8,12,16,24,32\}$. \\
Comb & Two leaves per spine vertex, $N=3\ell$. & $\ell\in\mathcal L$. \\
Broom & Handle with $\ell/2$ leaves at one endpoint, $N=3\ell/2$. & $\ell\in\{40,80,160,320,640\}$. \\
Double broom & Handle with $\ell/4$ leaves at each endpoint, $N=3\ell/2$. & $\ell\in\{40,80,160,320\}$. \\
Uniform spider & One hub, $k$ arms of length $L$, $N=1+kL$. & $(k,L)\in\{(8,6),(12,8),$ $(16,12),(20,20),(30,25)\}$. \\
Short-arm spider & One hub and $k$ two-edge arms, $N=1+2k$. & $k\in\mathcal L$, $L=2$. \\
Long-arm spider & Eight arms, $N=1+8L$. & $L\in\{10,40,160,300,500\}$. \\
Star & One hub and $N-1$ leaves. & $N\in\mathcal N$. \\
\bottomrule
\end{tabular}
\end{table}

The three random constructions also use $N\in\mathcal N$:
\begin{itemize}[noitemsep, topsep=0pt]
    \item \emph{Random recursive tree.} Start with vertex $0$ and, for each $i=1,\ldots,N-1$, connect $i$ to a uniformly sampled vertex in $\{0,\ldots,i-1\}$.
    \item \emph{Random-MST topology.} Give every edge of the complete graph an independent uniform $[0,1)$ weight and run Kruskal's algorithm. Retain the resulting adjacency list and discard the weights before integration.
    \item \emph{Preferential-attachment variant.} The implementation initializes a sampling list with vertex $0$. Each new vertex samples a parent from this list, then appends the parent and itself. Consequently, just before adding vertex $i$, an existing vertex $v$ is selected with probability $(\deg(v)+\mathbf{1}_{\{v=0\}})/(2i-1)$.
\end{itemize}


\subsection{Largest-size results}
\label{app:synthetic-largest}

Table~\ref{tab:synthetic-largest} reports all five methods at the largest tested size of each retained sweep. D2 has the lowest recorded runtime in 14 of the 15 rows; A2 is fastest on the long-arm spider.

\begin{table}[t]
\centering
\small
\setlength{\tabcolsep}{5pt}
\caption{Largest-size integration times (in milliseconds). BF denotes brute force; D1, D2, and A2 denote Diameter-1D, Diameter-2D, and Adaptive-2D. Random-family entries are means across three tree seeds. Bold indicates the lowest recorded mean in each row.}
\label{tab:synthetic-largest}
\begin{tabular}{@{}lrrrrrr@{}}
\toprule
Family & $N$ & BF & FTFI & D1 & D2 & A2 \\
\midrule
Path & 800 & 250.93 & 59.31 & 1.19 & \textbf{1.15} & 2.09 \\
Balanced binary & 1023 & 399.17 & 64.10 & 49.46 & \textbf{23.80} & 56.63 \\
Complete 3-ary & 1093 & 458.63 & 56.36 & 63.56 & \textbf{29.72} & 61.63 \\
Caterpillar & 1600 & 1,004.16 & 84.25 & 35.77 & \textbf{35.43} & 38.02 \\
Caterpillar + paths & 1056 & 435.47 & 64.77 & 62.78 & \textbf{6.68} & 16.67 \\
Comb & 1200 & 577.43 & 67.53 & 24.20 & \textbf{24.11} & 25.74 \\
Broom & 960 & 356.59 & 54.76 & 10.63 & \textbf{10.57} & 11.80 \\
Double broom & 480 & 91.45 & 26.70 & 5.39 & \textbf{5.32} & 5.83 \\
Uniform spider & 751 & 213.60 & 44.51 & 32.04 & \textbf{4.57} & 5.86 \\
Short-arm spider & 801 & 249.24 & 41.22 & 27.13 & \textbf{25.65} & 30.51 \\
Long-arm spider & 4001 & 6,361.91 & 259.52 & 70,288.89 & 382.98 & \textbf{20.33} \\
Star & 800 & 240.45 & 25.25 & 23.71 & \textbf{23.52} & 25.02 \\
Random recursive & 800 & 256.57 & 43.73 & 37.48 & \textbf{18.29} & 38.83 \\
Random-MST topology & 800 & 261.73 & 46.37 & 54.21 & \textbf{17.09} & 41.79 \\
Preferential attachment & 800 & 255.10 & 38.42 & 32.38 & \textbf{21.14} & 33.22 \\
\bottomrule
\end{tabular}
\end{table}

\paragraph{Separating backbone and convolution gains.}
On paths, the D2 speedup over FTFI grows from $30.12\times$ at $N=50$ to $51.42\times$ at $N=800$. Since $H=0$, both diameter variants reduce the global interaction to a one-dimensional convolution, and their largest-size runtimes differ by only $1.03\times$. For balanced binary trees, D2 remains $2.69$-$2.82\times$ faster than FTFI across the six tested depths. On caterpillars with long hanging paths, the D1-to-D2 speedup increases from $5.02\times$ at $N=72$ to $9.40\times$ at $N=1056$. At the latter size, the root has $(m,H)=(96,32)$: grouping the depth interactions into a 2D convolution reduces runtime from $62.78$ to $6.68$\,ms.

\paragraph{When adaptive selection pays off.}
For the eight-arm spider at $N=4001$, D1 takes $70.29$\,s, whereas D2 takes $382.98$\,ms, a $183.53\times$ improvement. Nevertheless, D2 remains $1.48\times$ slower than FTFI. A2 reduces runtime further to $20.33$\,ms. The logged root separator changes from a diameter with $(m,H)=(1001,500)$ to the hub with $(m,H)=(1,500)$, reducing the anchor-depth array from $501{,}501$ to $501$ entries. The remaining components are eight paths of 500 vertices. The $1001\times$ array-size reduction explains the opportunity for acceleration, while the measured complete-call speedup is $18.84\times$. Along this sweep, A2 is slower than D2 at $N=81$, but becomes $1.55\times$, $5.96\times$, $11.36\times$, and $18.84\times$ faster at $N=321,1281,2401,4001$, respectively.

The local cost proxy does not invariably improve the complete recursion. On the $N=1056$ caterpillar with hanging paths, A2 chooses a centroid with $(m,H)=(1,48)$ and takes $16.67$\,ms, compared with $6.68$\,ms for D2. Although the root array is smaller, the centroid leaves components of sizes $495$, $528$, and $32$; the diameter leaves 30 paths of 32 vertices. Thus, root-grid size alone does not determine runtime. On the $N=800$ path, both methods choose the full diameter, but A2 takes $2.09$\,ms versus $1.15$\,ms for D2, reflecting the cost of evaluating separator candidates without changing the decomposition.

\subsection{Results for the remaining tree structures}
\label{app:synthetic-additional}

\paragraph{Shallow branches and regular branching.}
At their largest sizes, D2 achieves $2.38\times$, $2.80\times$, $5.18\times$, and $5.02\times$ speedups over FTFI on caterpillars, combs, brooms, and double brooms, respectively. Their diameter backbones have $H=1$, and D1 and D2 runtimes differ by less than $1.4\%$ in these four instances. The gains therefore primarily reflect the backbone decomposition. For the complete 3-ary tree at $N=1093$, D2 takes $29.72$\,ms versus $56.36$\,ms for FTFI and $63.56$\,ms for D1, corresponding to $1.90\times$ and $2.14\times$ speedups. A2 takes $61.63$\,ms and is $9.34\%$ slower than FTFI in this case.

\begin{figure}[h]
    \centering
    \includegraphics[width=\linewidth]
        {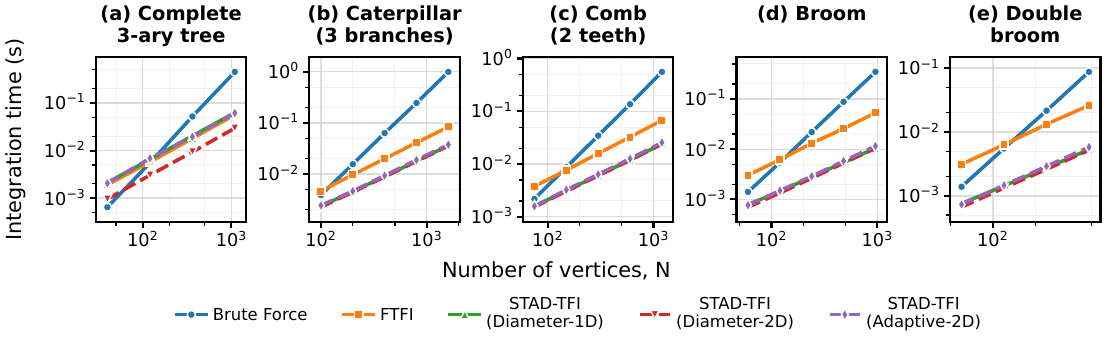}
    \vspace{-20pt}
    \caption{Additional runtime sweeps for complete 3-ary trees, caterpillars with three branches, combs with two  teeth, brooms, and double brooms.}
    \label{fig:synthetic-additional-structured}
\end{figure}




\begin{figure}[h]
    \centering
    \includegraphics[width=0.9\linewidth]
        {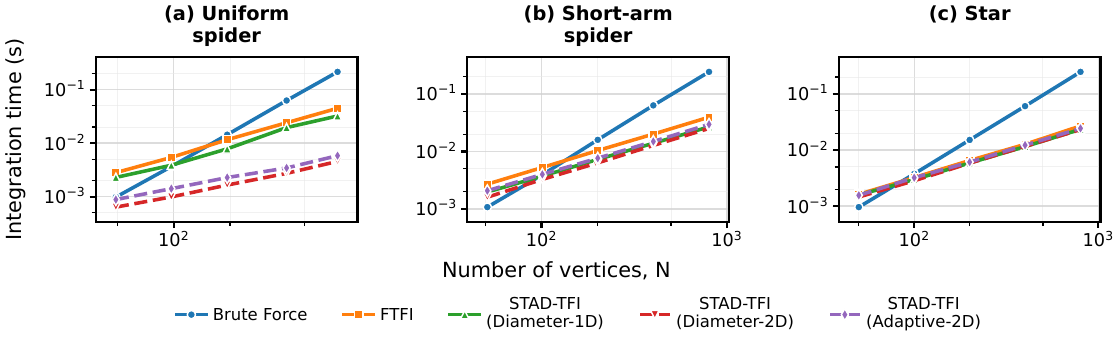}
    \vspace{-10pt}
    \caption{Additional runtime sweeps for uniform spiders,
    short-arm spiders, and stars.}
    \label{fig:synthetic-additional-shallow}
\end{figure}

\begin{figure}[htbp]
    \centering
    \includegraphics[width=0.9\linewidth]
        {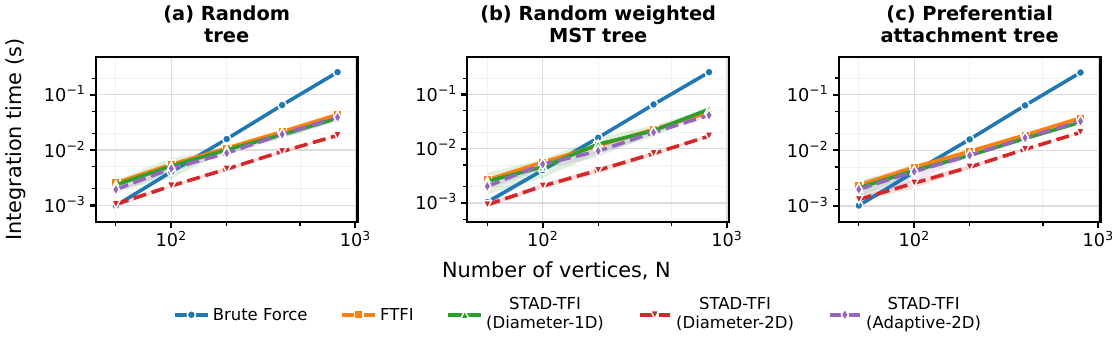}
    \vspace{-8pt}
    \caption{Runtime sweeps for random trees, random weighted MST trees, and preferential attachment trees. Curves show     means of per-seed timing minima across three tree seeds; bands show approximate $95\%$ confidence intervals.}
    \label{fig:synthetic-additional-random}
\end{figure}

\paragraph{Hubs with shorter arms.}
On the uniform spider with 30 arms of length 25 ($N=751$), D2 takes $4.57$\,ms versus $44.51$\,ms for FTFI and $32.04$\,ms for D1: speedups of $9.74\times$ and $7.02\times$. A2 takes $5.86$\,ms, or $28.21\%$ longer than D2, despite selecting the hub. On the short-arm spider with 400 two-edge arms ($N=801$), D2 takes $25.65$\,ms versus $41.22$\,ms for FTFI, a $1.61\times$ speedup; A2 takes $30.51$\,ms. For the star at $N=800$, D2 and A2 take $23.52$ and $25.02$\,ms versus $25.25$\,ms for FTFI, giving only $1.07\times$ and $1.01\times$ speedups. These cases show that a high-degree hub alone is insufficient to make adaptive selection advantageous; the long-arm spider is distinguished by the large diameter-depth array avoided by selecting its hub.

\paragraph{Random structures and variability.}
At $N=800$, D2 achieves $2.39\times$, $2.71\times$, and $1.82\times$ speedups over FTFI on random recursive trees, random-MST topologies, and the preferential-attachment variant, respectively. D2's mean runtimes and $95\%$ confidence-interval half-widths are $18.29\pm0.64$, $17.09\pm0.62$, and $21.14\pm0.38$\,ms. The corresponding FTFI values are $43.73\pm1.36$, $46.37\pm2.45$, and $38.42\pm0.96$\,ms. Relative to D1, D2 improves runtime by $2.05\times$, $3.17\times$, and $1.53\times$. A2 instead takes $38.83$, $41.79$, and $33.22$\,ms, retaining smaller $1.13\times$, $1.11\times$, and $1.16\times$ speedups over FTFI.


These results compare the supplied implementations in Python using NumPy and SciPy, for a scalar input field and one kernel. They do not measure GPU performance or the multichannel workload used by attention. In particular, D1 is not uniformly preferable to BF: on the largest long-arm spider it takes $70.29$\,s compared with $6.36$\,s for BF. The figures describe finite-size runtime behavior; asymptotic bounds must follow from the algorithmic analysis rather than fitted plot slopes.

\section{Geodesic Sinkhorn: experimental details and additional results}
\label{app:thingi-details}

\subsection{Optimal-Transport Formulation and Sinkhorn Computation}
\label{app:thingi-formulation}

Let $G=(V,E,w)$ be a connected weighted mesh graph with $n=|V|$ vertices. Each edge $(i,j)\in E$ is weighted by its Euclidean length. The graph-geodesic distance is
\[
d_G(i,j) = \min_{\pi:i\leadsto j} \sum_{(u,v)\in\pi}w_{uv},
\]
where the minimum is over paths connecting $i$ and $j$. Let $a,b\in\mathbb{R}_{\geq0}^{n}$ be source and target probability distributions on the vertices, satisfying $\mathbf{1}^{\top}a=\mathbf{1}^{\top}b=1$.

\paragraph{Entropically regularized optimal transport.}
A transport plan $\Pi\in\mathbb{R}_{\geq0}^{n\times n}$ specifies the mass transferred between vertices: $\Pi_{ij}$ is the amount transported from source vertex $i$ to target vertex $j$. The feasible plans satisfy the prescribed marginals, $U(a,b) = \left\{ \Pi\geq0: \Pi\mathbf{1}=a,\; \Pi^{\top}\mathbf{1}=b \right\}$. For a regularization parameter $\varepsilon>0$, we solve the entropically regularized optimal-transport problem~\citep{montesuma2024recent}
\begin{equation}
\min_{\Pi\in U(a,b)}
\left\{
\sum_{i,j}\Pi_{ij}d_G(i,j)
+
\varepsilon\sum_{i,j}
\Pi_{ij}\bigl(\log \Pi_{ij}-1\bigr)
\right\}.
\label{eq:thingi-ot}
\end{equation}
The first term measures transportation cost, while the second encourages a more diffuse allocation of mass; $\varepsilon$ controls the strength of this regularization.

\paragraph{Sinkhorn iterations.}
Define the kernel $K_G(i,j)= \exp\!\left(-{d_G(i,j)}/{\varepsilon}\right)$. The optimal regularized plan admits the representation $\Pi_G=\operatorname{diag}(u)K_G\operatorname{diag}(v)$. Sinkhorn's algorithm~\citep{cuturi2013sinkhorn} computes the scaling vectors through alternating updates
\begin{equation}
u^{(t+1)}
=
\frac{a}{K_Gv^{(t)}},
\qquad
v^{(t+1)}
=
\frac{b}{K_G^{\top}u^{(t+1)}},
\label{eq:thingi-sinkhorn}
\end{equation}
where division is elementwise and $v^{(0)}$ is initialized with positive entries. Each update rescales one marginal; at convergence, the resulting plan satisfies both marginal constraints.

Each iteration requires two kernel-vector products. Since the graph is undirected, $K_G^{\top}=K_G$, so both products use the same kernel operator. Nevertheless, this operator is dense: every pair of vertices has a positive kernel entry. Explicit storage therefore requires $O(n^2)$ memory, and the two products require $O(n^2)$ work per iteration. Accelerating these repeated products is the computational motivation for geodesic Sinkhorn~\citep{choromanski2026near}.

\paragraph{Spanning-tree approximation.}
Let $T\subseteq G$ be a spanning tree that retains the original edge weights, and let $d_T$ denote its weighted path distance. Replacing $d_G$ with $d_T$ in Eq.~\ref{eq:thingi-ot} gives a tree-based transport problem with kernel
\[
K_T(i,j) = \exp\!\left(-\frac{d_T(i,j)}{\varepsilon}\right).
\]
Its Sinkhorn updates have the same form as Eq.~\ref{eq:thingi-sinkhorn}, with $K_G$ replaced by $K_T$. Each required kernel-vector product is, therefore $(K_Tx)_i = \sum_{j=1}^{n} \exp(-d_T(i,j)/\varepsilon)x_j$, for $i=1,\ldots,n$. This is precisely a tree-field integration with the distance-dependent kernel $f_\varepsilon(d)=\exp(-d/\varepsilon)$. Because every tree path is also a path in $G$,
\[
d_T(i,j)\geq d_G(i,j),
\qquad
K_T(i,j)\leq K_G(i,j).
\]
The metric replacement changes the kernel and its
Sinkhorn scaling vectors, and can consequently change
the transport plan and transportation cost.
Our accuracy experiments assess this approximation
using the same source distribution, target distribution,
and regularization parameter for each tree construction.
The runtime experiments hold the tree fixed to compare
the computational cost of the field integrators
within the complete Sinkhorn solve.

\subsection{Meshes and probability distributions}
\label{app:thingi-data}

We use 28 Thingi10K meshes from the collection evaluated in the geodesic Sinkhorn experiments of GenusSink~\cite{choromanski2026near}. Each triangular mesh is converted into a weighted undirected graph by retaining every unique triangle edge and assigning it weight equal to the Euclidean distance between its endpoint coordinates. Fig.~\ref{fig:meshes} shows the evaluated meshes. The geometries are drawn from real mesh data; the source and target probability distributions (a and b respectively) are constructed as follows.

\begin{figure}[h]
    \centering
    \includegraphics[width=0.6\linewidth]{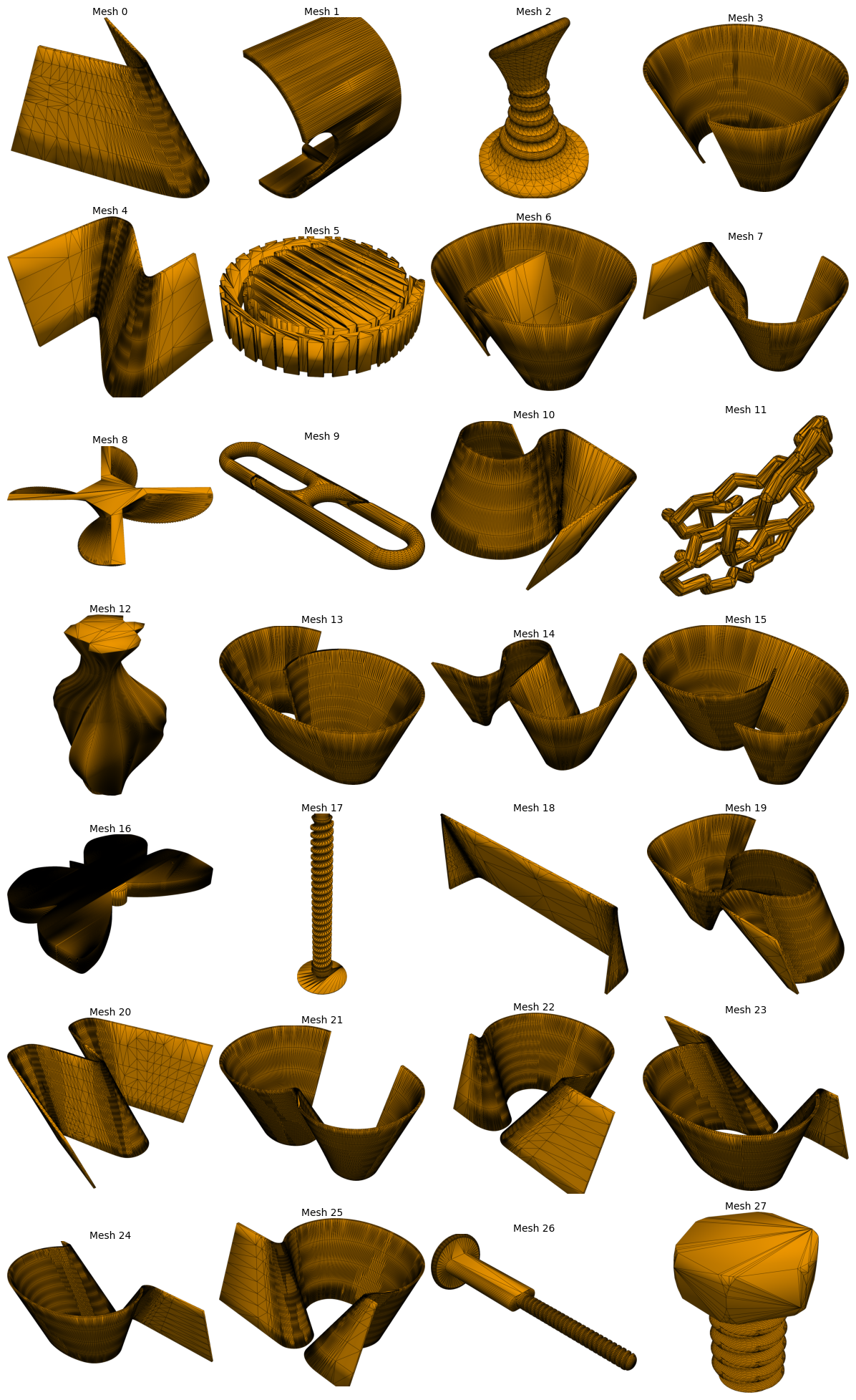}
    \caption{The 28 Thingi10K meshes used in the geodesic Sinkhorn experiments.}
    \label{fig:meshes}
\end{figure}

PCA identifies left, right, low, and high anchor vertices. For an anchor $c$, we normalize a geodesic Gaussian over the graph vertices: $g_c(i)=\frac{\exp(-d_G(i,c)^2/(2\sigma^2))}{\sum_{j\in V}\exp(-d_G(j,c)^2/(2\sigma^2))}$. The two probability measures are $a=0.7g_{\mathrm{left}}+0.3g_{\mathrm{low}}$, and $b=0.65g_{\mathrm{right}}+0.35g_{\mathrm{high}}$.

Writing $L_{\mathrm{box}}$ for the Euclidean diagonal length of the axis-aligned bounding box and $\bar w$ for the mean mesh-edge length, we use $\sigma=\max\{0.18L_{\mathrm{box}},3\bar w,10^{-8}\}$, $\varepsilon=0.2L_{\mathrm{box}}$. $L_{\mathrm{box}}$ is a geometric scale, not the graph-geodesic diameter. The measures and regularization parameter are held fixed across tree constructions for each mesh.

\subsection{Spanning-tree constructions}
\label{app:thingi-trees}

\begin{definition}[Stretch]
Given a spanning tree \(T \subseteq G\), the \emph{stretch} of an edge \((i,j)\in E\) induced by \(T\) is defined as $\operatorname{stretch}_T(i,j) = \frac{d_T(i,j)}{d_G(i,j)}$. The corresponding average stretch over the original graph edges is $\frac{1}{|E|} \sum_{(i,j)\in E} \operatorname{stretch}_T(i,j)$.
\end{definition}

We compare the following spanning-tree approximations. Each construction is applied independently to each mesh graph. The resulting trees retain the mesh-edge weights.
\begin{enumerate}[noitemsep, topsep=0pt]
    \item \textbf{Minimum Spanning Tree (MST):} constructs a spanning tree minimizing the total weight of the retained edges, without explicitly optimizing preservation of pairwise graph distances. Using a standard Kruskal or Prim implementation, the tree can be constructed in \(O(m\log n)\) time.

    \item \textbf{Random Shortest-Path Tree (Random SPT):} selects a root vertex uniformly at random and constructs a Dijkstra shortest-path tree, thereby exactly preserving distances from the selected root to every other vertex. Using a binary heap, the construction requires \(O(m\log n)\) time.

    \item \textbf{Approximate-Center Shortest-Path Tree (Center SPT):} first estimates a central graph vertex using a double-sweep and then constructs a shortest-path tree rooted at this approximate center. Since it requires only a constant number of Dijkstra searches, its overall construction time is \(O(m\log n)\).

    \item \textbf{Best-5 Shortest-Path Tree (Best-5 SPT):} constructs shortest-path trees from five randomly selected roots and retains the tree with the smallest sampled \(95\)th-percentile multiplicative stretch. Since the number of candidate roots is fixed, constructing the five candidate trees requires \(O(m\log n)\) time overall; with a fixed-size stretch sample, the selection step does not change this asymptotic complexity.

    \item \textbf{Diameter/Backbone Spanning Tree:} uses an approximate diameter as a backbone and attaches all remaining vertices to the backbone through a multi-source shortest-path forest. The approximate diameter computation and the multi-source Dijkstra search each require \(O(m\log n)\) time, giving an overall construction time of \(O(m\log n)\).

    \item \textbf{Alon-Karp-Peleg-West (AKPW) Tree:} applies the classical \citet{alon1995graph} low-stretch spanning-tree construction, designed to control the average stretch induced on the original graph edges. The method processes edges across increasing length scales, forming local clusters connected by shortest-path trees and recursively contracting these clusters until a spanning tree is obtained. The classical construction can be implemented in \(O(m\log n)\) time.

    \item \textbf{Abraham-Neiman Tree:} uses a prototype implementation based on the hierarchical petal decomposition framework of \citet{abraham2012using}. At each recursive step, the current connected subgraph is partitioned into a central residual cluster (the stigma) and several peripheral ``petals'' constructed around vertices far from the current cluster center. The resulting petals and stigma are connected through selected inter-cluster edges in a tree structure, and the construction is recursively applied within each cluster. The resulting low-stretch spanning tree is constructed in \(O(m\log n\log\log n)\) time.
\end{enumerate}
The experimental pipeline constructs and evaluates each of these trees independently for every mesh.

The motivation for considering several tree constructions is that they optimize different structural objectives. For example, the MST minimizes the total retained edge weight, $\sum_{e\in T} w_e,$ but does not directly optimize preservation of pairwise shortest-path distances. In contrast, a shortest-path tree rooted at $r$ exactly preserves all distances from the root, $d_T(r,v)=d_G(r,v)$ $\forall v\in V$, while low-stretch constructions explicitly aim to preserve graph distances more accurately. Our experiments therefore evaluate which notion of a ``good'' spanning tree is most closely aligned with preserving geodesic optimal transport.

\subsection{Runtime Comparisons and Speedups}
\label{app:thingi-runtime}

The runtime comparison fixes the tree construction to MST for every mesh. It contains the following seven implementations.
\begin{enumerate}[noitemsep, topsep=0pt]
    \item \emph{Brute Force Variations:}
    \begin{itemize}[noitemsep, topsep=0pt]
        \item \emph{Additive Brute Force.} For every source vertex, this method traverses the entire tree, explicitly accumulates the additive tree distance to every target, and evaluates the exponential kernel pairwise. Its cost is therefore $\Theta(n^2)$ per kernel-vector product.
        
        \item \emph{Exponential-Product Brute Force.} This method is also $\Theta(n^2)$, but exploits the identity $\exp\!\left(-\frac{d_T(i,j)}{\varepsilon}\right)=\prod_{e\in P_T(i,j)}\exp\!\left(-\frac{w_e}{\varepsilon}\right)$. The edge factors are therefore precomputed and propagated along tree paths, reducing the constant factor while retaining quadratic complexity. 
    \end{itemize}

    \item \emph{FFT-based tree field integrators.} The Thingi10K trees inherit arbitrary real-valued Euclidean edge weights, whereas the FFT formulations of Section~\ref{sec:algo} operate on discrete distance coordinates. We therefore quantize the tree metric before applying either FFT method; the complete construction and its approximation error are given in Appendix~\ref{app:fft_quantization}.
    
    \begin{itemize}[noitemsep, topsep=0pt]
        \item FTFI(\emph{1D-FFT}). We use an FTFI-style balanced-separator baseline following \citet{choromanski2022block}, in which distance-dependent cross interactions are reduced to one-dimensional structured convolutions. In our weighted setting these convolutions are evaluated on the quantized metric $d_Q$.
    
        \item \STADTFI(\emph{Adaptive-2D-FFT}). Our method uses the structure-adaptive decomposition of Section~\ref{sec:stad-adaptive}. At each recursive subtree it compares a diameter-path backbone with a singleton-centroid backbone according to Algo.~\ref{alg:choose_backbone}, and evaluates the resulting cross interactions using the 2D-FFT formulation of Algo.~\ref{alg:global-anchor-2dfft}. It operates on the same quantized kernel $K_Q$ as the 1D-FFT baseline; hence differences in runtime arise from the different decomposition and convolution strategies rather than from evaluating a different kernel.
    \end{itemize}

    \item \emph{Exact exponential-kernel methods.} These methods operate directly on the original real-valued tree metric $d_T$, require no quantization, and exploit $e^{-(a+b)/\varepsilon} =e^{-a/\varepsilon}e^{-b/\varepsilon}$. The Sinkhorn kernel is the $J=1$ case of the specialized kernel family in Section~\ref{sec:stad-specialized}, with $c_1=1$ and $\lambda_1=e^{-1/\varepsilon}$.
    
    \begin{itemize}[noitemsep, topsep=0pt]
        \item FTFI\ (\emph{SpclK-Centroid}). Uses a centroid-based implementation of the exponential-kernel factorization of \citet{choromanski2024fast}. In our formulation, this implementation is equivalent to \STADTFI\ (\emph{SpclK-Centroid}), with the backbone fixed to the centroid at every call.

        \item FTFI (\emph{SpclK-LinearExactDP}) denotes the exact two-pass dynamic program on the entire weighted tree, as mentioned in \citet[Lemma~3.5]{choromanski2022block}.
    
        \item \STADTFI\ (\emph{SpclK-Diameter}). Applies the backbone decomposition of Section~\ref{sec:stad-backbone} with a weighted diameter path. It aggregates source contributions at backbone anchors and evaluates interactions using the forward and backward recurrences of Section~\ref{sec:stad-specialized}, with exponential factors determined by the backbone edge weights.
    
        \item \STADTFI\ (\emph{SpclK-Adaptive}). Selects between a diameter backbone and a singleton centroid at each recursive call, following Section~\ref{sec:stad-adaptive}. Cross-component interactions are evaluated using the exponential-kernel specialization of Section~\ref{sec:stad-specialized}.
    \end{itemize}
\end{enumerate}

We quantify the computational gains of \STADTFI{} through five paired comparisons on the minimum spanning trees of the 28 Thingi10K meshes:
(i) \STADTFI{} (\emph{Adaptive-2D-FFT}) against FTFI (\emph{1D-FFT});
(ii) \STADTFI{} (\emph{SpclK-Diameter}) against FTFI (\emph{SpclK-Centroid});
(iii) \STADTFI{} (\emph{SpclK-Adaptive}) against FTFI (\emph{SpclK-Centroid});
(iv) \STADTFI{} (\emph{SpclK-Adaptive}) against FTFI (\emph{SpclK-LinearExactDP}); and
(v) \STADTFI{} (\emph{SpclK-Diameter}) against FTFI (\emph{SpclK-LinearExactDP}).

\paragraph{Speedup definition.}
For mesh $i$, baseline $B$, and proposed method $A$, define speedup as $S_i^{(q)}(B,A)=t_{i,B}^{(q)}/t_{i,A}^{(q)}$, where $q$ identifies the timing quantity. We consider the median runtime of a single kernel-vector product, the median time per Sinkhorn iteration, the median total Sinkhorn runtime excluding setup, and the latter total runtime including the recorded setup costs. These costs include tree construction, quantization where applicable, and operator and distance-matrix setup. A speedup greater than one indicates that $A$ is faster.

We aggregate the paired speedups using their geometric mean:
\begin{equation}
    S_{\mathrm{GM}}^{(q)}(B,A)
    =
    \exp\left(
        \frac{1}{|\mathcal{I}|}
        \sum_{i\in\mathcal{I}}
        \log S_i^{(q)}(B,A)
    \right),
    \label{eq:sinkhorn-speedup-geomean}
\end{equation}
where $\mathcal{I}$ is the set of meshes included in the comparison.
This gives each mesh equal weight in the comparison.

\begin{table}[t]
    \centering
    \small
    \setlength{\tabcolsep}{5pt}
    \caption{
        Geometric-mean speedups $t_B/t_A$ over 28 meshes.
        Values greater than one indicate that $A$ is faster.
    }
    \label{tab:sinkhorn-speedups}
    \resizebox{\linewidth}{!}{%
    \begin{tabular}{llrrrrr}
        \toprule
        \STADTFI{} variant ($A$)
        & Baseline ($B$)
        & Meshes
        & Single $Kx$
        & Per iteration
        & Sinkhorn
        & With setup \\
        \midrule
        \emph{SpclK-Diameter}
        & \multirow{2}{*}{FTFI (\emph{SpclK-Centroid})}
        & 28 & $3.31\times$ & $3.38\times$
        & $3.38\times$ & $3.35\times$ \\
        \emph{SpclK-Adaptive}
        &
        & 28 & $3.00\times$ & $3.07\times$
        & $3.07\times$ & $2.98\times$ \\
        \midrule
        \emph{Adaptive-2D-FFT}
        & FTFI (\emph{1D-FFT})
        & 28 & $5.99\times$ & $6.32\times$
        & $6.32\times$ & $5.84\times$ \\
        \midrule
        \emph{SpclK-Diameter}
        & \multirow{2}{*}{FTFI (\emph{SpclK-LinearExactDP})}
        & 28 & $1.24\times$ & $1.28\times$
        & $1.28\times$ & $1.20\times$ \\
        \emph{SpclK-Adaptive}
        &
        & 28 & $1.12\times$ & $1.17\times$
        & $1.17\times$ & $1.07\times$ \\
        \bottomrule
    \end{tabular}%
    }
\end{table}
Table~\ref{tab:sinkhorn-speedups} shows that the diameter and adaptive specialized variants achieve geometric-mean speedups of $3.38\times$ and $3.07\times$, respectively, in total Sinkhorn runtime over FTFI~(\emph{SpclK-Centroid}). Both outperform this centroid-based specialized baseline on all 28 meshes. Including setup retains speedups of $3.35\times$ and $2.98\times$, demonstrating that the gains persist after accounting for preprocessing. 

Compared with FTFI (\emph{SpclK-LinearExactDP}), \STADTFI{} (\emph{SpclK-Diameter}) and \STADTFI{} (\emph{SpclK-Adaptive}) achieve geometric-mean speedups of $1.24\times$ and $1.12\times$ for a single kernel-vector product, and $1.28\times$ and $1.17\times$ for total Sinkhorn runtime, respectively. Including setup, these speedups are $1.20\times$ and $1.07\times$. The diameter and adaptive variants are faster in total Sinkhorn runtime on 25 and 22 of the 28 meshes, respectively; including setup, they are faster on 24 and 19 meshes.

Across all 28 meshes, \STADTFI{} (\emph{Adaptive-2D-FFT}) achieves geometric-mean speedups of $5.99\times$ for a single kernel-vector product and $6.32\times$ for both per-iteration and recorded total Sinkhorn runtime over FTFI (\emph{1D-FFT}). Including setup costs, the speedup remains $5.84\times$.

Within each of the five paired comparisons, the recorded Sinkhorn iteration counts agree on every mesh. Consequently, the per-iteration and total Sinkhorn speedups coincide: the improvement comes from faster iterations rather than fewer iterations.

\subsection{Distance quantization for FFT-Based Integrators}
\label{app:fft_quantization}

The FFT-based tree field integrators require an additional discretization step when applied to the weighted trees obtained from the Thingi10K meshes. The spanning trees inherit the Euclidean edge lengths of the mesh and therefore have arbitrary real-valued edge weights $w_e\in\mathbb{R}_{+}$. Consequently, $d_T(i,j)=\sum_{e\in \Pi_T(i,j)} w_e$ does not, in general, take values on a uniformly spaced grid.

The FFT itself does not require integer-valued signal values. Rather, the issue is that the convolutional representation of the field integrator assumes that the distance argument can be indexed on a regular discrete lattice. For arbitrary real-valued tree distances there is, in general, no integer array index corresponding exactly to each possible value of $d_T(i,j)$. We therefore discretize the distance coordinate before applying the FFT.

Let $\Delta={\varepsilon}/{B}$, where $B$ denotes the number of distance bins per Sinkhorn regularization scale $\varepsilon$. For every tree edge $e$, define $q_e=\max\left\{1,\, \operatorname{round}\left({w_e}/{\Delta}\right)\right\}$. Thus $q_e\in\mathbb{Z}_{\geq 1}$ is the length of edge $e$ measured in integer lattice units. The corresponding quantized geometric edge length is $w_e^{(Q)}=q_e\Delta$. The induced quantized tree metric is therefore
\[
    d_Q(i,j)=\Delta \sum_{e\in P_T(i,j)}q_e.
\]

Thus, if $\widehat d_Q(i,j):=\frac{d_Q(i,j)}{\Delta}=\sum_{e\in P_T(i,j)} q_e$, then $\widehat d_Q(i,j)\in\mathbb{Z}_{\geq 0}$ and can be used directly as an array index in the FFT computations. For the Sinkhorn kernel, this gives
\[
    K_Q(i,j)
    =
    \exp\left(
        -\frac{d_Q(i,j)}{\varepsilon}
    \right)
    =
    \exp\left(
        -\frac{\Delta}{\varepsilon}
        \widehat d_Q(i,j)
    \right)
    =
    \exp\left(
        -\frac{\widehat d_Q(i,j)}{B}
    \right),
\]
since $\Delta=\varepsilon/B$.  Hence, after quantization, the kernel is sampled on a uniformly spaced integer distance lattice and can be evaluated using discrete FFT correlations and convolutions enabling the 2D-FFT algorithm of Section~\ref{sec:algo}.


It is important to distinguish this discretization error from the graph-to-tree approximation studied in Section \ref{thingi_treeapprox}. The latter replaces the original mesh metric $d_G$ by the tree metric $d_T$, whereas quantization subsequently replaces $d_T$ by $d_Q$ solely for the FFT-based implementation.  The exact exponential methods stop after the first approximation and evaluate the kernel induced by $d_T$ directly; only the FFT-based methods require the additional transition from $d_T$ to $d_Q$.

\paragraph{Quantization error.}
Let $\ell_{ij}=|P_T(i,j)|$ be the number of edges on the tree path
between $i$ and $j$. The rounding rule with $q_e\geq1$ gives
$|w_e^{(Q)}-w_e|\leq\Delta$ for every edge. Consequently,
\begin{equation}
    |d_Q(i,j)-d_T(i,j)|
    \leq \sum_{e\in P_T(i,j)}|w_e^{(Q)}-w_e|
    \leq \ell_{ij}\Delta,
    \label{eq:quant-distance-bound}
\end{equation}
and the corresponding kernel ratio satisfies
\begin{equation}
    e^{-\ell_{ij}/B}
    \leq \frac{K_Q(i,j)}{K_T(i,j)}
    \leq e^{\ell_{ij}/B},
    \qquad K_T(i,j)=e^{-d_T(i,j)/\varepsilon}.
    \label{eq:quant-kernel-bound}
\end{equation}
If every edge on the path has $w_e\geq\Delta/2$, ordinary rounding
gives the sharper distance bound $\ell_{ij}\Delta/2$.
Shorter edges are raised to $\Delta$ and can incur errors approaching
$\Delta$. Thus, the resolution parameter $B$ controls the per-edge
discretization scale, while pathwise errors can accumulate across edges.

\paragraph{Empirical quantization and transport-cost errors.}
All 28 MST benchmarks use $B=64$, so $\Delta=\varepsilon/64$.
We compare the original tree metric $d_T$ with $d_Q$ on the same MST,
with the same marginals and regularization parameter. Let $\Pi_T$ and
$\Pi_Q$ denote the corresponding Sinkhorn transport plans, and let
$\Pi_G$ denote the original-graph reference plan. Define
$C_G(\Pi)=\sum_{i,j}\Pi_{ij}D_G(i,j)$ and
$C_{\mathrm{ref}}=C_G(\Pi_G)$. We measure
\begin{align}
    E_{K,Q}
    &=\frac{\|K_Q-K_T\|_F}{\|K_T\|_F},
    \label{eq:quant-kernel-error}\\
    E_{\mathrm{rel}}(Q)
    &=\frac{|C_G(\Pi_Q)-C_{\mathrm{ref}}|}
        {C_{\mathrm{ref}}},
    \label{eq:quant-relative-ot-error}\\
    \operatorname{TV}_Q
    &=\frac12\sum_{i,j}|(\Pi_Q)_{ij}-(\Pi_T)_{ij}|.
    \label{eq:quant-plan-tv}
\end{align}
The reference plans are computed using the linear exponential-kernel
tree DP on the original and quantized metrics, respectively; both
reference solves satisfy the marginal-residual tolerance $10^{-7}$
on every mesh. The kernel error and plan TV isolate the effect of
quantization relative to the original tree problem. The relative OT
error evaluates the quantized-tree plan under the original mesh metric
$D_G$, so it includes the combined effects of tree approximation and
quantization. Its normalization supports comparisons across geometric
scales. These reference-plan metrics are evaluated separately from
the convergence behavior of the FFT-based solvers.

\begin{table}[t]
    \centering
    \small
    \setlength{\tabcolsep}{5pt}
    \caption{Quantization and transport-cost errors on the MSTs of
        28 Thingi10K meshes at $B=64$. Kernel error and plan TV
        compare the quantized and original tree problems; relative
        OT error uses the original-graph transport-cost reference.
        Each mesh has equal weight. We report the arithmetic mean
        $\pm$ sample standard deviation, median, and maximum.
        All values are dimensionless.}
    \label{tab:quantization-errors}
    \begin{tabular}{lrrr}
        \toprule
        Metric & Mean $\pm$ std. & Median & Maximum \\
        \midrule
        Kernel error $E_{K,Q}$
        & $0.3158\pm0.1409$ & $0.3201$ & $0.6909$ \\
        Relative OT error $E_{\mathrm{rel}}(Q)$
        & $0.0295\pm0.0262$ & $0.0274$ & $0.1014$ \\
        Plan variation $\operatorname{TV}_Q$
        & $0.0295\pm0.0258$ & $0.0238$ & $0.1140$ \\
        \bottomrule
    \end{tabular}
\end{table}

Table~\ref{tab:quantization-errors} shows a mean relative OT error of
$0.02955$ ($2.955\%$) for the quantized-tree plans. For comparison,
the original-tree plans have mean relative OT error $0.02838$
($2.838\%$), using the same graph reference. Thus, quantization
increases the mean relative OT error by approximately $0.117$
percentage points on these meshes. This is the difference between
the two mean graph-reference errors. The mean quantization-only
kernel error and plan TV are $0.3158$ and $0.0295$, respectively.

\subsection{Sinkhorn Implementation Details}
\label{app:sinkhorn-implementation}

\paragraph{Solver settings and stopping criterion.}
All methods in the runtime comparison use the same Sinkhorn driver, with double-precision scaling vectors initialized as $u^{(0)}=v^{(0)}=\mathbf{1}$. For the kernel operator $K$ associated with each method, a complete iteration consists of the elementwise updates
\begin{equation}
    u^{(t)}=\frac{a}{K v^{(t-1)}+\delta\mathbf{1}},
    \qquad
    v^{(t)}=\frac{b}{K^\top u^{(t)}+\delta\mathbf{1}},
    \qquad \delta=10^{-30}.
    \label{eq:app-sinkhorn-updates}
\end{equation}
The kernels are symmetric, so the same operator implements $K$ and $K^\top$. We evaluate the absolute marginal residual
\begin{equation}
    R_t=\max\!\left\{
        \left\|u^{(t)}\odot(Kv^{(t)})-a\right\|_\infty,
        \left\|v^{(t)}\odot(K^\top u^{(t)})-b\right\|_\infty
    \right\},
    \label{eq:app-sinkhorn-residual}
\end{equation}
where $\odot$ denotes elementwise multiplication. The residual is checked after the first complete iteration and every 20 iterations thereafter, i.e., at $t=1,21,41,\ldots$. A run terminates at the first check satisfying $R_t<10^{-7}$ or after 1,000 iterations. Each iteration uses two kernel-vector products; each residual check uses two additional products, whose cost is included in the recorded runtime. For a run reaching the iteration cap, the saved residual is the last periodically checked value, at iteration 981; the current runner does not recompute it at iteration 1,000.

\paragraph{Timing protocol and setup accounting.}
The recorded benchmarks use one numerical-library thread. For each mesh and integrator, the single kernel-vector product and the complete Sinkhorn solve are each timed three times, and we report their median runtimes. Each of the two brute-force baselines is timed once for each quantity. The input to the single-product benchmark is generated with NumPy random seed 42 and is shared across methods on the same mesh. Garbage collection is performed before each timed repetition and is excluded from the timing interval. Operator setup, including decomposition construction, is performed before these repetitions and reused throughout each solve. The reported time per iteration is the median, across repetitions, of the solve runtime divided by its executed iteration count; it therefore includes the amortized residual-checking cost. The ``With setup'' quantity adds the recorded tree-construction, operator-setup, quantization, and distance-matrix setup costs to the median solve runtime. The distance-matrix setup field is zero in the supplied runtime records: shared all-pairs distance calculations used for reference evaluation are excluded from these timings. For the FFT methods, the quantization resolution is $\Delta=\varepsilon/64$, as used by the supplied benchmark runner.

\section{Topological Vision Transformer Experiments: Details}
\label{app:vit}

\subsection{Masked linear attention as tree-field integration}
\label{app:reduction}

Let $q_i,k_j,v_j\in\mathbb{R}^{d_h}$ denote the query, key, and value
vectors for a single attention head, and let
$\phi:\mathbb{R}^{d_h}\rightarrow\mathbb{R}^{r}$ be the deterministic
feature map. The topological mask is
\[
M_{ij}
=
\exp\!\left(
u\,d_T(i,j)^2
+
v\,d_T(i,j)
+
w
\right).
\]
Masked linear attention computes
\begin{equation}
\mathrm{out}_i
=
\frac{
\sum_j M_{ij}
\big(\phi(q_i)^\top\phi(k_j)\big)v_j
}{
\sum_j M_{ij}
\big(\phi(q_i)^\top\phi(k_j)\big)
}.
\label{eq:app-masklin}
\end{equation}
Since $\phi(q_i)$ is independent of $j$, it can be moved outside the
summations:
\begin{equation}
\mathrm{out}_i
=
\frac{
\phi(q_i)^\top
\left(
\sum_j M_{ij}\phi(k_j)v_j^\top
\right)
}{
\phi(q_i)^\top
\left(
\sum_j M_{ij}\phi(k_j)
\right)
}.
\end{equation}

The two terms in parentheses are both of the form
$\sum_j f(d_T(i,j))x_j$, and are therefore instances of the tree-field
integration problem in Eq.~\ref{eq:stad-target}. We combine
the two payloads into
\begin{equation}
x_j
=
\big[
\phi(k_j)v_j^\top
\;\big|\;
\phi(k_j)
\big]
\in\mathbb{R}^{C},
\qquad
C
=
r\,d_h+r
=
r(d_h+1).
\end{equation}
A single tree-field integration over these $C$ channels gives both
quantities required by the attention computation. The first $r\,d_h$
channels produce the matrix appearing in the numerator, while the
remaining $r$ channels produce the normalizer.

For the models used here, $r=d_h=64$, giving
$C=64\cdot65=4{,}160$ channels per head. For fixed $r$ and $d_h$, the
channel width $C$ is independent of sequence length. A subquadratic
tree-field integrator in $N$ therefore gives a subquadratic
implementation of masked linear attention in $N$.

The same topological mask can be incorporated into softmax attention as
an additive logit bias,
\[
\log M_{ij}
=
u\,d_T(i,j)^2
+
v\,d_T(i,j)
+
w.
\]
This does not remove the quadratic pairwise computation of standard
exact softmax attention: it still requires $\Theta(N^2)$ query-key
interactions and corresponding pairwise topological biases. These
quantities need not be stored simultaneously as explicit $N\times N$
matrices, but the pairwise computation remains quadratic in $N$.

\subsection{Training details}
\label{app:training}

All models are trained from scratch on CIFAR-100 with a single fixed
recipe. Architecture: ViT-S ($\dim{=}384$, depth $12$, $6$ heads, MLP
$1536$), patch embedding at $224$px input; patch size $16$ gives
$N{=}196$ patch tokens and patch size $7$ gives $N{=}1{,}024$, in each
case plus one CLS token. Optimizer AdamW, learning rate $10^{-3}$,
weight decay $0.05$, $5$-epoch linear warmup followed by cosine decay to
zero over $100$ epochs, label smoothing $0.1$, gradient-norm clipping at
$1.0$, and \texttt{bfloat16} autocast. Augmentation is
$\mathrm{RandomCrop}(32,\text{pad}\,4)$ and horizontal flip on the native
$32\times32$ image, upsampled to $224\times224$ on device with bilinear
interpolation (antialiased). The batch size is $128$ at $N{=}196$ and
$256$ at $N{=}1{,}024$; all other hyperparameters are identical across
scales.

The mask uses the \emph{synced} setting of
\citet{choromanski2024fast}: three parameters $(u,v,w)$ per layer,
shared across heads, initialized to zero, so the mask is the identity at
initialization and the masked model coincides with the unmasked model
there (verified: at initialization the two models agree to floating-point
precision under matched arithmetic). The CLS row and column are left
unmasked ($M=1$). The exponent is clamped to $[-30,5]$ before
exponentiation, so $f$ is exactly its floor $e^{-30}$ beyond a
mask-dependent radius; this prevents overflow of the $u\,\dtree^2$ term.
The unmasked linear baseline is a Performer
\citep{choromanski2021performer} with $\phi(x)=\mathrm{elu}(x)+1$.

During training the mask is applied by exact evaluation of
Eq.~\eqref{eq:app-masklin}, computed in query blocks so that the
$N\times N$ score matrix is never materialized in full; this is exact
and is the appropriate choice at the training sequence lengths, where
the dense evaluation is faster than the recursive integrator. The
integrators are used at inference/measurement time; across $200$ tested
configurations (trees, layers, and learned parameter settings) their
outputs match the trained models' attention with maximum relative
deviation $8.6\times10^{-8}$.

Three-seed statistics (Table~\ref{tab:vit-acc}, $N{=}196$) use seeds
$\{0,1,2\}$; single-seed entries use seed $0$.

\begin{figure}[t]
\centering
\includegraphics[width=\linewidth]{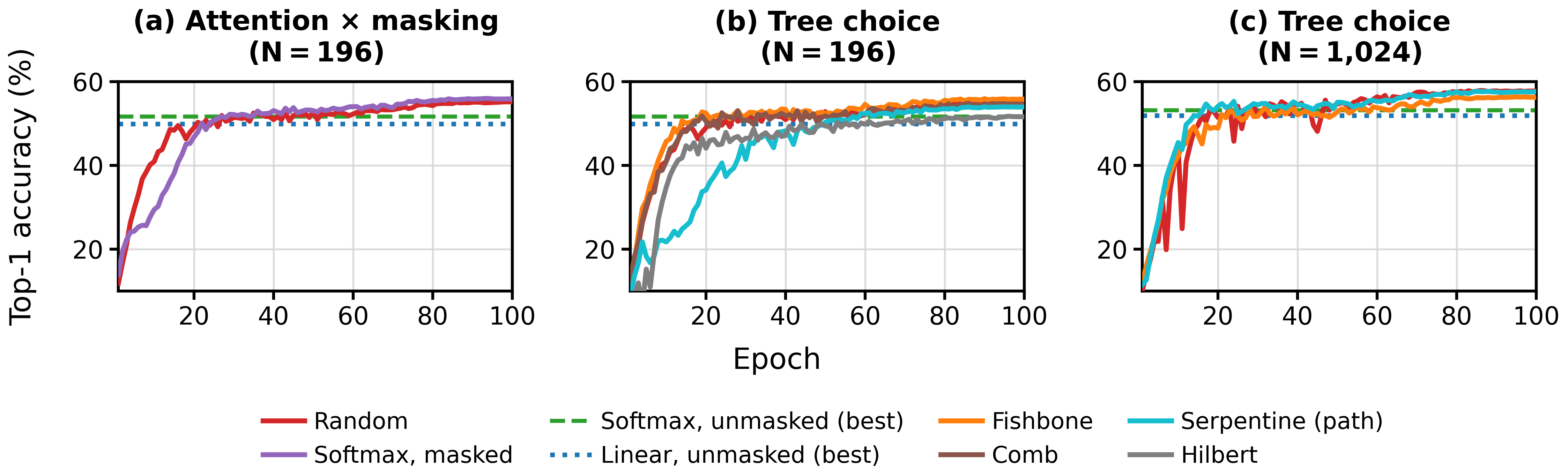}
\vspace{-18pt}
\caption{CIFAR-100 top-1 accuracy versus training epoch.
(a)~Masked softmax and masked linear attention at $N{=}196$ (curves
are means over three seeds). (b,~c)~Masked linear attention across
spanning trees at $N{=}196$ and $N{=}1{,}024$; multi-seed
configurations are drawn as seed means, the remaining configurations
as single seeds. In all panels, the dashed and dotted horizontal
lines mark the best accuracy of the unmasked softmax and unmasked
linear baselines, respectively. Every masked tree exceeds both
baselines (Appendix~\ref{app:vit-tree-choice}).}
\label{fig:vit-acc-curves}
\vspace{-6pt}
\end{figure}


\paragraph{Tiny-ImageNet replication.} The Tiny-ImageNet column of
Table~\ref{tab:vit-acc} uses Tiny-ImageNet-200 ($100{,}000$ training
and $10{,}000$ validation images at $64\times64$, $200$ classes)
under the identical recipe: augmentation is
$\mathrm{RandomCrop}(64,\text{pad}\,8)$ and horizontal flip at native
resolution, followed by the same on-device bilinear upsampling to
$224\times224$; normalization uses the dataset statistics; the
classification head has $200$ outputs; all remaining hyperparameters,
including patch size $16$ ($N{=}196$) and batch size $128$, are
unchanged. Every configuration was trained with seeds $\{0,1,2\}$ and
all runs converged. Per-seed best accuracies: unmasked softmax
$\{37.50, 37.52, 38.86\}$, unmasked linear $\{36.47, 36.51, 37.34\}$,
masked linear on the random tree $\{39.77, 40.24, 40.41\}$, masked
linear on the serpentine path $\{42.42, 41.15, 40.47\}$.

\paragraph{Masked-softmax ablation.} Softmax attention accepts the
same topological mask as an additive logit bias
(Appendix~\ref{app:reduction}). With the random-tree mask at
$N{=}196$, masked softmax reaches $56.06\pm1.31$ over three seeds,
versus $55.27\pm0.68$ for masked linear attention: no detectable
difference (Welch $t=0.9$, n.s.), so replacing softmax with linear
attention under the same mask preserves accuracy within the resolution
of three seeds. We do not train masked softmax at $N{=}1{,}024$, as
its $N\times N$ score-and-bias computation offers no route to
subquadratic evaluation.

\subsection{Attention-width evaluation of the cross term (G)}
\label{app:hankel}

At the attention width used in our experiments, $C=4{,}160$, the
per-channel cost of the backbone cross term becomes significant. For
some subtree shapes, the two-dimensional FFT construction of
Algo.~\ref{alg:global-anchor-2dfft} is therefore more expensive than
a matrix-multiplication-based evaluation.

Recall from Eq.~\ref{eq:sumofFFTterms} that for a backbone of length
$m$ and maximum off-backbone depth $H$,
\begin{equation}
G_q[s]
=
\sum_{r=0}^{H}
\sum_{t=1}^{m}
f\big(q+r+|s-t|\big)A_r[t],
\qquad
0\le q\le H,
\quad
1\le s\le m.
\label{eq:app-G}
\end{equation}
Writing
\[
\ell=q+|s-t|,
\]
the kernel argument becomes $\ell+r$. This separates the depth
correlation from the subsequent summation over backbone offsets.

\paragraph{Depth correlation.}
Define
\begin{equation}
B[\ell,t]
=
\sum_{r=0}^{H}
f(\ell+r)A_r[t],
\qquad
\ell=0,\dots,H+m-1.
\end{equation}
Let
\[
F[\ell,r]=f(\ell+r).
\]
Since the entries of $F$ depend only on the sum $\ell+r$, $F$ is a
Hankel matrix of shape $(H+m)\times(H+1)$. Combining the backbone index
$t$ and channel index into the column dimension gives the matrix product
\[
B=FA,
\]
where
\[
F\in\mathbb{R}^{(H+m)\times(H+1)},\qquad
A\in\mathbb{R}^{(H+1)\times(mC)}.
\]
Thus, the depth correlations for all backbone locations and channels can
be evaluated in one GEMM.

\paragraph{Backbone offset.}
After computing $B$,
\[
G_q[s]
=
\sum_{t=1}^{m}B[q+|s-t|,t].
\]
Splitting the sum at $t=s$ gives
\begin{equation}
G_q[s]
=
\underbrace{
\sum_{t\le s}B[q+s-t,t]
}_{L_q[s]}
+
\underbrace{
\sum_{t\ge s}B[q+t-s,t]
}_{R_q[s]}
-
B[q,s].
\end{equation}
The final term removes the duplicate contribution at $t=s$.

For the left term, define the antidiagonal coordinate
\[
\sigma=\ell+t.
\]
When $t\le s$ and $\ell=q+s-t$, every contributing entry satisfies
$\sigma=q+s$. An inclusive cumulative sum over $t$ for fixed $\sigma$
therefore gives $L_q[s]$ by reading the prefix ending at $t=s$.
The right term is obtained by applying the same construction after
reversing the backbone direction. The scan stage requires
$O\big((H+2m)m\big)$ memory operations per channel and makes no
assumption on the form of $f$.

\paragraph{Cost and dispatch.}
For one channel, the GEMM multiplies an
$(H+m)\times(H+1)$ matrix by an $(H+1)\times m$ matrix. Under the
standard convention that a multiplication and an addition count as
separate floating-point operations, this requires approximately
\begin{equation}
2(H+m)(H+1)m
\end{equation}
floating-point operations per channel. The prefix-scan stage requires
$O\big((H+2m)m\big)$ memory operations per channel.

For each subtree shape $(H,m)$, we compare a calibrated estimate of the
GEMM-plus-scan cost with the corresponding FFT-based evaluator and use
the cheaper implementation. For pure paths, where $H=0$,
Eq.~\ref{eq:app-G} reduces to
\[
G_0[s]
=
\sum_{t=1}^{m} f(|s-t|)A_0[t],
\]
which is a one-dimensional convolution along the path. The FFT
implementation is retained for this case.

In our numerical checks, the Hankel-GEMM evaluation agrees with direct
evaluation of Eq.~\ref{eq:app-G} to relative error on the order of
$10^{-15}$, including for growing kernels. Double-precision
accumulation is used for trained masks with large dynamic range; in
these cases we observed larger numerical error from padded FFT
evaluation.

\subsection{Channel-aware backbone selection}
\label{app:planner}

Algo.~\ref{alg:choose_backbone} selects a backbone independently
within each recursive subtree using a proxy for the resulting
computation cost. At channel width $C$, the cost contains two parts:
overhead paid once for the subtree and work repeated independently for
each channel. We therefore use the channel-aware proxy
\begin{equation}
\mathcal{C}(H,m,s,\{z_j\};C)
=
\mathcal{C}_{\mathrm{tree}}
+
C\,\mathcal{C}_{\mathrm{chan}}.
\end{equation}

The per-channel term $\mathcal{C}_{\mathrm{chan}}$ estimates the cost of
the evaluator used for the current subtree shape $(H,m)$. Depending on
the shape, this is the one-dimensional evaluator used when $m=1$, the
Hankel-GEMM evaluator of Appendix~\ref{app:hankel}, or the FFT-based
evaluator. The estimate also includes the correction term from
Appendix~\ref{app:eterms} and the per-channel cost of subsequent
recursive calls.

The per-tree term $\mathcal{C}_{\mathrm{tree}}$ accounts for
decomposition, dispatch, and recursive overhead that is incurred once
per subtree rather than once per channel.

The channel width
\[
C=r(d_h+1)
\]
is fixed by the model architecture and does not depend on the input.
For a fixed architecture and hardware calibration, the resulting
backbone plan therefore depends only on the tree and $C$ and can be
constructed once and reused. Since the per-channel term is multiplied
by $C$, changing the channel width can change which backbone minimizes
the proxy cost. In our cost model, subtree shapes that favor
diameter-style backbones at small channel width can favor more
centroid-like backbones at $C=4{,}160$.

For the path trees used in the runtime experiments, the selection is
trivial. The full tree is itself the diameter backbone, giving $m=N$
and $H=0$, so the field integration is evaluated in a single call
without recursive decomposition.

\subsection{Depth-bucketed correction terms (E)}
\label{app:eterms}

For a hanging component $B_\alpha$ with depth profile $b_\alpha$, the
correction term in Eq.~\ref{eq:stad-E} is
\begin{equation}
E_\alpha[q]
=
\sum_{r=1}^{H}
f(q+r)b_\alpha[r].
\end{equation}
Let
\[
D_\alpha
=
\max_{v\in B_\alpha}\mathrm{depth}(v)
\]
denote the maximum depth of the component. Since
$b_\alpha[r]=0$ for $r>D_\alpha$, and queries to $E_\alpha$ only require
$q\le D_\alpha$, both dimensions of the correlation can be restricted
to $D_\alpha$. Evaluating every component at the full subtree depth
$H$ would instead perform work on rows containing only structural
zeros.

We group components into logarithmic depth buckets according to
$\lceil\log_2 D_\alpha\rceil$. Let $\mathcal{B}_b$ denote one such
bucket and let
\[
\widehat D_b
=
\max_{\alpha\in\mathcal{B}_b}D_\alpha
\]
be its largest component depth. A single batched GEMM is then used for
the components in $\mathcal{B}_b$, sized according to
$\widehat D_b$. The per-channel correction cost is therefore
\begin{equation}
O\left(
\sum_b
|\mathcal{B}_b|
(\widehat D_b+1)^2
\right).
\end{equation}
Because components placed in the same logarithmic bucket have depths
within a constant factor of one another,
\begin{equation}
O\left(
\sum_b
|\mathcal{B}_b|
(\widehat D_b+1)^2
\right)
=
O\left(
\sum_\alpha
(D_\alpha+1)^2
\right).
\end{equation}
Thus, relative to evaluating all $s$ components at the global depth
$H$, whose per-channel cost is
$O\big(s(H+1)^2\big)$, depth bucketing gives the component-wise
complexity
\begin{equation}
O\left(
\sum_\alpha(D_\alpha+1)^2
\right),
\end{equation}
up to the constant-factor padding introduced by the logarithmic
buckets. This has the same component-wise depth dependence as the
corresponding per-component correlations in FTFI.

The truncation itself is exact: entries beyond $D_\alpha$ are zero in
the component depth profile, and output values with
$q>D_\alpha$ are never required.

\subsection{Runtime measurement protocol}
\label{app:runtime}

Runtime is measured on a single machine with $32$ CPU cores and
$256$\,GB of memory; the large memory allows the full $C=4{,}160$-channel
integration to run in a single pass, without the channel chunking that a
smaller memory would force (each chunk otherwise re-pays the per-call
tree work). Within each reported row (sequence length) all methods are
timed in the same process. Main-text tables report the median of
$K\geq 5$ timed runs after one discarded warm-up run ($K{=}3$ for the
recursive-bisection cell at $N{=}262{,}144$), with interquartile range
within $4\%$ of the median in all cases. The per-tree grid of
Appendix~\ref{app:vit-tree-choice} reports an earlier single-run
shared-process campaign on the same machine, whose values agree with
the repeated-run medians within ${\sim}10\%$. Because the unweighted patch grid admits any
spanning tree, each method is evaluated on the tree most favorable to
it: \ftfi{} on a random spanning tree (balanced, its best case for
centroid decomposition) and \STADTFI{} on a path tree
($H{=}0$). Dense evaluation materializes $M=f(D)$ and applies it; its
$N\times N$ mask is $\Theta(N^2)$ in time and memory and becomes
infeasible on commodity memory beyond ${\sim}3\times10^4$ tokens. Every
integrator cell is verified against dense evaluation where feasible, and
cross-validated between integrators otherwise, at relative tolerance
$10^{-9}$; observed relative error is ${\sim}10^{-15}$ throughout,
including on the learned masks. The advantage of \STADTFI{}
over \ftfi{} is in part algorithmic and in part architectural: on a path
tree its masking primitive reduces to one-dimensional convolutions and
matrix products that parallelize across cores, whereas \ftfi{}'s centroid
recursion is largely sequential.

\subsection{The Spanning Tree as a Design Choice}
\label{app:vit-tree-choice}

\subsubsection{Every spanning tree is a minimum spanning tree}
\label{sec:tc-free}

The topological mask $M_{ij}=\exp(u\,\dtree(i,j)^2+v\,\dtree(i,j)+w)$
depends on the hop distance $\dtree$ on a spanning tree $T$ of the
$s\times s$ patch grid. Prior work fixes $T$ to a random minimum
spanning tree. We observe that this leaves a large design space
unexplored: because every edge of the $4$-neighbor patch grid has unit
weight, \emph{every} spanning tree of the grid is a minimum spanning
tree. The choice of $T$ is therefore unconstrained by the MST
requirement, invisible to the original construction, and free to be
optimized. We study how this choice affects both the downstream accuracy
of the trained model and the runtime of the masking primitive, and find
that the two objectives are governed by different, and partly opposed,
properties of $T$.

\subsubsection{Tree families}
\label{sec:tc-families}

We evaluate five structurally distinct spanning trees of the patch grid,
chosen to span the range from balanced, high-branching structure to
path-like structure
(Figure~\ref{fig:trees}):
\begin{itemize}[noitemsep,topsep=2pt]
\item \textbf{random}: a seeded-Kruskal spanning tree; the incumbent
choice of \citet{choromanski2024fast}; a balanced, high-branching
tree with no dominant path.
\item \textbf{serpentine}: the row-major (boustrophedon) Hamiltonian
path; the whole tree is a single path.
\item \textbf{Hilbert}: the Hilbert space-filling-curve Hamiltonian
path.
\item \textbf{comb}: a caterpillar with the spine on the boundary row
and full-height teeth hanging from every column.
\item \textbf{fishbone}: a caterpillar with the spine on the central row
and half-height ribs extending up and down.
\end{itemize}
The two caterpillars share a spine-and-teeth structure but differ in the
depth of the teeth ($\approx\!s$ for comb, $\approx\!s/2$ for fishbone);
the two Hamiltonian paths differ only in traversal order.

\begin{figure}[t]
\centering
\includegraphics[width=\linewidth]{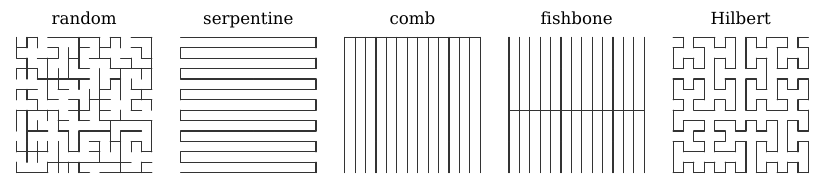}
\caption{The five spanning-tree families of the patch grid evaluated in
the tree sweep, from bushy (random) to path-like (serpentine, Hilbert).
Every spanning tree of the unit-weight grid is a minimum spanning tree,
so all are admissible under the topological-mask construction.}
\label{fig:trees}
\end{figure}

\subsubsection{Edge stretch as an accuracy predictor}
\label{sec:tc-stretch}

To predict how a tree affects accuracy we reuse the \emph{stretch}
statistic of the mesh-approximation experiments
(Section~\ref{sec:sinkhornexp}). For a spanning tree $T\subseteq G$ the stretch of a graph
edge $(i,j)$ is $\dtree(i,j)/d_G(i,j)$; since all patch-grid edges have
unit weight and connect graph-adjacent patches, the \emph{average edge
stretch} of $T$ reduces to the mean tree distance $\dtree(i,j)$ over
grid-adjacent patch pairs,
\begin{equation}
\mathrm{stretch}(T)=\frac{1}{|E|}\sum_{(i,j)\in E}\dtree(i,j),
\end{equation}
where $E$ is the set of grid edges. Low stretch means image-neighboring
patches remain close in tree metric, so the learned exponential-decay
mask can express a faithful locality prior; high stretch means
neighboring patches are separated by long tree detours, and the mask
cannot encode their proximity.

\subsubsection{Accuracy across trees}
\label{sec:tc-accuracy}

\paragraph{At $N{=}196$.} Table~\ref{tab:tc-acc196} reports single-seed
top-1 accuracy for the five trees under the identical training recipe
(ViT-S, masked linear attention), together with each tree's average edge
stretch on the $14\times14$ grid. Accuracy spans $4.21$ points across the
trees and decreases with stretch: the Pearson correlation between stretch
and accuracy is $-0.86$ (Spearman $-0.82$). The fishbone tree attains
both the lowest stretch and the highest accuracy, $+6.16$ points over the
unmasked baseline ($49.79$) and $+1.15$ over the incumbent random tree.
Two observations qualify the trend. At equal stretch ($7.50$) the comb
caterpillar exceeds the serpentine path by $2.04$ points, so global tree
structure carries information beyond the local statistic. The Hilbert
tree, although the Hilbert \emph{ordering} is widely used to preserve
locality, has the largest edge stretch as a \emph{tree}
(vertically adjacent patches connect only through long detours), gains
the least, and its layers converge to nearly flat masks with positive
coefficients, indicating reduced reliance on a poorly aligned geometry.

\begin{table}[t]
\centering
\caption{Accuracy as a function of tree choice (CIFAR-100, ViT-S,
masked linear attention, $N{=}196$, single seed), together with the
average edge stretch on the $14\times14$ grid. Accuracy decreases
with stretch (Pearson $-0.86$).}
\label{tab:tc-acc196}
\begin{tabular}{lcccc}
\toprule
tree & family & avg.\ edge stretch & top-1 (\%) & $\Delta$ vs.\ unmasked \\
\midrule
\emph{unmasked} & --- & --- & 49.79 & --- \\
fishbone   & caterpillar & \textbf{4.50} & \textbf{55.95} & $+6.16$ \\
comb       & caterpillar & 7.50 & 54.86 & $+5.07$ \\
random     & Kruskal     & 5.35 & 54.80 & $+5.01$ \\
serpentine & path        & 7.50 & 52.82 & $+3.03$ \\
Hilbert    & path        & 8.81 & 51.74 & $+1.95$ \\
\bottomrule
\end{tabular}
\end{table}

\paragraph{At $N{=}1{,}024$, the ranking reverses.}
Table~\ref{tab:tc-acc1024} reports accuracy at patch size $7$
($32\times32$ grid) for the trees trained at this scale, with stretch on
the $32\times32$ grid. The mask benefit is larger than at $N{=}196$ on
every tree, but the \emph{ranking among trees changes}: the random tree
($58.08$) now exceeds fishbone ($56.38$), reversing the $N{=}196$ order.
The stretch statistic predicts this reversal: from $14\times14$ to
$32\times32$, fishbone's stretch grows from $4.50$ to $9.00$ (a factor
of $2.0$), while random's grows only from $5.35$ to $7.21$ (a factor of
$1.35$), so the two cross and random becomes the lower-stretch, and
higher-accuracy, tree at the larger grid. The association between lower stretch and higher accuracy therefore
holds at both scales, but the stretch-minimizing tree is
resolution-dependent: the stretch of the structured trees grows faster
with grid size than that of the balanced random tree.

\begin{table}[t]
\centering
\caption{Accuracy as a function of tree choice at $N{=}1{,}024$
(patch $7$, $32\times32$ grid; CIFAR-100, ViT-S, masked linear
attention);
random and serpentine are means over two seeds
(Section~\ref{sec:tc-trainability}), fishbone single seed; average
edge stretch on the $32\times32$ grid. References: unmasked linear
$51.89$, plain softmax $53.19\pm0.49$ (three seeds). ``n.c.'': not
converged at this scale (Section~\ref{sec:tc-trainability}).}
\label{tab:tc-acc1024}
\begin{tabular}{lccc}
\toprule
tree & avg.\ edge stretch ($32^2$) & top-1 (\%) & $\Delta$ vs.\ softmax \\
\midrule
random     & 7.21 & \textbf{58.02} & $+4.83$ \\
serpentine & 16.50 & 57.84 & $+4.65$ \\
fishbone   & 9.00  & 56.38 & $+3.19$ \\
rec.\ bisection & \textbf{6.99} & \emph{n.c.} & --- \\
\bottomrule
\end{tabular}
\end{table}

\subsubsection{Seed protocol at $N{=}1{,}024$}
\label{sec:tc-trainability}

Three seeds per masked configuration were trained at $N{=}1{,}024$
under the shared recipe; one seed per configuration failed to
converge and is excluded from the reported means (random: seeds
$\{58.08, 57.96\}$; serpentine: $\{57.21, 58.47\}$). All unmasked
baseline seeds converged (softmax: $\{52.62, 53.44, 53.51\}$). The
recursive-bisection configuration did not converge in either of its
two seeds at this scale and is reported as such in
Table~\ref{tab:tc-acc1024}, despite training stably at $N{=}196$
($55.65\pm0.67$ over three seeds).

\subsubsection{Runtime is governed by off-backbone depth}
\label{sec:tc-runtime}

The choice of $T$ affects runtime through \STADTFI['s]
per-call cost, which for a subtree with backbone length $m$ and
off-backbone depth $H$ is $O\big(Hm\log(Hm)\big)$ plus recursion. The
recurrence attains its minimum, $O(N\log N)$ (one factor of $\log N$
below \ftfi{}'s $O(N\log^2 N)$), exactly when $H{=}0$: then the backbone
is the whole tree, $m{=}N$, and the cross term is a single
one-dimensional convolution with no recursion. A tree with $H{=}0$ is a
Hamiltonian path. Empirically, at attention channel width the per-channel
arithmetic dominates and is proportional to the off-backbone structure,
so a path, which has no off-backbone nodes, attains the arithmetic
minimum among all spanning trees; any tree with $H>0$ performs strictly more
per-channel work.

Table~\ref{tab:tc-runtime} confirms this: single-pass runtime of the
masking primitive at $C{=}4{,}160$ scales near-linearly for the path
($H{=}0$) but super-linearly for trees whose depth grows with the grid.
Fishbone ($H\!\approx\!s/2\!=\!\Theta(\sqrt{N})$) grows as $\sim\!N^{1.5}$
and comb ($H\!\approx\!s\!=\!\Theta(\sqrt{N})$, deeper teeth) grows
faster still; within the caterpillar--path family the runtimes are
monotone in $H$ at every $N$ ($H$ alone does not predict absolute
runtime across families---the balanced random tree is faster than comb
despite a far larger diameter-backbone $H$, because the integrator
recurses rather than paying that depth directly). Fishbone also
illustrates that favorability at one problem size does not transfer
across scales: against the \ftfi{}-random baseline it is faster at
$N{=}16{,}384$ ($13.41$ vs.\ $28.68$\,s), roughly tied at
$N{=}65{,}536$, and slower beyond. This is the empirical counterpart of the $H=O(\log N)$ condition in
the complexity analysis (Section~\ref{sec:stad-complexity}): only the path
keeps $H$ bounded at every grid size, and only the path therefore
realizes the near-linear guarantee at scale.

\begin{table}[t]
\centering
\caption{Single-pass runtime (seconds) of the masking primitive per
tree at $C{=}4{,}160$ on a $32$-core CPU; all methods within a row
are timed in the same process. The upper block reports \STADTFI; the
middle block reports \ftfi{} on the same tree ($\dagger$: not
completed within a $15$-minute budget); the final row is the
incumbent \ftfi{}-random baseline. The path scales near-linearly for
\STADTFI, deeper trees super-linearly, and the relative efficiency
shifts toward \STADTFI{} as trees become more backbone-dominated
(cf.\ Table~\ref{tab:vit-crossover}).}
\label{tab:tc-runtime}
\begin{tabular}{lcccc}
\toprule
$N$ & $16{,}384$ & $65{,}536$ & $147{,}456$ & $262{,}144$ \\
\midrule
\STADTFI, serpentine (path, $H{=}0$)      & \textbf{8.89}  & \textbf{41.80}  & \textbf{96.71}  & \textbf{196.49} \\
\STADTFI, fishbone ($H{\approx}\sqrt N/2$)& 13.41 & 103.19 & 492.82 & 813.26 \\
\STADTFI, comb ($H{\approx}\sqrt N$)      & 30.70 & 264.71 & 786.96 & 2593.84 \\
\STADTFI, random                          & 27.63 & 133.01 & 397.78 & 1618.01 \\
\midrule
\ftfi{}, serpentine                   & 240.31 & 1099.08 & $\dagger$ & $\dagger$ \\
\ftfi{}, fishbone                     & 34.69  & 225.81  & $\dagger$ & $\dagger$ \\
\ftfi{}, comb                         & 56.88  & 407.87  & $\dagger$ & $\dagger$ \\
\ftfi{}, random                       & 27.75  & 108.62  & 252.77 & 472.26 \\
\midrule
\ftfi{}, random (baseline)            & 28.68 & 109.38 & 263.13 & 447.22 \\
\bottomrule
\end{tabular}
\end{table}

\begin{table}[t]
\centering
\caption{Runtime crossover at $N{=}65{,}536$, $C{=}4{,}160$
(seconds; protocol as in Table~\ref{tab:vit-runtime}). $H$ is the
off-backbone depth of the diameter backbone; the relative
efficiency shifts toward \STADTFI{} on backbone-dominated trees.}
\label{tab:vit-crossover}
\begin{tabular}{lcccc}
\toprule
Tree & $H$ & $T_{\ftfi{}}$ &
$T_{\mathrm{StAd}}$ & $T_{\ftfi{}}/T_{\mathrm{StAd}}$ \\
\midrule
Random     & 616 & \textbf{100.0} & 121.4 & $0.82\times$ \\
Comb       & 255 & 379.3  & 231.1 & $1.6\times$ \\
Fishbone   & 128 & 217.1  & 91.4  & $2.4\times$ \\
Serpentine & \textbf{0} & 964.1
  & \textbf{37.8} & $\mathbf{25.5\times}$ \\
\bottomrule
\end{tabular}
\end{table}

To isolate the benefit of tree choice from that of theintegrator, Table~\ref{tab:vit-crossover} runs both methods (FTFI and \STADTFI) on four tree structures with identical inputs at $N=65{,}536$. \STADTFI{} is $1.6\times$ faster on the comb and $2.4\times$ faster on the fishbone, showing that its advantage extends beyond paths. On the serpentine path, the speedup reaches $25.5\times$. On the random tree, however, \ftfi{} is faster: $100.0$\,s versus $121.4$\,s. Thus, the measured benefit depends on tree structure, as motivated by Theorem~\ref{thm:stad-adaptive-geometry}; a matching worst-case asymptotic guarantee does not imply a runtime improvement on every tree.

\subsubsection{Integrator crossover across tree structure}
\label{sec:tc-crossover}

Which integrator is computationally favorable is itself governed by
tree structure; Table~\ref{tab:vit-crossover} (main text) reports both
integrators on each tree at $N{=}65{,}536$ under the repeated-timing
protocol of Appendix~\ref{app:runtime}. Two details beyond the
main-text discussion. First, no single structural scalar governs
absolute runtime: $T_{\mathrm{StAd}}$ is lower on random than on comb
despite random's far larger $H$---what shifts systematically with
backbone dominance is the \emph{relative} efficiency. Second, the
crossover is monotone in the diameter-backbone depth $H$ across all
measured families, and appears already on non-path caterpillars,
i.e.\ well before the degenerate path endpoint.

\subsubsection{Choosing the path: minimum linear arrangement}
\label{sec:tc-minla}

Because the runtime criterion selects the path family ($H{=}0$), the
remaining freedom is the choice of Hamiltonian path; all paths have
identical runtime and differ only in stretch, and hence in accuracy. Minimizing the average edge
stretch of a Hamiltonian path is the minimum linear arrangement problem
for the grid. Table~\ref{tab:tc-paths} compares five path constructions by average
edge stretch: the boustrophedon (row- or column-snake, which coincide by
symmetry) achieves the lowest stretch, below the Hilbert, diagonal-snake,
and spiral orderings, consistent with the boustrophedon being near-optimal
for grid minimum linear arrangement. We therefore use the serpentine
(row-snake) path. A fundamental limit remains: the grid has bandwidth
$\Theta(\sqrt N)$, so \emph{every} Hamiltonian path has average edge
stretch $\Theta(\sqrt N)$; at $32\times32$ the stretch of the best
path ($16.50$) is roughly twice that of the balanced trees
($7.21$--$9.00$). Linearly
ordering a two-dimensional grid unavoidably distorts locality; this is
the reason the runtime-optimal tree is not the accuracy-optimal tree.

\begin{table}[t]
\centering
\caption{Hamiltonian-path constructions ranked by average edge
stretch; all attain identical $H{=}0$ runtime. The row-snake
(serpentine) ordering attains the lowest stretch.}
\label{tab:tc-paths}
\begin{tabular}{lcc}
\toprule
path construction & stretch ($14^2$) & stretch ($32^2$) \\
\midrule
row-snake (serpentine) & \textbf{7.50} & \textbf{16.50} \\
column-snake           & 7.50 & 16.50 \\
Hilbert                & 8.81 & 19.62 \\
diagonal-snake         & 9.50 & 21.50 \\
spiral                 & 16.64 & 40.56 \\
\bottomrule
\end{tabular}
\end{table}

\subsubsection{A negative result: bounded-depth caterpillars}
\label{sec:tc-capped}

A natural attempt to obtain both low stretch and low runtime is a
\emph{capped} fishbone: cap the rib length at a constant $h$ (chaining
several short-ribbed spines), so that $H{=}h{=}O(1)$ satisfies the
$H=O(\log N)$ complexity condition while the local geometry, and hence
stretch, stays fishbone-like (measured stretch $6.22$ at $32\times32$,
below plain fishbone's $9.00$). This construction bounds the runtime \emph{exponent} but does not
outperform the path; profiling attributes the gap as follows.
Capping the ribs creates $\Theta(N/h)$ identical small components; a
naive recursion pays per-call overhead on all of them, but at attention
width this overhead ($1.28$s of a $33$s call at $N{=}16{,}384$) is
negligible against the per-channel \emph{arithmetic} ($32$s). Batching
the identical components into a single stacked solve (exact, verified to
$10^{-15}$) removes the overhead but not the arithmetic. The capped
fishbone performs $2.3\times$ more per-channel arithmetic than the path
($7.7$ vs.\ $3.3$\,ms), because its ribs and two-dimensional spine structure entail work that
a path, which has no off-backbone structure, does not perform. The path therefore remains both the runtime and the arithmetic
optimum; bounded depth is necessary but not sufficient.

\subsubsection{Summary: an accuracy--runtime frontier}
\label{sec:tc-frontier}

The spanning tree is a design parameter that trades accuracy against
masking runtime, and the two objectives are governed by opposed
properties. Accuracy is maximized by \emph{low edge stretch}, which at
scale favors balanced trees such as the random MST, on which \STADTFI{} has no
structural advantage over \ftfi{}. Masking
runtime is minimized by \emph{zero off-backbone depth}, i.e.\ a
Hamiltonian path, on which the integrator attains its $O(N\log N)$ best
case but which, by the grid's $\Theta(\sqrt N)$ bandwidth, has high
stretch. No spanning tree of the grid is simultaneously bushy and a
path, so no single tree is optimal on both axes. Among paths, serpentine
minimizes stretch and, crucially, still \emph{exceeds} softmax-level
accuracy at both evaluated scales ($+2.43$ at $N{=}196$, $+4.65$ at
$N{=}1{,}024$; Table~\ref{tab:vit-acc}), so the runtime-optimal tree
retains the mask's accuracy benefit while admitting the fastest exact
masking primitive at scale.